\documentclass[twoside,11pt]{article}

\usepackage{blindtext}
\usepackage{zymacros}
\usepackage{graphicx}
\graphicspath{{figures/}}

\usepackage[preprint]{jmlr2e}
\usepackage{subcaption}
\usepackage{enumitem}

\newtheorem{condition}{Condition}%
\usepackage{color}

\usepackage{makecell}

\usepackage{lastpage}

\firstpageno{1}

\begin{document}

\title{Deep Neural Networks with General Activations: Super-Convergence in Sobolev Norms}

\author{\name Yahong Yang \email yyang3194@gatech.edu \\
       \addr School of Mathematics\\
       Georgia Institute of Technology\\
686 Cherry Street,
Atlanta, GA 30332-0160, USA
       \AND
       \name Juncai He\thanks{\hspace{-2pt}Corresponding author.} 
       \email jche@tsinghua.edu.cn \\
       \addr Yau Mathematical Sciences Center\\
       Tsinghua University\\
       Haidian District, Beijing 100084, China}

\editor{TBD}

\maketitle

\begin{abstract}
This paper establishes a comprehensive approximation result for deep fully-connected neural networks with commonly-used and general activation functions in Sobolev spaces $W^{n,\infty}$, with errors measured in the $W^{m,p}$-norm for $m < n$ and $1\le p \le \infty$. The derived rates surpass those of classical numerical approximation techniques, such as finite element and spectral methods, exhibiting a phenomenon we refer to as \emph{super-convergence}. Our analysis shows that deep networks with general activations can approximate weak solutions of partial differential equations (PDEs) with superior accuracy compared to traditional numerical methods at the approximation level. Furthermore, this work closes a significant gap in the error-estimation theory for neural-network-based approaches to PDEs, offering a unified theoretical foundation for their use in scientific computing.
\end{abstract}


\section{Introduction}

Deep neural networks (DNNs) have had a significant impact on scientific and engineering fields, including the numerical solution of partial differential equations (PDEs) \cite{Lagaris1998,weinan2017deep,raissi2019physics,de2022error}. Compared to learning a target function in regression or classification tasks, solving PDEs is more challenging, since the DNN must approximate the target function with small discrepancies not only in magnitude but also in its derivatives—often up to higher orders. For example, when using the Deep Ritz method \cite{weinan2017deep} to solve a PDE such as
\begin{equation}
\begin{cases}
-\Delta u = f, & \text{in } \Omega, \\
\displaystyle \frac{\partial u}{\partial \nu} = 0, & \text{on } \partial \Omega,
\end{cases}
\label{PDE}
\end{equation}
the corresponding loss function can be written as
\[
\mathcal{E}_D(\vtheta) \;:=\; \frac{1}{2} \int_{\Omega} \lvert \nabla \phi(\vx;\vtheta)\rvert^2 \,\mathrm{d}\vx
\;+\; \frac{1}{2} \Bigl(\!\int_{\Omega} \phi(\vx;\vtheta)\,\mathrm{d}\vx\Bigr)^{2}
\;-\; \int_{\Omega} f(\vx)\,\phi(\vx;\vtheta)\,\mathrm{d}\vx,
\]
where $\vtheta$ denotes all network parameters and $\Omega = [0,1]^d$ is the domain. Proposition 1 in \cite{lu2021priori} establishes that
\[
c\,\bigl\lVert \phi(\,\cdot\,;\vtheta)-u^* \bigr\rVert_{H^1(\Omega)}
\le
\mathcal{E}_D(\vtheta)
\le
C\,\bigl\lVert \phi(\,\cdot\,;\vtheta)-u^* \bigr\rVert_{H^1(\Omega)},
\]
where \(u^*(\vx)\) denotes the exact solution of \eqref{PDE}, and \(c,C>0\) are constants independent of the neural network. Here, the Sobolev norm is defined by
\[
\lVert v \rVert_{H^1(\Omega)} \;:=\; \Biggl(\sum_{|\boldsymbol{\alpha}|\le 1} \bigl\lVert D^{\boldsymbol{\alpha}} v \bigr\rVert_{L^2(\Omega)}^2\Biggr)^{1/2}.
\]
For physics-informed neural networks (PINNs) \cite{raissi2019physics},
the residual loss can likewise be controlled by an \(H^{2}\)-type loss:
by the trace theorem and the continuity of the differential operator in the
underlying PDE, the PINN loss is bounded above by the corresponding
\(H^{2}\)-norm of the error.
A detailed proof is given in \cite[Lemma 4.3]{yang2024deeper}.

There have been several works on applying neural networks to approximate target functions in Sobolev spaces, measured by Sobolev norms \cite{guhring2020error,guhring2021approximation,yang2023nearly,yang2023nearlys,yang2024deeper,yang2024near,abdeljawad2022approximations,shen2022approximation,de2021approximation}. In \cite{guhring2020error,guhring2021approximation,abdeljawad2022approximations,shen2022approximation,de2021approximation}, the approximation rates achieved by continuous function approximators are nearly optimal in the sense of \cite{devore1989optimal}, but they do not improve upon classical methods for solving PDEs (e.g., finite element or spectral methods). In other words, those works provide approximation rates that are comparable to traditional PDE solvers, but they do not explain why one would prefer neural networks in this context. For a more detailed discussion of continuous versus discontinuous approximators, see \cite{siegel2022optimal,yang2024near,yarotsky2020phase}.

On the other hand, \cite{yang2023nearly,yang2024deeper,yang2024near,shen2022approximation} establish optimal approximation rates of deep neural networks measured by Sobolev norms—rates that outperform continuous approximators. We refer to this phenomenon as \emph{super-convergence} \cite{shen2019nonlinear,shen2020deep,shen2022optimal,lu2021deep}. However, these results are all based on the rectified linear unit (ReLU) activation function \cite{glorot2011deep}, whose weak derivative is piecewise constant and hence cannot capture smooth features accurately. For example, consider the one‐dimensional Poisson problem
\begin{equation}\label{eq:poisson}
  \Delta u(x) = -\pi^2 \sin(\pi x), 
  \quad
  x\in[-1,1],
  \qquad
  u(-1)=u(1)=0.
\end{equation}
As shown in Figure~\ref{fig:compare-b}, although a deep ReLU network can approximate the solution \(u(x)\) very accurately, its approximation of the first‐order derivative \(u'(x)\) exhibits large discrepancies.

\begin{figure}[!th]
\centering
\begin{subfigure}[t]{0.43\textwidth}
    \centering
    \includegraphics[width=0.99\textwidth]{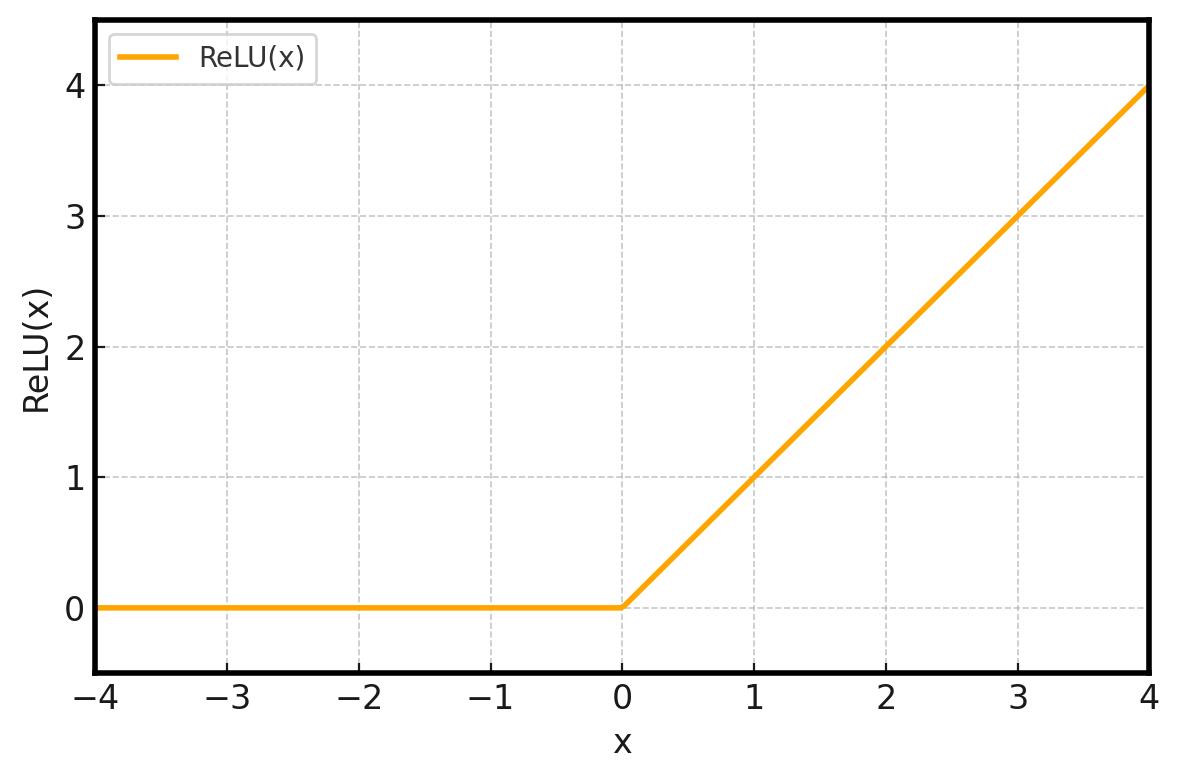}
    \caption{ReLU}
    \label{fig:compare-a}
\end{subfigure}
\hfill
\begin{subfigure}[t]{0.50\textwidth}
    \centering
    \includegraphics[width=0.99\textwidth]{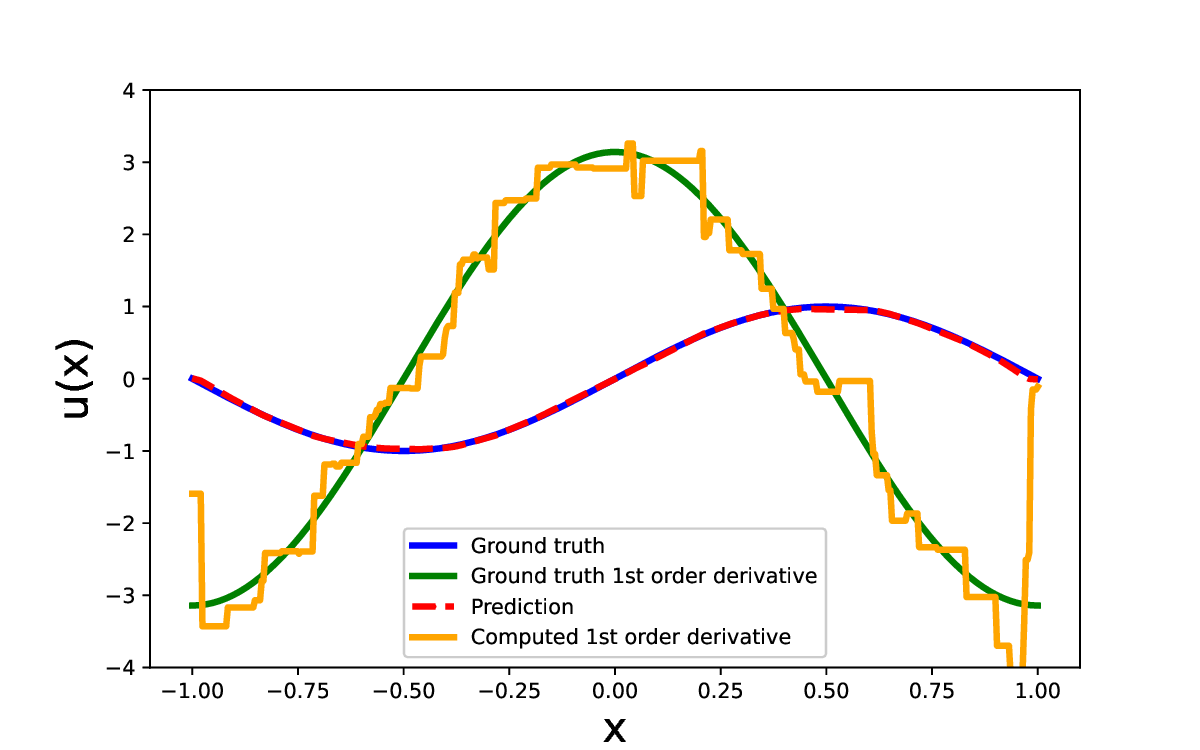}
    \caption{Deep ReLU networks to solve \eqref{eq:poisson}}
    \label{fig:compare-b}
\end{subfigure}
\caption{(a) The ReLU activation function. (b) Approximation of \(u\) and \(u'\) by a deep ReLU network for the Poisson problem \eqref{eq:poisson}.}
\label{fig:compare}
\end{figure}

To overcome this issue, \cite{yang2023nearlys} combines ReLU and the square of ReLU to smooth the network; however, this approach introduces two activation functions. In practice, most neural networks are built with a single activation function to simplify network design and training. Therefore, in this paper, we establish approximation results for deep neural networks with single smooth activation functions.

In order to establish super-convergence rates for deep neural networks with general activation functions, one might first apply a $W^{m,\infty}$‐approximation result for deep ReLU or the power of ReLU networks, and then use a fixed‐size subnetwork with a smooth activation (e.g.\ GELU) to approximate the ReLU or the power of ReLU activation itself to arbitrarily small error. By composing these two approximations, one would obtain super-convergence for deep networks with general smooth activations, since the additional cost of the fixed‐size subnetwork does not change the overall network size. This idea and framework were developed in \cite{zhang2024deep} for the $L^\infty$‐norm. For instance, recall that
\[
\mathrm{GELU}(x) \;=\; \frac{x}{\sqrt{2\pi}}\int_{-\infty}^x e^{-t^2/2}\,\mathrm{d}t.
\]
It is well known that for any fixed $M>0$, 
\[
\lim_{K\to\infty}\Bigl\|\tfrac{\mathrm{GELU}(Kx)}{K} - \mathrm{ReLU}(x)\Bigr\|_{L^\infty([-M,M])} \;=\; 0.
\]
However, this convergence does not hold in the $W^{m,\infty}$‐sense, even for $m=1$, because ReLU is not differentiable at the origin. In particular, no fixed‐size GELU subnetwork can be uniformly $W^{1,\infty}$‐close to a ReLU activation. Instead, one only has
\[
\lim_{K\to\infty}\Bigl\|\tfrac{\mathrm{GELU}(Kx)}{K} - \mathrm{ReLU}(x)\Bigr\|_{W^{1,p}([-M,M])} = 0
\qquad
\text{for each } p\in[1,\infty).
\]
But this weaker $W^{1,p}$‐approximation is not sufficient to recover the argument in \cite{zhang2024deep}, even if the final goal is only a $W^{1,p}$‐approximation. Their key composition estimate
\begin{equation}\label{eq:composition‐est}
\|\,g_1\circ f - g_2\circ f\|_{W^{1,p}(\Omega)}
\;\le\;
\|\,g_1 - g_2\|_{W^{1,\infty}(I)}\Bigl(|\Omega|^{1/p} + \|\nabla f\|_{L^p(\Omega)}\Bigr),
\end{equation}
requires a $W^{1,\infty}$‐bound on $g_1 - g_2$ over
\[
I \;\supset\; f(\Omega) \;=\; \{\,f(x): x\in\Omega\}.
\]
One cannot replace $\|g_1 - g_2\|_{W^{1,\infty}(I)}$ with $\|g_1 - g_2\|_{W^{1,p}(I)}$ in \eqref{eq:composition‐est}. Consequently, the method of \cite{zhang2024deep} fails even if one aims only for a final $W^{1,p}$‐approximation. Therefore, in the present work, we take a different route: rather than reducing to the ReLU case, we construct a single deep neural network using the chosen smooth activation (e.g., GELU) and prove its super-convergence directly in $W^{m,\infty}$.

The most comprehensive results on super approximation in the $W^{m,\infty}$ norm for smooth functions are given in~\cite{yang2022approximation,yang2023nearlys}. These works rely on two key ideas: (1) constructing smooth partitions of unity using ${\rm ReLU}^m$ activations; and (2) employing ReLU networks for bit-extraction. Building on these insights, we show that any activation function satisfying the following two properties yields super-convergence in $W^{m,\infty}(\Omega)$.
Specifically, for any \(M,\varepsilon,\delta>0\), there exist \(\sigma\)-networks \(\phi_1\) and \(\phi_2\) with fixed depth and width, independent of \(M\), \(\delta\), and \(\varepsilon\), while the network parameters may depend on \(M\), \(\delta\), and \(\varepsilon\), such that
\begin{enumerate}[label=(\roman*)]
  \item\label{itm:phi1}
  \[
    \bigl\|\phi_{1}-\mathrm{ReLU}^{m+1}\bigr\|_{W^{m,\infty}([-M,M])}
    \le \varepsilon.
  \]

  \item\label{itm:phi2}
  \[
    \bigl\|\phi_{2}-\mathrm{ReLU}\bigr\|_{L^{\infty}([-M,M])}
    \le \varepsilon,
    \qquad
    \bigl\|\phi_{2}-\mathrm{ReLU}\bigr\|_{W^{m,\infty}([-M,-\delta]\cup[\delta,M])}
    \le \varepsilon.
  \]
\end{enumerate}
\begin{remark}
    Here we would like to remark that, as \(\delta\to 0\), \(\varepsilon\to 0\), or \(M\to\infty\), the network parameters may diverge. This phenomenon suggests that the super-convergent neural network approximations may require large parameter values. Based on our current proof strategy, this appears to be unavoidable.
\end{remark}

These two properties of the activation function $\sigma(x)$ offer an intuitive pathway toward explaining the super-convergence rate from a constructive perspective. Specifically, they enable a three-stage network construction: First, we approximate characteristic step functions by leveraging the fact that any such activation can approximate \(\mathrm{ReLU}\) uniformly in the \(W^{m,\infty}\)-norm on \([-M,-\delta]\cup[\delta,M]\); Second, we build a fixed-size subnetwork that approximates \(\mathrm{ReLU}^{m+1}\), and by combining this with the polynomial-approximation result of \cite{he2023expressivity}, we obtain fixed-size networks for arbitrary polynomials on \([-M,M]\); Finally, we merge these subnetworks to implement a smooth partition of unity over the domain. Applying the Bramble–Hilbert Lemma \cite[Lemma~4.3.8]{brenner2008mathematical} on each cell of the partition then completes the proof of Theorem~\ref{firstmian}.

It is not difficult to verify that the activation function \(\operatorname{ReLU}^{m+1}\) satisfies the two key properties stated above; see Chapter~\ref{reluk} for more details. By contrast, for more general activation functions, such as \(\tanh\), \(\operatorname{GELU}\), and others, verifying these properties directly is generally less transparent. For this reason, we introduce two more general and easily checkable sufficient conditions on the activation function. Their role is to ensure that the activation can approximate \(\operatorname{ReLU}\) and \(\operatorname{ReLU}^{m+1}\) well enough for the super-convergence argument to go through. Thus, the proof mechanism itself does not change; rather, these conditions provide a practical criterion that makes the theory applicable to a wider class of activation functions.

A key motivation comes from the identity
\begin{equation}\label{eq:defReLUK}
\mathrm{ReLU}^k(x) = x^k H(x),
\quad\text{where}\quad
H(x) = \mathrm{ReLU}^0(x) :=
\begin{cases}
0, & x \le 0, \\
1, & x > 0.
\end{cases}\notag
\end{equation}
This inspires the following unified framework for establishing super-convergence in the $W^{m,\infty}$ norm.
\begin{condition}[Global: $m$-th order quasi-decay]\label{condition1}
An activation function $\sigma:\mathbb R\to\mathbb R$ is said to satisfy the
global \(m\)-th order quasi-decay condition if for every integer $0\le k\le m$ and every $x\neq0$,
\begin{equation}\label{eq:Cond1}
   \left|\sigma^{(k)}(x)-H^{(k)}(x)\right|
   \le
   \min \left\{\dfrac{C}{|x|^{k+1}}, G\right\},
\end{equation}
where \(C>0\) and \(G>0\) are fixed constants.
\end{condition}
\begin{remark}\label{rmk:condition1}
The representational power of deep neural networks remains invariant under affine transformations of the activation function. Specifically, shifting and scaling an activation function $\sigma(x)$ does not alter the network’s expressive capacity. Accordingly, we say that an activation function $\sigma(x)$ satisfies Condition~\ref{condition1} in the affine-invariant sense if there exist constants $\alpha, \beta \in \mathbb{R}$ such that the transformed function $\widetilde{\sigma}(x) = \beta (\sigma(x)+\alpha)$ satisfies Condition~\ref{condition1}.

A more direct and sufficient approach to verify this condition is to ensure both that $\sigma \in W^{m,\infty}_{\mathrm{loc}}(\mathbb{R})$ and that the following decay estimate holds:
$$
\left|\sigma^{(k)}(x) - H^{(k)}(x)\right| \le \frac{C}{|x|^{k+1}} \quad \text{as } x \to \pm\infty,
$$
for all $k \le m$.

Here we would like to point out a limitation of this condition. Roughly speaking, it requires that \(\sigma\) (or later \(\sigma/x\)) behave approximately like a constant or a linear function for large inputs. In other words, this is an asymptotic condition on the activation function at infinity. Such a condition is satisfied by many commonly used activation functions, but it excludes \(\operatorname{ReLU}^k\) for \(k\ge 2\). We emphasize, however, that this does not mean that \(\operatorname{ReLU}^k\)-networks cannot achieve super-convergence. Rather, \(\operatorname{ReLU}^k\) falls outside this general sufficient condition and therefore needs to be treated separately. We discuss this case in detail in Chapter~\ref{reluk}.
\end{remark}

\begin{condition}[Local: nonlinear behavior]\label{condition2}
A function $\sigma$ is said to exhibit nonlinear behavior if there exist $a \in \mathbb{R}$ and $\delta_* > 0$ such that
\[
\sigma \in C^{\infty}(a - \delta_*, a + \delta_*) \quad \text{and} \quad \sigma''(a) \ne 0.
\]
In other words, $\sigma$ is locally smooth and non-affine over some interval.
\end{condition}

Under either Conditions~\ref{condition1}--\ref{condition2} or Items~\ref{itm:phi1}--\ref{itm:phi2}, we establish the following main theorem.

\begin{theorem}\label{thm:mian}
Suppose that one of the following holds:
\begin{itemize}
\item[(a)] \(\sigma\) satisfies Condition~\ref{condition2}, and either \(\sigma(x)\) or
    \(
    \frac{\sigma(x)-\sigma(0)}{x}
    \)
    is well defined on \(\mathbb{R}\) and satisfies Condition~\ref{condition1};
    \item[(b)] \(\sigma\) satisfies Items~\ref{itm:phi1} and \ref{itm:phi2}.
\end{itemize}
Then neural networks with activation \(\sigma\) achieve super-convergence rates in \(W^{m,\infty}\).

More precisely, for any \(f \in W^{n,\infty}(\Omega)\) with \(0 \le m < n\), and for any \(N, L \in \mathbb{N}_+\), there exists a \(\sigma\)-neural network \(\phi\) with depth \(C_1 L \log L\) and width \(C_2 N \log N\) such that
\begin{equation}\label{eq:maintheorem}
\|f - \phi\|_{W^{m,\infty}(\Omega)} \le C_0 \|f\|_{W^{n,\infty}(\Omega)} N^{-2(n-m)/d} L^{-2(n-m)/d},
\end{equation}
where the constants \(C_0, C_1, C_2 > 0\) are independent of \(N\) and \(L\).
\end{theorem}
\begin{remark}
    For notational simplicity in the proofs of Chapter \ref{apprch}, we impose the additional assumption \(\sigma(0)=0\). This assumption is made only to simplify the exposition and does not affect the generality of the theorem statement.
\end{remark}

The approximation result in Theorem~\ref{thm:mian} exhibits a super-convergence phenomenon: it outperforms classical continuous approximation methods such as finite element and spectral methods. Indeed, the typical approximation rate achieved by such methods is \(M^{-\frac{n-m}{d}}\), where \(n\) denotes the Sobolev regularity of the target function, \(m\) is the order of the norm used to measure the error, \(d\) is the input dimension, and \(M\) is the number of degrees of freedom. This rate is known to be optimal for continuous approximators; see \cite{devore1989optimal}. By contrast, Theorem~\ref{thm:mian} yields an error bound of order \((NL)^{-\frac{2(n-m)}{d}}\). Since the total number of parameters scales as \(\mathcal{O}(N^2L)\), up to logarithmic factors, it follows that when the width \(N\) is fixed, we have \(M\sim L\), and thus the approximation rate becomes \(L^{-\frac{2(n-m)}{d}}\). This is strictly faster than the classical rate \(L^{-\frac{n-m}{d}}\). The key reason for this improvement is that deep neural networks form a highly nonlinear and discontinuous approximation class; see \cite{siegel2022optimal}.

Based on Theorem~\ref{thm:mian}, super-convergence can be obtained whenever the activation function satisfies either Conditions~\ref{condition1} and \ref{condition2} or Items~\ref{itm:phi1} and \ref{itm:phi2}. In general, Conditions~\ref{condition1} and \ref{condition2} are easier to verify and are satisfied by many popular activation functions, as discussed in Section~\ref{further}. For this reason, in Section~\ref{apprch} we first focus on applying these two conditions to establish super-convergence. On the other hand, some activation functions, such as \(\operatorname{ReLU}^k\) with \(k\ge 2\), do not satisfy Conditions~\ref{condition1} and \ref{condition2}; see Remark~\ref{rmk:condition1}. Even so, they still satisfy Items~\ref{itm:phi1} and \ref{itm:phi2}, and therefore Theorem~\ref{thm:mian} shows that they also achieve super-convergence. This case is discussed in Section~\ref{reluk}.
For convenience, Section~\ref{preliminaries} collects several lemmas that will be used throughout the paper.

\section{Further Discussions}\label{further}
In this section, we further elaborate on several aspects related to the main theorem, focusing on commonly used activation functions and relevant literature.

\subsection{Activation functions satisfying the two conditions}

We summarize the corresponding super-convergence results in \(W^{m,\infty}(\Omega)\) for different activation functions in Table~\ref{tab:summary}. Details on verifying whether these activation functions satisfy the Conditions \ref{condition1} and \ref{condition2} can be found in Appendix~\ref{activtion section}.
\begin{table}[h!]
\centering
\begin{tabular}{|c|c|c|}
\hline
\(W^{m,\infty}(\Omega)\) in~\eqref{eq:maintheorem} & Applicable activations & Reference \\
\hline
\(m = 0\) & All popular activations & \cite{zhang2024deep} \\
\hline
\(m \le \max\{1, n - 1\}\) & \makecell{ReLU, LeakyReLU, HardTanh, HardSigmoid,\\ ELU (\(\alpha \ne 1\)), SELU (\(\alpha \ne 1\))} & \makecell{\cite{yang2023nearly} \\ This work} \\
\hline
\(m \le \max\{2, n - 1\}\) & Softsign, CELU, ELU (\(\alpha = 1\)), SELU (\(\alpha = 1\)) & This work \\
\hline
\(0 \le m \le n-1\) & \makecell{Sigmoid, Tanh, \(\arctan\), dSiLU, SRS,\\ Softplus, SiLU (Swish), Mish, GELU} & This work \\
\hline
\end{tabular}
\caption{Summary of super-convergence results in \(W^{m,\infty}(\Omega)\) from~\eqref{eq:maintheorem} for \(f\in W^{n,\infty}(\Omega)\) (\(n \ge 1\)), corresponding to various activation functions and references.}
\label{tab:summary}
\end{table}

\subsection{Related Work}
There has been extensive research on the expressivity and approximation capabilities of neural networks, spanning different architectures, error measures, and activation functions. This work focuses on Sobolev-type error estimates for deep fully-connected neural networks with both commonly used and general activation functions.

A key distinction in the literature lies between the approximation properties of shallow (single-hidden-layer) and deep networks. From the 1990s onward, a series of foundational works~\cite{hornik1989multilayer, cybenko1989approximation, jones1992simple, leshno1993multilayer, ellacott1994aspects, pinkus1999approximation} established the universal approximation capabilities of shallow networks. Quantitative approximation rates—particularly in the $L^2$ norm—were subsequently investigated in works such as~\cite{barron1993universal, klusowski2018approximation, e2019priori, siegel2020approximation, e2021kolmogorov, siegel2022sharp, siegel2022high}, which also characterized the function classes that admit efficient approximation. These include the Barron space~\cite{barron1993universal,e2019priori}, motivated by integral representations, and generalized variation spaces~\cite{siegel2022sharp, siegel2022high}, which are particularly relevant for activations such as ${\rm ReLU}^k$.
For more general activation functions, a classical work~\cite{mhaskar1996neural} analyzed approximation rates in the $L^\infty$ norm for a broad class of activations, while~\cite{siegel2020approximation} established error bounds in Sobolev norms $H^m$ for all non-polynomial activation functions, highlighting their approximation capacity beyond the ReLU-type framework.

In contrast, approximation theory for deep neural networks is a more recent development. As discussed in the introduction, several works have studied both continuous and discontinuous approximations in Sobolev norms. Our focus here is on super-convergence phenomena for general activation functions, extending beyond ReLU and ${\rm ReLU}^k$. For instance,~\cite{lu2021deep, guhring2020error, muller2022error, de2021approximation, guhring2021approximation} investigate the approximation of smooth functions or functions in Sobolev spaces using deep networks, with error measured in $L^p(\Omega)$ or $W^{s,p}(\Omega)$. However, these results are often suboptimal or restricted to $L^p$ norms. For example,~\cite{lu2021deep} relies on Taylor expansions to approximate smooth functions, which do not directly generalize to Sobolev spaces, while~\cite{zhang2024deep} extends these ideas to more general activations.
In~\cite{guhring2020error}, the authors show that ReLU networks can approximate functions in $W^{1,p}(\Omega)$, though with convergence rates similar to classical methods such as finite element approximations. Further,~\cite{he2020relu, he2023deep} prove optimal approximation results in $W^{k,p}(\Omega)$ for $k = 0,1$ and $p < \infty$, based on the exact representation of finite element functions using ReLU and ReLU-ReLU$^2$ activations.

In this work, we establish super-convergence results for deep fully-connected networks of arbitrary depth and width, without relying on asymptotic arguments, and under general activation functions. The derived error estimates hold in any $W^{m,p}$ norm and represent, to the best of our knowledge, the most general and comprehensive approximation theory for deep networks with commonly-used and also general activations. These results have broad applicability, including rigorous generalization error bounds for PINNs and the Deep Ritz method.

\section{Preliminaries}\label{preliminaries}
		Let us summarize all basic notations used in the DNNs as follows:
		
		\textbf{1}. Matrices are denoted by bold uppercase letters. For example, $\vA\in\sR^{m\times n}$ is a real matrix of size $m\times n$ and $\vA^\T$ denotes the transpose of $\vA$.
		
		\textbf{2}. Vectors are denoted by bold lowercase letters. For example, $\vv\in\sR^n$ is a column vector of size $n$. Furthermore, denote $\vv(i)$ as the $i$-th elements of $\vv$.
		
		\textbf{3}. For a $d$-dimensional multi-index $\valpha=[\alpha_1,\alpha_2,\cdots\alpha_d]\in\sN^d$, we denote several related notations as follows: $(a)~ |\boldsymbol{\alpha}|=\left|\alpha_1\right|+\left|\alpha_2\right|+\cdots+\left|\alpha_d\right|$; $(b)~\boldsymbol{x}^\alpha=x_1^{\alpha_1} x_2^{\alpha_2} \cdots x_d^{\alpha_d},~ \boldsymbol{x}=\left[x_1, x_2, \cdots, x_d\right]^\T$; $ (c)~\boldsymbol{\alpha} !=\alpha_{1} ! \alpha_{2} ! \cdots \alpha_{d} !.$
		
		\textbf{4}. Assume $\vn\in\sN_+^n$, then $f(\vn)=\fO(g(\vn))$ means that there exists positive $C$ independent of $\vn,f,g$ such that $f(\vn)\le Cg(\vn)$ when all entries of $\vn$ go to $+\infty$.

		\textbf{5}. Define $L,N\in\sN_+$, $N_0=d$ and $N_{L+1}=1$, $N_i\in\sN_+$ for $i=1,2,\ldots,L$, then a $\sigma$ neural network $\phi$ with the width $N$ and depth $L$ can be described as follows:\[\boldsymbol{x}=\tilde{\boldsymbol{h}}_0 \stackrel{W_1, b_1}{\longrightarrow} \boldsymbol{h}_1 \stackrel{\sigma}{\longrightarrow} \tilde{\boldsymbol{h}}_1 \ldots \stackrel{W_L, b_L}{\longrightarrow} \boldsymbol{h}_L \stackrel{\sigma}{\longrightarrow} \tilde{\boldsymbol{h}}_L \stackrel{W_{L+1}, b_{L+1}}{\longrightarrow} \phi(\boldsymbol{x})=\boldsymbol{h}_{L+1},\] where $\vW_i\in\sR^{N_i\times N_{i-1}}$ and $\vb_i\in\sR^{N_i}$ are the weight matrix and the bias vector in the $i$-th linear transform in $\phi$, respectively, i.e., $\boldsymbol{h}_i:=\boldsymbol{W}_i \tilde{\boldsymbol{h}}_{i-1}+\boldsymbol{b}_i, ~\text { for } i=1, \ldots, L+1$ and $\tilde{\boldsymbol{h}}_i=\sigma\left(\boldsymbol{h}_i\right),\text{ for }i=1, \ldots, L.$ In this paper, an DNN with the width $N$ and depth $L$, means
		(a) The maximum width of this DNN for all hidden layers less than or equal to $N$.
		(b) The number of hidden layers of this DNN less than or equal to $L$.

        All the analysis in this paper is based on Sobolev spaces:\begin{definition}[Sobolev Spaces]
			Denote $\Omega$ as $[0,1]^d$, $D$ as the weak derivative of a single variable function and $D^{\valpha}=D^{\alpha_1}_1D^{\alpha_2}_2\ldots D^{\alpha_d}_d$ as the partial derivative where $\valpha=[\alpha_{1},\alpha_{2},\ldots,\alpha_d]^T$ and $D_i$ is the derivative in the $i$-th variable. Let $n\in\sN$ and $1\le p\le \infty$. Then we define Sobolev spaces\[W^{n, p}(\Omega):=\left\{f \in L^p(\Omega): D^{\valpha} f \in L^p(\Omega) \text { for all } \boldsymbol{\alpha} \in \sN^d \text { with }|\boldsymbol{\alpha}| \leq n\right\}\] with a norm \[\|f\|_{W^{n, p}(\Omega)}:=\left(\sum_{0 \leq|\alpha| \leq n}\left\|D^{\valpha} f\right\|_{L^p(\Omega)}^p\right)^{1 / p},\] if $p<\infty$, \[\|f\|_{W^{n, \infty}(\Omega)}:=\max_{0 \leq|\alpha| \leq n}\left\|D^{\valpha} f\right\|_{L^\infty(\Omega)}.\]
			Furthermore, for $\vf=(f_1,f_2,\ldots,f_d)$, $\vf\in W^{m,\infty}(\Omega,\sR^d)$ if and only if $ f_i\in W^{m,\infty}(\Omega)$ for each $i=1,2,\ldots,d$ and \[\|\vf\|_{W^{m,\infty}(\Omega,\sR^d)}:=\max_{i=1,\ldots,d}\{\|f_i\|_{W^{m,\infty}(\Omega)}\}.\]
		\end{definition}

\begin{lemma}\label{lem:Wm_inf_bounds}
Let $\Omega\subset\mathbb{R}^d$ be a bounded domain and $m\in\mathbb{N}_+$.  Then for all sufficiently smooth functions on~$\Omega$ the following hold:

\begin{enumerate}
  \item \emph{(Product estimate)}  
    If $f,g\in W^{m,\infty}(\Omega)$, then
    \[
      \|f\,g\|_{W^{m,\infty}(\Omega)}
      \;\le\;
      \sum_{k=0}^m \binom{m}{k}
        \|f\|_{W^{k,\infty}(\Omega)}\,
        \|g\|_{W^{m-k,\infty}(\Omega)}
      \;\le\;
      2^m\,\|f\|_{W^{m,\infty}(\Omega)}\,
           \|g\|_{W^{m,\infty}(\Omega)}.
    \]

  \item \emph{(Composition difference)}  
    Let $F\in W^{m+1,\infty}(I)$ on an interval $I\supset g(\Omega)\cup h(\Omega)$ and $g,h\in W^{m,\infty}(\Omega)$.  Then
    \[
      \|F\circ g - F\circ h\|_{W^{m,\infty}(\Omega)}
      \;\le\;
     C_{m}\bigl(\|g\|_{W^{m,\infty}},\,\|h\|_{W^{m,\infty}}\bigr)\,
      \|F\|_{W^{m+1,\infty}(I)}\,
      \|g - h\|_{W^{m,\infty}(\Omega)},
    \]
    where $C_{m}(\|g\|_{W^{m,\infty}},\|h\|_{W^{m,\infty}})$ depends only on $m$, $\|g\|_{W^{m,\infty}(\Omega)}$, and $\|h\|_{W^{m,\infty}(\Omega)}$.

  \item \emph{(Post‐composition difference)}  
    If $f_1,f_2\in W^{m,\infty}(I)$ and $g\in W^{m,\infty}(\Omega)$ with $g(\Omega)\subset I$, then
    \[
      \|f_1\circ g - f_2\circ g\|_{W^{m,\infty}(\Omega)}
      \;\le\;
      C_{m}(\,\|g\|_{W^{m,\infty}(\Omega)})\,
      \|f_1 - f_2\|_{W^{m,\infty}(I)},
    \]where $C_{m}(\|g\|_{W^{m,\infty}})$ depends only on $m$ and $\|g\|_{W^{m,\infty}(\Omega)}$.
\end{enumerate}
\end{lemma}

\begin{proof}
(i) Follows directly from Leibniz’s rule and the binomial theorem.

(ii) We proceed by induction on $m$.

For $m=0$, the Lipschitz bound gives
\[
  |F(g(\vx)) - F(h(\vx))|
  \;\le\;
  \|F\|_{W^{1,\infty}(I)}\,|g(\vx) - h(\vx)|,
\]
so the claim holds for $m=0$.

Assume it holds for all orders up to $m-1$.  Then for general $m$,
\[
  \|F\circ g - F\circ h\|_{W^{m,\infty}}
  \;\le\;
  \|F\circ g - F\circ h\|_{L^{\infty}}
  \;+\;\sum_{i=1}^d
    \bigl\|\partial_{x_i}(F\circ g - F\circ h)\bigr\|_{W^{m-1,\infty}}.
\]
The $L^\infty$‐term is bounded by $\|F\|_{W^{1,\infty}}\|g-h\|_{L^\infty}$.  For each derivative,
\[
  \partial_{x_i}(F\circ g - F\circ h)
  = F'\circ g\,\partial_{x_i}g
    -F'\circ h\,\partial_{x_i}h
  = (F'\circ g - F'\circ h)\,\partial_{x_i}g
    +F'\circ h\,(\partial_{x_i}g-\partial_{x_i}h).
\]
By the product estimate (i) and the induction hypothesis,
\begin{align*}
  &\|(F'\circ g - F'\circ h)\,\partial_{x_i}g\|_{W^{m-1,\infty}}\\
  \le&
    2^{m-1}\,\|F'\circ g - F'\circ h\|_{W^{m-1,\infty}}\,
    \|\partial_{x_i}g\|_{W^{m-1,\infty}}\\
 \le&
    2^{m-1}C_{m-1}\bigl(\|g\|_{W^{m-1,\infty}},\|h\|_{W^{m-1,\infty}}\bigr)
    \|F'\|_{W^{m,\infty}(I)}
    \|g-h\|_{W^{m-1,\infty}}
    \|\partial_{x_i}g\|_{W^{m-1,\infty}}\\
  \le&
    C^*_m(\|g\|_{W^{m,\infty}},\|h\|_{W^{m-1,\infty}})
    \|F\|_{W^{m+1,\infty}(I)}
    \|g-h\|_{W^{m-1,\infty}}.
\end{align*}
Similarly, by (i),
\begin{align*}
  &\|F'\circ h\,(\partial_{x_i}g-\partial_{x_i}h)\|_{W^{m-1,\infty}}\le
    2^{m-1}\,\|F'\circ h\|_{W^{m-1,\infty}}
    \|\partial_{x_i}g-\partial_{x_i}h\|_{W^{m-1,\infty}}\\
  \le&
    C^{**}_m(\|h\|_{W^{m-1,\infty}})
    \|F\|_{W^{m+1,\infty}(I)}
    \|g-h\|_{W^{m,\infty}}.
\end{align*}
Combining these bounds yields the desired estimate for order~$m$.

(iii) Follows immediately by writing
\[
  f_1\circ g - f_2\circ g
  = (f_1-f_2)\circ g
\]
and applying the product estimate with a constant factor.
\end{proof}
The next lemma shows that neural networks with activation function \(\sigma\) satisfying Condition~\ref{condition2} can approximate \(x^2\) in the \(W^{m,\infty}\)-norm. The basic construction is related to Step~3 in the proof of Theorem~1 of \cite{leshno1993multilayer}. Our setting differs in two important aspects. First, since we only need approximation on a bounded interval, it is sufficient to assume local regularity of \(\sigma\); global smoothness on the whole real line is not required. Second, for the later quantitative analysis, we need a fixed-size neural network together with explicit \(C^m\)-error bounds, including control of higher-order derivatives. For this reason, we write the network in closed form and estimate the approximation error and its derivatives explicitly.

\begin{lemma}\label{sqr11}
For any $\varepsilon>0$, integer $m\ge 0$, and any bounded interval $[-M,M]$, there exists a neural network
\(
  \phi(x)
\)
with the depth 1 and width $4$ with the activation function $\sigma$ satisfying Condition \ref{condition2}, such that
\[
  \bigl\|\,x^2 - \phi(x)\bigr\|_{C^m([-M,M])}
  \;\le\;\varepsilon.
\]
\end{lemma}

\begin{proof}
Invoking Condition~\ref{condition2}, without loss of generality we may assume that the distinguished point is \(a=0\). Then there exists \(\delta_*>0\) such that
\[
\sigma\in C^\infty((-\delta_*,\delta_*))
\qquad\text{and}\qquad
\sigma''(0)\neq 0.
\]
Define
\[
a_i:=\frac{\sigma^{(i)}(0)}{i!}, \qquad i=0,1,2.
\]
Then, for every \(x\in[-\delta,\delta]\subset(-\delta_*,\delta_*)\), Taylor's formula with integral remainder gives
\[
\sigma(x)=\sum_{i=0}^{2}a_i x^i+r(x)
:=\sum_{i=0}^{2}a_i x^i+\frac{1}{2}\int_0^x \sigma^{(3)}(t)(x-t)^2\,\mathrm{d}t,
\]
where
\[
a_2=\frac{\sigma''(0)}{2}\neq 0.
\]

Choose \(\lambda\in\mathbb{R}\setminus\{0,\pm1\}\), and define
\[
f_1(x):=\sigma(\lambda x)-\sigma(x).
\]
Then the constant term is canceled, and
\[
f_1(x)=a_1(\lambda-1)x+a_2(\lambda^2-1)x^2+\bigl(r(\lambda x)-r(x)\bigr).
\]
Next define
\[
f_2(x):=f_1(x)-\lambda f_1(x/\lambda).
\]
Then the linear term is also canceled, and
\[
f_2(x)=a_2(\lambda^2-1)\bigl(1-\lambda^{-1}\bigr)x^2+h_2(x),
\]
where
\[
h_2(x):=\bigl(r(\lambda x)-r(x)\bigr)-\lambda\bigl(r(x)-r(x/\lambda)\bigr).
\]
Moreover, \(h_2\) can be written as
\[
h_2(x)=\sum_{j=1}^{4} c_{1,j}\int_{0}^{c_{2,j}x}\sigma^{(3)}(t)\bigl(c_{2,j}x-t\bigr)^2\,\mathrm{d}t,
\]
for all \(|x|\le \delta/c_*\), where
\[
c_*:=\max_{1\le j\le 4}|c_{2,j}|=\max\{1,|\lambda|,|\lambda|^{-1}\},
\]
and the constants \(c_{1,j}\) and \(c_{2,j}\) depend only on \(\lambda\). Since \(\lambda\neq 0,\pm1\) and \(a_2\neq 0\), the coefficient of \(x^2\) is nonzero. Denote
\[
a_{2,2}:=a_2(\lambda^2-1)\bigl(1-\lambda^{-1}\bigr).
\]

Then, for any \(K>0\),
\[
x^2=\frac{K^2f_2(x/K)}{a_{2,2}}-\frac{K^2h_2(x/K)}{a_{2,2}}.
\]
Define
\[
\phi(x):=\frac{K^2f_2(x/K)}{a_{2,2}}.
\]
Then \(\phi\) is a depth-\(1\), width-\(4\) neural network with activation \(\sigma\).

The approximation error is
\begin{align}
R(x):=x^2-\phi(x)
&=-\frac{K^2h_2(x/K)}{a_{2,2}} \notag\\
&=\sum_{j=1}^4 c_{1,j}\frac{K^2}{a_{2,2}}
\int_0^{c_{2,j}x/K}\sigma^{(3)}(t)\Bigl(c_{2,j}\frac{x}{K}-t\Bigr)^2\,\mathrm{d}t \notag\\
&=\sum_{j=1}^4 c_{1,j}\frac{K^2}{a_{2,2}}
\int_0^x \sigma^{(3)}\Bigl(c_{2,j}\frac{s}{K}\Bigr)\Bigl(c_{2,j}\frac{1}{K}\Bigr)^3 (x-s)^2\,\mathrm{d}s \notag\\
&=\frac{1}{K}\sum_{j=1}^4 c_{1,j}\int_0^x
\sigma^{(3)}\Bigl(c_{2,j}\frac{s}{K}\Bigr)\frac{c_{2,j}^3}{a_{2,2}}(x-s)^2\,\mathrm{d}s.
\label{eq:Rerror-sqr}
\end{align}

Now choose
\[
K=\frac{c_*M}{\delta}.
\]
Then for every \(x\in[-M,M]\), we have \(|x/K|\le \delta/c_*\), so the above representation is valid. Moreover,
\[
\left|\frac{c_{2,j}s}{K}\right|
\le \frac{|c_{2,j}|M}{K}
\le \delta
<\delta_*,
\qquad s\in[-M,M],\quad j=1,2,3,4.
\]

For \(0\le p\le 2\), differentiating \eqref{eq:Rerror-sqr} under the integral sign yields
\[
\bigl|D^pR(x)\bigr|
=
\left|
\frac{1}{K}\sum_{j=1}^{4}c_{1,j}
\int_0^x
\sigma^{(3)}\!\Bigl(\tfrac{c_{2,j}s}{K}\Bigr)
\frac{p!\,c_{2,j}^3}{2!\,a_{2,2}}
(x-s)^{2-p}\,\mathrm{d}s
\right|.
\]
Hence
\[
\bigl|D^pR(x)\bigr|
\le
\frac{C}{K}\,\|\sigma\|_{C^3((-\delta_*,\delta_*))}\,M^{3-p}
=
\mathcal{O}(\delta),
\qquad 0\le p\le 2.
\]

If \(m\ge 3\), then for \(3\le p\le m\),
\[
\bigl|D^pR(x)\bigr|
=
\left|
\frac{1}{K}
\sum_{j=1}^{4}c_{1,j}
\Bigl(\tfrac{c_{2,j}}{K}\Bigr)^{p-3}
\sigma^{(p)}\!\Bigl(\tfrac{c_{2,j}x}{K}\Bigr)
\frac{p!\,c_{2,j}^{3}}{2!\,a_{2,2}}
\right|
\le
\frac{C}{K^{p-2}}
\max_{3\le q\le m}\|\sigma^{(q)}\|_{C((-\delta_*,\delta_*))}
=
\mathcal{O}(\delta^{p-2}),
\]
again because \(|c_{2,j}x/K|<\delta_*\).

Therefore, for any $\varepsilon>0$, by choosing \(\delta>0\) sufficiently small, we obtain
\[
\|\phi-x^2\|_{C^m([-M,M])}\le \varepsilon.
\]
Consequently, \(\phi\) is a depth-\(1\), width-\(4\) \(\sigma\)-network that approximates \(x^2\) within accuracy \(\varepsilon\) in the \(C^m\)-norm.
\end{proof}

Based on Lemma~\ref{sqr11}, we can establish the following lemmas step by step. We move the detailed proofs to Appendix~\ref{app lem}.

\begin{lemma}\label{time}
    For any $\varepsilon>0$, $m\in\sN_+$ and bounded domain $[-M,M]^2$, there is a neural network $\phi$ with depth 1 and width $12$ with the activation $\sigma$ satisfying Condition \ref{condition2}, such that \begin{equation}
        \|xy-\phi(x,y)\|_{C^m([-M,M]^2)}\le \varepsilon.\notag
    \end{equation}
\end{lemma}
\begin{lemma}\label{Id}
For any $\varepsilon>0$, integer $m\ge 0$, and any bounded interval $[-M,M]$, there exists a neural network 
\(
  \phi(x)
\) with depth 1 and width $8$
with the activation $\sigma$ satisfying Condition \ref{condition2}, such that
\[
  \bigl\|\,x - \phi(x)\bigr\|_{C^m([-M,M])}
  \;\le\;\varepsilon.
\]
\end{lemma}

Next, we prove that \(\sigma\)-neural networks can approximate \(x^p\), \(\prod_{i=1}^d x_i\), and \(\vx^\alpha\) in the \(W^{m,\infty}\)-norm. Although analogous \(C^m\)-approximation results could also be derived, this would require a corresponding \(C^m\)-version of Lemma~\ref{lem:Wm_inf_bounds}. Since such estimates are not needed for the subsequent proofs in the paper, we only present the \(W^{m,\infty}\)-results here.

\begin{lemma}\label{xm+1}
    For any $\varepsilon>0$, $m\in\sN_+$, $p\ge 2$ and bounded domain $[-M,M]$, there is a neural network $\phi$ with depth $p-1$ and width $20$ with the activation $\sigma$ satisfying Condition \ref{condition2}, such that \begin{equation}
        \|x^p-\phi(x)\|_{W^{m,\infty}([-M,M])}\le \varepsilon.\notag
    \end{equation} 
\end{lemma}

\begin{lemma}\label{times}
    For any $\varepsilon>0$, $d,m\in\sN_+$ and bounded domain $[-M,M]^d$, there is a neural network $\phi$ with depth $d-1$ and width $8d-4$ with the activation $\sigma$ satisfying Condition \ref{condition2}, such that \begin{equation}
        \|x_1x_2\ldots x_d-\phi(\vx)\|_{W^{m,\infty}([-M,M]^d)}\le \varepsilon.\notag
    \end{equation}
\end{lemma}
\begin{proof}
The proof is illustrated by the network architecture in Fig.~\ref{fig:xprod}.

\begin{figure}[ht]
    \centering
    \includegraphics[width=0.77\linewidth]{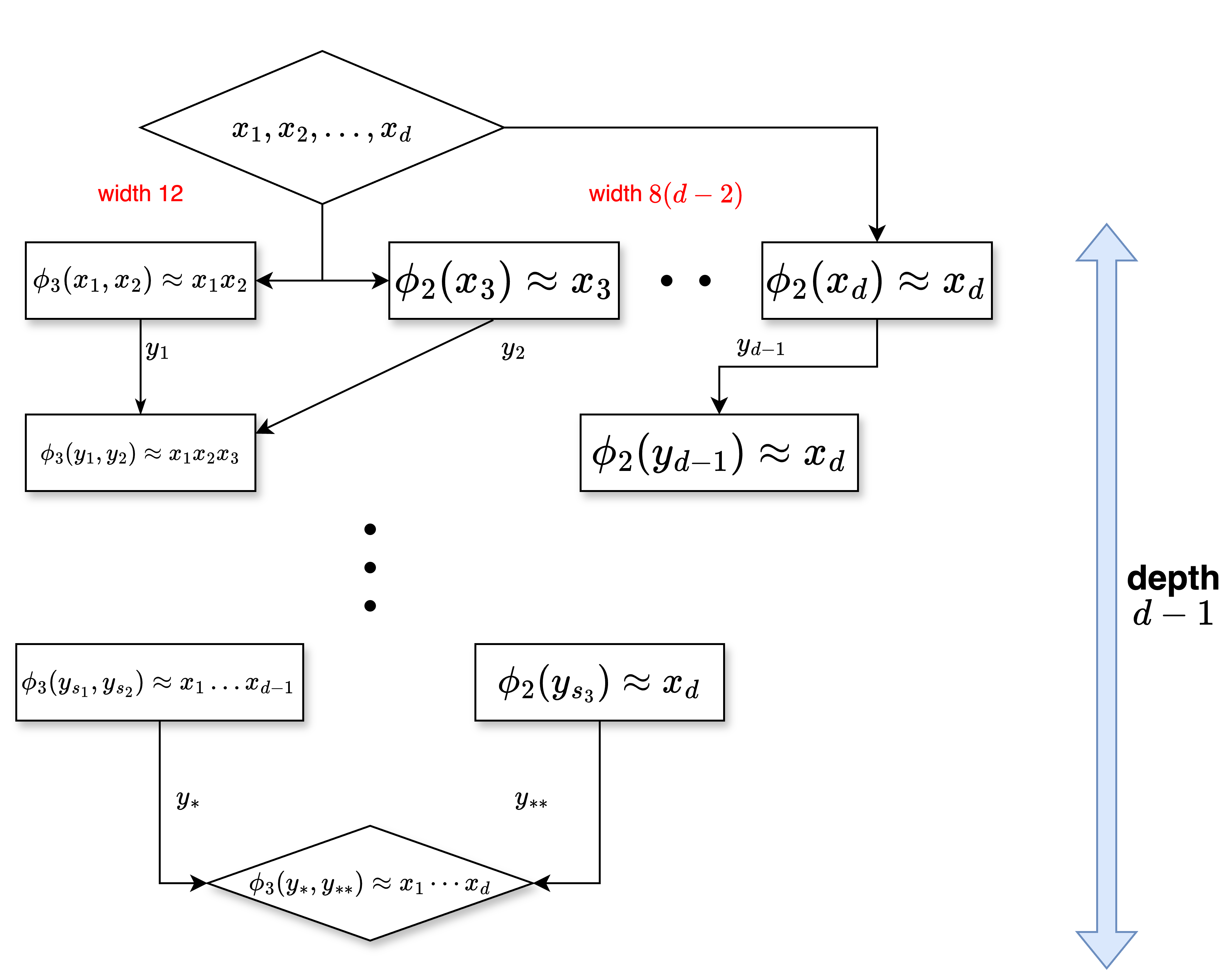}
    \caption{Network architecture used to approximate \(x_1x_2\ldots x_d\).}
    \label{fig:xprod}
\end{figure}

The detailed proof is given in Appendix~\ref{app lem}.
\end{proof}

\begin{remark}
When the dimension is large, the result in Lemma~\ref{times} can be improved to logarithmic depth in \(d\). The corresponding network architecture is illustrated in Fig.~\ref{fig:newxprod}.

\begin{figure}[ht]
    \centering
    \includegraphics[width=0.77\linewidth]{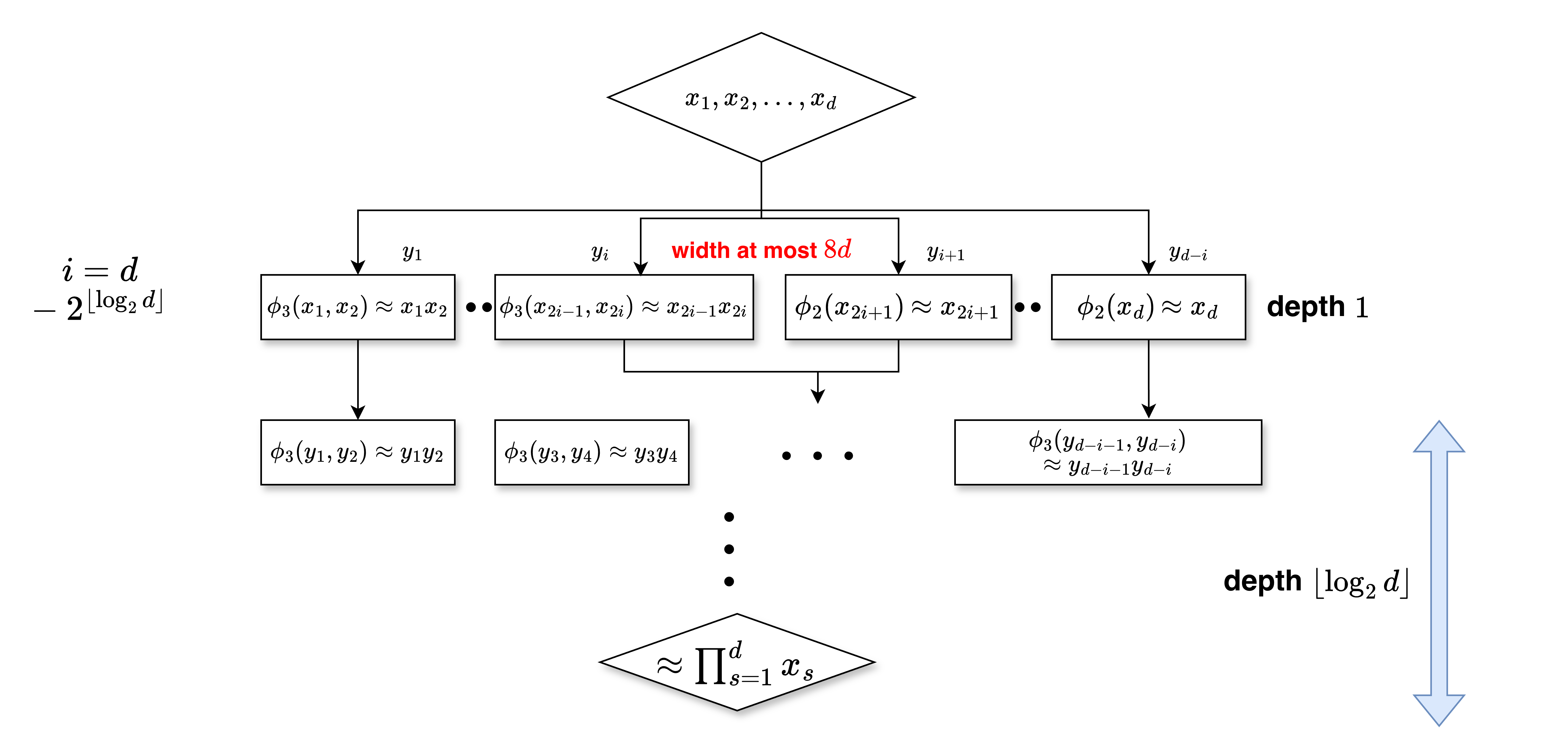}
    \caption{An alternative network architecture used to approximate \(x_1x_2\ldots x_d\).}
    \label{fig:newxprod}
\end{figure}

Set
\[
k:=2^{\lfloor \log_{2} d \rfloor},
\qquad
i:=d-k.
\]
Then \(0\le i<k\), and \(k\) is the largest power of two not exceeding \(d\).

In the first hidden layer, we reduce the list \((x_1,\dots,x_d)\) of length \(d\) to a list of length \(k\) as follows: we perform \(i\) adjacent multiplications
\[
(x_1,x_2),\ (x_3,x_4),\ \dots,\ (x_{2i-1},x_{2i}),
\]
each realized by one bilinear block, and pass the remaining inputs unchanged through identity blocks. Since each multiplication reduces the list length by one, after these \(i\) multiplications exactly \(k\) outputs remain. This layer uses
\[
12i+8(d-2i)\le 8d
\]
neurons.
\end{remark}

Similarly, we have that \begin{lemma}\label{timess} For any $\varepsilon>0$, $d,m,n\in\sN_+$, $|\valpha|\le n$ and bounded domain $[-M,M]^d$, there is a neural network $\phi$ with depth $n-1$ and width $8n-4$ with the activation $\sigma$ satisfying Condition \ref{condition2}, such that \begin{equation} \|\vx^{\valpha}-\phi(\vx)\|_{W^{m,\infty}([-M,M]^d)}\le \varepsilon.\notag \end{equation} \end{lemma}

\section{Construction of Deep Networks for \(W^{m,\infty}\)-Approximation}\label{apprch}
In this section, we construct a deep neural network with the activation function satisfying Conditions \ref{condition1} and \ref{condition2} to approximate a target function \(f\in W^{n,\infty}(\Omega)\), where \(\Omega=[0,1]^d\).  First, by applying \cite[Lemma 4.3.8]{brenner2008mathematical}, we approximate \(f\) on each subdomain \(\Omega_{\vm}\) by a local polynomial of degree \(n-1\):
\begin{equation}
  f_{J,\vm}(\vx)
  \;=\;
  \sum_{|\valpha|\le n-1} g_{f,\valpha,\vm}(\vx)\,\vx^{\valpha},\label{important equations}
\end{equation}
where each coefficient \(g_{f,\valpha,\vm}(\vx)\) is a piece-wise constant on \(\Omega_{\vm}\).  The construction of \(\{g_{f,\valpha,\vm}\}\) is given in Proposition \ref{fN}, and the proof appears in \cite{yang2023nearly}.

Each coefficient is first realised by a small ReLU subnetwork;  
Section~\ref{RE1} (via Proposition~\ref{non-di}) then replaces every such
ReLU block with an equally accurate \(W^{m,\infty}\)-block whose activation
satisfies Conditions~\ref{condition1}--\ref{condition2}.
Monomials are covered by Lemma~\ref{timess}, while Section~\ref{RE2} shows
that a fixed-width network with the same activation can approximate
\(\mathrm{ReLU}^{\,m+1}\) to arbitrary accuracy; since any polynomial can
be represented by a \(\mathrm{ReLU}^{\,m+1}\) network
\cite{he2023expressivity}, every local polynomial
\(p(\vx)=\sum_{\valpha} c_{\valpha}\,\vx^{\valpha}\) therefore admits a
\(W^{m,\infty}\)-accurate, condition-compliant realisation of fixed size.
With these ingredients in place, we build a smooth partition of unity
(also detailed in Section~\ref{RE2}) that stitches the local approximants
together, yielding a single deep network—still using the same
activation—that approximates \(f\) on the entire domain \(\Omega\) with
the desired \(W^{m,\infty}\) accuracy.

\subsection{Approximation of ReLU neural networks by neural networks with the activation function satisfying Conditions \ref{condition1} and \ref{condition2}}\label{RE1}
In \cite{yang2023nearly}, two specialized ReLU‐based architectures play a central role in achieving super-convergence.  The first is the bit‐extraction network of Bartlett et al.\ \cite{bartlett2019nearly}, which encodes $M^2$ arbitrary values with only $M$ parameters.  The second is a family of step‐function approximators built from shallow ReLU networks.  To extend these constructions to any activation satisfying Conditions \ref{condition1} and \ref{condition2}, we begin by proving that every ReLU network can be approximated arbitrarily well by a network whose activation function meets Conditions \ref{condition1} and \ref{condition2}.

After establishing the following proposition, we can construct a \(\sigma\)-neural network to approximate \(g_{f,\valpha,\vm}\) in \eqref{important equations} by using a ReLU neural network as an intermediate link. Moreover, because of the product estimate in Lemma~\ref{lem:Wm_inf_bounds}, it is also necessary to control the higher-order derivatives.

\begin{proposition}\label{non-di}
    For any $M>0$, $m\in\sN_+$ and a one-dimensional deep ReLU neural network of depth $L$ and width $N$ $\phi_{\rm{ReLU}}$ with non-differential point $\{x_i\}_{i=1}^M\subset\Omega=[-M,M]$, for any $1>\delta>0$, we denote that $\Omega_{\delta}=\Omega\backslash (\cup_{i=1}^M [-\delta+x_i,\delta+x_i])$, For any $\varepsilon>0$, there is a deep neural network $\phi_{\text{G}}$ whose activation function meets Conditions \ref{condition1} and \ref{condition2} with depth $2L$ and width $12N$ such that \begin{equation}
        \max\{\|\phi_{\rm{ReLU}}-\phi_{{G}}\|_{C(\Omega)},\|\phi_{\rm{ReLU}}-\phi_{{G}}\|_{W^{m,\infty}(\Omega_\delta)}\}\le \varepsilon.\notag
    \end{equation}
\end{proposition}

\begin{proof}
Set
\[
  \psi_K(x)
  =
  \begin{cases}
    \displaystyle \frac{\sigma(Kx)}{K}, & \text{if }\frac{\sigma(x)}{x}\text{ satisfies Conditions~\ref{condition1} and~\ref{condition2}}, \\[1em]
    x\,\sigma(Kx), & \text{if }\sigma(x)\text{ satisfies Conditions~\ref{condition1} and~\ref{condition2}}.
  \end{cases}
\]
To handle the \(C(\Omega)\)-norm, fix \(M,\varepsilon>0\), and let  \(K\ge \left(\frac{2(CM+G)}{\varepsilon}\right)^2\) where \(C\) and \(G\) are the constants in Condition~\ref{condition1}. Set
\(
\tau:=K^{-1/2}.
\)
In the first case, we split the interval into the regions \(|x|\ge \tau\) and \(|x|\le \tau\). Then
\begin{align}
  \sup_{x\in[-M,M]}\bigl|\psi_K(x)-\mathrm{ReLU}(x)\bigr|
  &\le
  \sup_{M\ge |x|\ge \tau}\bigl|x\bigl(\tfrac{\sigma(Kx)}{Kx}-H(x)\bigr)\bigr|
  +\sup_{|x|\le \tau}\bigl|x\bigl(\tfrac{\sigma(Kx)}{Kx}-H(x)\bigr)\bigr| \notag\\
  &\le \frac{CM}{\tau K}+\tau G =(CM+G)K^{-1/2}.
\end{align}
Since
\[
K \ge \left(\frac{2(CM+G)}{\varepsilon}\right)^2,
\]
we obtain
\[
\sup_{x\in[-M,M]}\bigl|\psi_K(x)-\mathrm{ReLU}(x)\bigr|\le \frac{1}{2}\varepsilon.
\]

In the second case, a completely analogous argument shows that, for \(K\) sufficiently large,
\[
  \sup_{x\in[-M,M]}\bigl|x\,\sigma(Kx)-\mathrm{ReLU}(x)\bigr|
  \le \frac{1}{2}\varepsilon.
\]
Therefore, in either case, for any \(\varepsilon>0\), one can choose \(K \ge \left(\frac{2(CM+G)}{\varepsilon}\right)^2\) sufficiently large such that
\[
  \|\psi_K-\mathrm{ReLU}\|_{C(\Omega)}\le \frac{1}{2}\varepsilon.
\]

Next, we consider the \(W^{m,\infty}\)-part. We first approximate the ReLU activation function in the \(W^{m,\infty}\)-norm away from the small region \([-\delta,\delta]\). More precisely, for any \(M>\delta>0\), we consider the domain
\[
[-M,-\delta]\cup[\delta,M].
\]

In the first case, for any \(0\le s\le m\), choose
\[
K=4C(M+m)\delta^{-m}\varepsilon^{-1},
\]
where \(C\) is the constant in Condition~\ref{condition1}. Then
\begin{align}
&\sup_{M\ge |x|\ge \delta}
\left|
\left(
x\Bigl(\tfrac{\sigma(Kx)}{Kx}-H(x)\Bigr)
\right)^{(s)}
\right|
\notag\\
\le\;&
\sup_{M\ge |x|\ge \delta}
\left|
x\Bigl(\tfrac{\sigma(Kx)}{Kx}-H(x)\Bigr)^{(s)}
\right|
+
s\sup_{M\ge |x|\ge \delta}
\left|
\Bigl(\tfrac{\sigma(Kx)}{Kx}-H(x)\Bigr)^{(s-1)}
\right|
\notag\\
\le\;&
\frac{CM}{K\delta^s}+\frac{Cs}{K\delta^{s-1}}
\le \frac{1}{2}\varepsilon.
\end{align}

In the second case, one checks similarly that, for \(K\) sufficiently large,
\[
\sup_{M\ge |x|\ge \delta}
\bigl|(x\,\sigma(Kx)-\mathrm{ReLU}(x))^{(s)}\bigr|
\le \frac{1}{2}\varepsilon,
\qquad 0\le s\le m.
\]

Therefore, in either case, for any \(M,\delta,\varepsilon>0\), one can choose \(K\) sufficiently large such that
\[
\|\psi_K-\mathrm{ReLU}\|_{W^{m,\infty}([-M,-\delta]\cup[\delta,M])}
\le \frac{1}{2}\varepsilon.
\]

In summary, for any \(\varepsilon>0\), if we choose
\[
K\ge \max\left\{
\left(\frac{2(CM+G)}{\varepsilon}\right)^2,\,
4C(M+m)\delta^{-m}\varepsilon^{-1}
\right\},
\]
then both
\[
\|\psi_K-\mathrm{ReLU}\|_{C(\Omega)}\le \frac{1}{2}\varepsilon
\]
and
\[
\|\psi_K-\mathrm{ReLU}\|_{W^{m,\infty}([-M,-\delta]\cup[\delta,M])}
\le \frac{1}{2}\varepsilon
\]
hold.

By Lemmas~\ref{time} and~\ref{Id} for $\sigma$ satisfying Condition~\ref{condition2}, in both versions of \(\psi_K\), there exists a \(\sigma\)-network \(\phi_K\) of depth \(2\) and width \(12\) such that
\[
\|\psi_K-\phi_K\|_{C^m([-M,M])}\le \frac{1}{2}\varepsilon.
\]
In particular,
\[
\|\psi_K-\phi_K\|_{C(\Omega)}\le \frac{1}{2}\varepsilon.
\]
Consequently,
\begin{equation}
\|\phi_K-\mathrm{ReLU}\|_{C(\Omega)}\le \varepsilon,
\qquad
\|\mathrm{ReLU}-\phi_K\|_{W^{m,\infty}([-M,-\delta]\cup[\delta,M])}
\le \varepsilon,\label{converge K}
\end{equation}
which completes the construction.

Next, we show convergence at the neural-network level. To make the proof more readable, we first introduce the notation. Denote the ReLU activation by \(\sigma_*\), and denote the approximating activation \(\phi_K\) by \(\sigma_K\). Any ReLU network of depth \(L\) and width \(N\), with non-differentiable points \(\{x_i\}_{i=1}^M\subset \Omega=[-M,M]\), can be written as
\[
\phi_{\mathrm{ReLU}}(x)
=
\mathcal{L}_L\circ \sigma_*\circ \mathcal{L}_{L-1}\circ \cdots \circ \sigma_*\circ \mathcal{L}_0(x),
\qquad x\in\mathbb{R},
\]
where \(N_0=1\), \(N_1,\dots,N_L\in\mathbb{N}_+\) with \(\max\{N_1,\dots,N_L\}\le N\), \(N_{L+1}=1\), and each affine map \(\mathcal{L}_\ell(\vy)=W_\ell \vy+b_\ell\) satisfies
\[
W_\ell\in\mathbb{R}^{N_{\ell+1}\times N_\ell},
\qquad
b_\ell\in\mathbb{R}^{N_{\ell+1}}.
\]

We want to show that, as \(K\to\infty\),
\[
\phi_{\mathrm G}(x)
=
\mathcal{L}_L\circ \sigma_K\circ \mathcal{L}_{L-1}\circ \cdots \circ \sigma_K\circ \mathcal{L}_0(x),
\qquad x\in\mathbb{R},
\]
converges to \(\phi_{\mathrm{ReLU}}\) in the \(W^{m,\infty}(\Omega_\delta)\)-norm. The corresponding convergence in the \(C(\Omega)\)-norm has already been proved in \cite{zhang2024deep}, so we omit it here.

For the first layer, it follows directly from \eqref{converge K} that
\[
\lim_{K\to\infty}
\bigl\|
\mathcal{L}_1\circ \sigma_K\circ \mathcal{L}_0
-
\mathcal{L}_1\circ \sigma_*\circ \mathcal{L}_0
\bigr\|_{W^{m,\infty}(\Omega_\delta;\mathbb{R}^{N_2})}
=0.
\]

For the second layer, define
\[
\phi_1:=\mathcal{L}_1\circ \sigma_*\circ \mathcal{L}_0,
\qquad
\phi_{1,K}:=\mathcal{L}_1\circ \sigma_K\circ \mathcal{L}_0.
\]
Then
\[
\sigma_*\circ \phi_1-\sigma_K\circ \phi_{1,K}
=
(\sigma_*-\sigma_K)\circ \phi_{1,K}
+
\bigl(\sigma_*\circ \phi_1-\sigma_*\circ \phi_{1,K}\bigr).
\]

For the second term, Lemma~\ref{lem:Wm_inf_bounds} yields
\[
\bigl\|
\sigma_*\circ \phi_1-\sigma_*\circ \phi_{1,K}
\bigr\|_{W^{m,\infty}(\Omega_\delta;\mathbb{R}^{N_2})}
\le
C\,
\|\phi_1-\phi_{1,K}\|_{W^{m,\infty}(\Omega_\delta;\mathbb{R}^{N_2})},
\]
where \(C\) depends only on
\[
\|\phi_1\|_{W^{m,\infty}(\Omega_\delta;\mathbb{R}^{N_2})},
\qquad
\|\phi_{1,K}\|_{W^{m,\infty}(\Omega_\delta;\mathbb{R}^{N_2})},
\qquad
\|\sigma_*\|_{W^{m+1,\infty}(I_*)},
\]
and \(I_*\subset \mathbb{R}\setminus[-\delta_*,\delta_*]\) is a bounded interval containing the images of both \(\phi_1\) and \(\phi_{1,K}\) for all sufficiently large \(K\).

For the first term, Lemma~\ref{lem:Wm_inf_bounds} also gives
\[
\bigl\|
(\sigma_*-\sigma_K)\circ \phi_{1,K}
\bigr\|_{W^{m,\infty}(\Omega_\delta;\mathbb{R}^{N_2})}
\le
C\bigl(\|\phi_{1,K}\|_{W^{m,\infty}(\Omega_\delta;\mathbb{R}^{N_2})}\bigr)\,
\|\sigma_*-\sigma_K\|_{W^{m,\infty}(I_{\phi_{1,K}})},
\]
where \(I_{\phi_{1,K}}\) denotes the image of \(\phi_{1,K}\).

Since \(\phi_{1,K}\to \phi_1\) uniformly on \(\Omega_\delta\), and each component of \(\phi_1(\Omega_\delta)\) is bounded and stays a positive distance away from the origin, there exist \(\delta_1>0\), \(\delta_*<\delta_1\), and \(K_0>0\) such that for all \(K\ge K_0\), the image \(I_{\phi_{1,K}}\) is bounded and disjoint from \([-\delta_*,\delta_*]\). Consequently,
\[
\|\sigma_*-\sigma_K\|_{W^{m,\infty}(I_{\phi_{1,K}})}\to 0
\qquad\text{as }K\to\infty.
\]
Therefore,
\[
\lim_{K\to\infty}
\bigl\|
\sigma_K\circ \mathcal{L}_1\circ \sigma_K\circ \mathcal{L}_0
-
\sigma_*\circ \mathcal{L}_1\circ \sigma_*\circ \mathcal{L}_0
\bigr\|_{W^{m,\infty}(\Omega_\delta;\mathbb{R}^{N_2})}
=0,
\]
and hence
\[
\lim_{K\to\infty}
\bigl\|
\mathcal{L}_2\circ \sigma_K\circ \mathcal{L}_1\circ \sigma_K\circ \mathcal{L}_0
-
\mathcal{L}_2\circ \sigma_*\circ \mathcal{L}_1\circ \sigma_*\circ \mathcal{L}_0
\bigr\|_{W^{m,\infty}(\Omega_\delta;\mathbb{R}^{N_3})}
=0.
\]

Repeating the same argument layer by layer yields convergence for any finite depth \(L\). This completes the proof.
\end{proof}
 We now show that a closely analogous network can be implemented with deep networks with :

\begin{lemma}\label{point}
  Let \(N,L,s\in\mathbb{N}_+\) and \(\{\xi_i\}_{i=0}^{N^2L^2-1}\subset[0,1]\).  There exists a network \(\phi\) of depth 
  \(
    10(L+2)\,\log_2(4L)
    \quad\text{and width}\quad
    192\,s\,(N+1)\,\log_2(8N)
  \)
  whose activation function meets Conditions \ref{condition1} and \ref{condition2} satisfying
  \begin{enumerate}
    \item \(\bigl|\phi(i)-\xi_i\bigr|\le2\,N^{-2s}L^{-2s}
      \quad\text{for all }i=0,1,\dots,N^2L^2-1,\)
    \item \(-1 \le \phi(x)\le 2\quad\text{for all }x\in\mathbb{R}.\)
  \end{enumerate}
\end{lemma}

\begin{proof}
By Proposition 4.4 of \cite{lu2021deep}, for any \(N,L,s\in\mathbb{N}_+\) and any prescribed values
\(
\{\xi_i\}_{i=0}^{N^2L^2-1}\subset[0,1],
\)
there exists a ReLU FNN \(\phi_{\mathrm{ReLU}}\) with width
\(
16\,s\,(N+1)\log_2(8N)
\)
and depth
\(
5\,(L+2)\log_2(4L)
\)
such that
\begin{enumerate}
    \item
    \[
    |\phi_{\mathrm{ReLU}}(i)-\xi_i|
    \le N^{-2s}L^{-2s},
    \qquad i=0,1,\dots,N^2L^2-1,
    \]
    \item
    \[
    0\le \phi_{\mathrm{ReLU}}(x)\le 1,
    \qquad x\in\mathbb{R}.
    \]
\end{enumerate}

Next, by Proposition~\ref{non-di}, any ReLU network of fixed size can be approximated arbitrarily well in the sup-norm by a network whose activation function satisfies Conditions~\ref{condition1} and~\ref{condition2}. Applying Proposition~\ref{non-di} to \(\phi_{\mathrm{ReLU}}\), we obtain a network \(\phi\) with activation satisfying Conditions~\ref{condition1} and~\ref{condition2}, depth
\(
10(L+2)\log_2(4L),
\)
and width
\(
192\,s\,(N+1)\log_2(8N),
\)
such that
\(
\|\phi-\phi_{\mathrm{ReLU}}\|_{L^\infty(\mathbb{R})}
\le N^{-2s}L^{-2s}.
\)

Therefore, for each \(i=0,1,\dots,N^2L^2-1\),
\[
|\phi(i)-\xi_i|
\le
|\phi(i)-\phi_{\mathrm{ReLU}}(i)|
+
|\phi_{\mathrm{ReLU}}(i)-\xi_i|
\le
2\,N^{-2s}L^{-2s}.
\]
Moreover, since \(0\le \phi_{\mathrm{ReLU}}(x)\le 1\) for all \(x\in\mathbb{R}\), we also have
\[
-1\le \phi(x)\le 2,
\qquad x\in\mathbb{R},
\]
provided the sup-norm approximation error is chosen to be at most \(1\). This proves the lemma.
\end{proof}

In \cite{yang2023nearly}, the second key ReLU‐based construction uses step functions.  Here, we aim to approximate ReLU‐activated networks in the $W^{m,\infty}$ norm by networks whose activation function meets Conditions \ref{condition1} and \ref{condition2}.  

\begin{lemma}\label{lem:step}
  Let \(m,N,L\in\mathbb{N}_+\) and set
  \(
    J \;=\;\Bigl\lfloor N^{1/d}\Bigr\rfloor^2 \,\Bigl\lfloor L^{2/d}\Bigr\rfloor,~\delta\in\Bigl(0,\tfrac1{3J}\Bigr].
  \)
  Define the piecewise‐constant function
  \[
    s(x)
    = k,
    \quad
    x\in \Bigl[\tfrac{k}{J},\,\tfrac{k+1}{J} - \delta\cdot\mathbf{1}_{\{k< J-1\}}\Bigr],
    \quad
    k=0,1,\dots,J-1.
  \]
  Then for any \(\varepsilon>0\), there exists a network \(\phi\) of width \(48N+36\) and depth \(8L+10\) whose activation function meets Conditions \ref{condition1} and \ref{condition2} such that
  \[
    \bigl\|s - \phi\bigr\|_{W^{m,\infty}\!\Bigl(\bigcup_{k=0}^{J-1}
      \bigl[\tfrac{k}{J},\,\tfrac{k+1}{J}-\delta\,\mathbf{1}_{\{k< J-1\}}\bigr]
    \Bigr)}
    \;\le\;\varepsilon.
  \]
\end{lemma}

\begin{proof}
  Proposition 4.3 of \cite{lu2021deep} constructs a ReLU network of the stated width and depth that exactly implements \(s\) on the union of intervals
  \(\bigcup_{k=0}^{J-1}[\tfrac{k}{J},\,\tfrac{k+1}{J}-\delta]\).  Since \(s\) is constant—and hence smooth—on each such interval, there are no non-differentiable points in the interior.  Applying Proposition~\ref{non-di} then yields a network whose activation function meets Conditions \ref{condition1} and \ref{condition2} that approximates the ReLU network in the \(W^{m,\infty}\) norm to within \(\varepsilon\).
\end{proof}

\subsection{Approximation of partition of unity by neural networks whose activation function meets Conditions \ref{condition1} and \ref{condition2}}\label{RE2}
In the next Lemma, we will show that the deep neural networks whose activation function meets Conditions \ref{condition1} and \ref{condition2} can approximate $\text{ReLU}^{p}$ for $p\ge m+1$.

\begin{lemma}\label{p of ReLU}
    For any $\varepsilon>0$, $m\in\sN_+$, $p\ge m+1$ and bounded domain $[-M,M]$, there is a neural network $\phi$ with depth $p$ and width $20$ whose activation function meets Conditions \ref{condition1} and \ref{condition2}, such that \begin{equation}
        \|\mathrm{ReLU}^{p}-\phi\|_{W^{m,\infty}([-M,M])}\le \varepsilon.\notag
    \end{equation}
\end{lemma}

    \begin{proof}
Set
\[
  \psi_{K,p}(x)
  =
  \begin{cases}
    \displaystyle x^{p-1}\sigma(Kx), & \text{if }\tfrac{\sigma(x)}{x}\text{ satisfies Conditions~\ref{condition1} and~\ref{condition2}}, \\[1em]
    x^{p}\sigma(Kx), & \text{if }\sigma(x)\text{ satisfies Conditions~\ref{condition1} and~\ref{condition2}}.
  \end{cases}
\]
Let \(\phi_{K,p}\) be the depth-\(p\), width-\(20\) network from
Lemmas~\ref{sqr11} and~\ref{xm+1} that approximates \(\psi_{K,p}\).
Then
\[
\|\mathrm{ReLU}^{p}-\phi_{K,p}\|_{W^{m,\infty}([-M,M])}
\le
\|\mathrm{ReLU}^{p}-\psi_{K,p}\|_{W^{m,\infty}([-M,M])}
+
\|\psi_{K,p}-\phi_{K,p}\|_{W^{m,\infty}([-M,M])}.
\]
By Lemmas~\ref{lem:Wm_inf_bounds}, \ref{sqr11}, and~\ref{xm+1}, for every fixed
\(K\), the second term can be made arbitrarily small. Thus it remains to prove
\begin{equation}\label{eq:relu-p-limit}
\lim_{K\to\infty}
\|\psi_{K,p}-\mathrm{ReLU}^{p}\|_{W^{m,\infty}([-M,M])}=0.
\end{equation}

Without loss of generality, we only prove the second case, namely
\[
\psi_{K,p}(x)=x^p\sigma(Kx),
\]
where \(\sigma\) satisfies Conditions~\ref{condition1} and~\ref{condition2}.
The first case is analogous.

We prove \eqref{eq:relu-p-limit} by induction on \(m\).

For \(m=0\), we only need to prove convergence in \(L^\infty([-M,M])\).
Since \(\mathrm{ReLU}^{p}(x)=0\) on \([-M,0]\), we have
\begin{align}
\|\psi_{K,p}-\mathrm{ReLU}^{p}\|_{L^\infty([-M,0])}
&=
\|\psi_{K,p}\|_{L^\infty([-M,0])} \notag\\
&\le
\|x^p\sigma(Kx)\|_{L^\infty([-M,-\delta])}
+
\|x^p\sigma(Kx)\|_{L^\infty([-\delta,0])} \notag\\
&\le
\frac{CM^p}{K\delta}+\delta^p,
\label{eq:base-neg}
\end{align}
where the last inequality follows from Conditions~\ref{condition1}
and~\ref{condition2}. Choosing
\[
\delta=K^{-1/(p+1)},
\]
we obtain
\[
\|\psi_{K,p}-\mathrm{ReLU}^{p}\|_{L^\infty([-M,0])}
=
\mathcal{O}\!\left(K^{-p/(p+1)}\right).
\]
Similarly,
\[
\|\psi_{K,p}-\mathrm{ReLU}^{p}\|_{L^\infty([0,M])}
=
\mathcal{O}\!\left(K^{-p/(p+1)}\right).
\]
Hence
\[
\|\psi_{K,p}-\mathrm{ReLU}^{p}\|_{L^\infty([-M,M])}\to 0
\qquad\text{as }K\to\infty.
\]
This proves \eqref{eq:relu-p-limit} for \(m=0\).

Assume now that \eqref{eq:relu-p-limit} holds for \(m-1\). We prove it for \(m\).
By the definition of the \(W^{m,\infty}\)-norm,
\[
\|\psi_{K,p}-\mathrm{ReLU}^{p}\|_{W^{m,\infty}([-M,M])}
\le
\|\psi_{K,p}-\mathrm{ReLU}^{p}\|_{L^\infty([-M,M])}
+
\left\|\frac{\mathrm d}{\mathrm dx}\bigl[\psi_{K,p}-\mathrm{ReLU}^{p}\bigr]\right\|_{W^{m-1,\infty}([-M,M])}.
\]
The \(L^\infty\)-term tends to zero by the argument above, which is independent
of \(m\). For the derivative term, we compute
\begin{align}
\frac{\mathrm d}{\mathrm dx}\bigl[x^p\sigma(Kx)\bigr]
&=
p\,x^{p-1}\sigma(Kx)+x^p(\sigma(Kx))'
=
p\,\psi_{K,p-1}+x^p(\sigma(Kx))',
\notag\\
\frac{\mathrm d}{\mathrm dx}\mathrm{ReLU}^{p}
&=
p\,\mathrm{ReLU}^{p-1}.
\notag
\end{align}
Therefore,
\begin{align*}
&\left\|
\frac{\mathrm d}{\mathrm dx}\bigl[\psi_{K,p}-\mathrm{ReLU}^{p}\bigr]
\right\|_{W^{m-1,\infty}([-M,M])}\\
\le{}&
p\,\|\psi_{K,p-1}-\mathrm{ReLU}^{p-1}\|_{W^{m-1,\infty}([-M,M])}
+
\|x^p(\sigma(Kx))'\|_{W^{m-1,\infty}([-M,M])}.
\end{align*}

Since \(p-1\ge m\), the induction hypothesis applied to the exponent \(p-1\)
and norm order \(m-1\) gives
\[
\|\psi_{K,p-1}-\mathrm{ReLU}^{p-1}\|_{W^{m-1,\infty}([-M,M])}\to 0
\qquad\text{as }K\to\infty.
\]
It remains to show that
\[
\|x^p(\sigma(Kx))'\|_{W^{m-1,\infty}([-M,M])}\to 0.
\]

By Leibniz' rule,
\[
\|x^p(\sigma(Kx))'\|_{W^{m-1,\infty}([-M,M])}
\le
\sum_{s=0}^{m-1}\sum_{q=0}^{s}
\left\|
\binom{s}{q}D^q(x^p)\,
D^{s-q+1}(\sigma(Kx))
\right\|_{L^\infty([-M,M])}.
\]
For \(0\le q\le s\le m-1\), the chain rule and Condition~\ref{condition1} imply
\begin{align}
\left|D^q(x^p)\,D^{s-q+1}(\sigma(Kx))\right|
&\le
C\,K^{s-q+1}|x|^{p-q}\frac{1}{|Kx|^{\,s-q+2}}
\notag\\
&=
C\,K^{-1}|x|^{p-s-2},
\label{eq:term-bound}
\end{align}where the constant $C$ is defined in Condition \ref{condition1}.
Since \(p\ge m+1\) and \(s\le m-1\), we have
\[
p-s-2\ge p-(m-1)-2=p-m-1\ge 0.
\]
Hence
\[
|x|^{p-s-2}\le M^{p-s-2}\le M^{p-2},
\]
and therefore \eqref{eq:term-bound} yields
\[
\left|D^q(x^p)\,D^{s-q+1}(\sigma(Kx))\right|
\le
C\,K^{-1}M^{p-2}.
\]
Summing over all \(0\le q\le s\le m-1\), we conclude that
\[
\|x^p(\sigma(Kx))'\|_{W^{m-1,\infty}([-M,M])}
\le
C\,K^{-1}M^{p-2}\to 0
\qquad\text{as }K\to\infty.
\]

Thus
\[
\left\|
\frac{\mathrm d}{\mathrm dx}\bigl[\psi_{K,p}-\mathrm{ReLU}^{p}\bigr]
\right\|_{W^{m-1,\infty}([-M,M])}\to 0,
\]
and consequently
\[
\|\psi_{K,p}-\mathrm{ReLU}^{p}\|_{W^{m,\infty}([-M,M])}\to 0.
\]
This completes the induction and proves \eqref{eq:relu-p-limit}.

Finally, choose \(K\) sufficiently large so that
\[
\|\psi_{K,p}-\mathrm{ReLU}^{p}\|_{W^{m,\infty}([-M,M])}\le \frac{\varepsilon}{2},
\]
and then choose \(\phi_{K,p}\) so that
\[
\|\phi_{K,p}-\psi_{K,p}\|_{W^{m,\infty}([-M,M])}\le \frac{\varepsilon}{2}.
\]
By the triangle inequality,
\[
\|\mathrm{ReLU}^{p}-\phi_{K,p}\|_{W^{m,\infty}([-M,M])}\le \varepsilon.
\]
This completes the proof.
\end{proof}

Now we use the Lemma \ref{p of ReLU} to show that for any deep $\operatorname{ReLU}^{m+1}$ neural network, there is a deep neural network whose activation function meets Conditions \ref{condition1} and \ref{condition2} to approximate it well in $W^{m,\infty}$-norm.

\begin{proposition}\label{sqr approixmation}
    For any deep $\operatorname{ReLU}^{m+1}$ neural network $\phi_{\mathrm{ReLU}^{m+1}}$ with depth $L$ and width $N$, for any $\varepsilon>0$ and a bounded domain $\Omega$, there is a deep neural networks $\phi_\text{G}$ with depth $(m+1)L$ and width $20N$ whose activation function meets Conditions \ref{condition1} and \ref{condition2} such that \begin{align}
        \left\|\phi_\text{G}-\phi_{\mathrm{ReLU}^{m+1}}\right\|_{W^{m,\infty}(\Omega)}\le \varepsilon.\notag
    \end{align}
\end{proposition}

\begin{proof}
  The proof is similar with Proposition \ref{non-di}.
\end{proof}

Based on such results, we can have the following proposition which shows that neural networks whose activation function meets Conditions \ref{condition1} and \ref{condition2} can approximate any polynomial in $W^{m,\infty}(\Omega)$ combining the result in \cite{he2023expressivity}.
\begin{lemma}\label{poly}
  Let \(P(\vx)\) be a polynomial of total degree \(s\) in \(\mathbb{R}^d\) on a bounded domain \(\Omega\).  For any \(\varepsilon>0\) and $m\in\sN_+$, there exists a deep neural network \(\phi\) of depth smaller than \((m+1)\log_{m+1} s\) and width \(40(s+1)^d\) whose activation function meets Conditions \ref{condition1} and \ref{condition2} such that
  \[
    \|P - \phi\|_{W^{m,\infty}(\Omega)} \;\le\; \varepsilon.
  \]
\end{lemma}

\begin{proof}
  By Theorem 1 in \cite{he2023expressivity}, one can construct a $\text{ReLU}^{m+1}$ network \(\phi_*\) of depth \(\log_{m+1} s\) and width \(2\,(s+1)^d\) that exactly represents the polynomial \(P(\vx)\).  Applying Proposition \ref{sqr approixmation}, which approximates the $\text{ReLU}^{m+1}$ activation by networks whose activation function meets Conditions \ref{condition1} and \ref{condition2} in the \(W^{m,\infty}\) norm, we obtain a network whose activation function meets Conditions \ref{condition1} and \ref{condition2} of depth \((m+1)\log_{m+1} s\) and width \(40\,(s+1)^d\) that approximates \(P\) up to error \(\varepsilon\).  This completes the proof.
\end{proof}

Finally, the remaining key ReLU network in the approximation scheme of \cite{yang2023nearly} is the partition of unity.  Unlike our step‐function construction, which only approximates locally and thus cannot be extended across the entire domain, the partition of unity must hold globally.  Therefore, we must adopt a different strategy: we will build a smoother partition of unity—refining the one used in \cite{yang2023nearly}—and then analyze its approximation properties.  In what follows, we first construct this smooth partition of unity and subsequently study how it can be approximated by networks whose activation function meets Conditions \ref{condition1} and \ref{condition2}.

\begin{lemma}\label{S}
For any integer \(m\ge1\), there exists a polynomial \(S_m\) of degree \(2m+1\) on \([0,1]\) such that
\[
S_m(0)=0,\quad S_m(1)=1,
\qquad
S_m^{(p)}(0)=S_m^{(p)}(1)=0
\quad (p=1,2,\dots,m),
\]
and
\[
S_m(x)+S_m(1-x)=1
\quad\text{for all }x\in[0,1].
\]
Moreover,
\[
0\le S_m(x)\le 1
\quad\text{for all }x\in[0,1].
\]
\end{lemma}

\begin{proof}
As before, construct an odd polynomial
\[
P_m(x)=\sum_{k=0}^m p_k\,x^{2k+1}
\quad\text{on }[-1,1]
\]
satisfying
\begin{enumerate}[label=(\roman*)]
\item \(P_m^{(p)}(\pm1)=0\) for \(p=1,\dots,m\),
\item \(P_m(-1)=-1,\;P_m(1)=1\),
\item \(P_m(-x)=-P_m(x)\).
\end{enumerate}
The conditions \(P_m(1)=1\) and \(P_m^{(p)}(1)=0\) for \(p=1,\dots,m\) yield an \((m+1)\times(m+1)\) linear system for \(\{p_k\}\), which has a unique solution. Define
\[
S_m(x)=\tfrac12\bigl(P_m(2x-1)+1\bigr),
\qquad x\in[0,1].
\]
It is immediate that \(\deg S_m=2m+1\), \(S_m(0)=0\), \(S_m(1)=1\), and
\[
S_m^{(p)}(0)=S_m^{(p)}(1)=0,
\qquad p=1,\dots,m.
\]
Moreover, since \(P_m\) is odd,
\[
S_m(x)+S_m(1-x)
=\tfrac12\bigl(P_m(2x-1)+1+P_m(1-2x)+1\bigr)
=1.
\]

It remains to show that \(0\le S_m(x)\le 1\) on \([0,1]\). Since \(P_m^{(p)}(\pm1)=0\) for \(p=1,\dots,m\), the derivative \(P_m'\) vanishes to order \(m\) at both \(x=-1\) and \(x=1\). As \(P_m'\) is a polynomial of degree \(2m\), these zeros account for all zeros of \(P_m'\). Therefore, \(P_m'\) does not change sign on \((-1,1)\), and hence \(P_m\) is monotone on \([-1,1]\). Since \(P_m(-1)=-1\) and \(P_m(1)=1\), we obtain
\[
P_m([-1,1])=[-1,1].
\]
Consequently,
\[
S_m([0,1])
=\tfrac12\bigl(P_m([-1,1])+1\bigr)
=[0,1].
\]
In particular, \(0\le S_m(x)\le 1\) for all \(x\in[0,1]\).
\end{proof}

        Then we define a partition of unity \(\{s_{\vm}\}_{\vm\in\{1,2\}^d}\) on \([0,1]^d\).

\begin{definition}\label{def:sm}
Let \(m\in\mathbb{N}_+\), and fix \(J\in\mathbb{N}_+\).

\medskip
\noindent
\textbf{(1).}
Define \(s:\mathbb{R}\to[0,1]\) by
\[
s(x):=
\begin{cases}
S_m(x), & x\in[0,1],\\[2mm]
1, & x\in[1,2],\\[2mm]
S_m(3-x), & x\in[2,3],\\[2mm]
0, & \text{otherwise},
\end{cases}
\]
where \(S_m\) is the polynomial given in Lemma~\ref{S}. Then \(s\in W^{m+1,\infty}(\mathbb{R})\); see Figure~\ref{ss} for $m=1$. \begin{figure}
    \centering
    \includegraphics[width=0.7\linewidth]{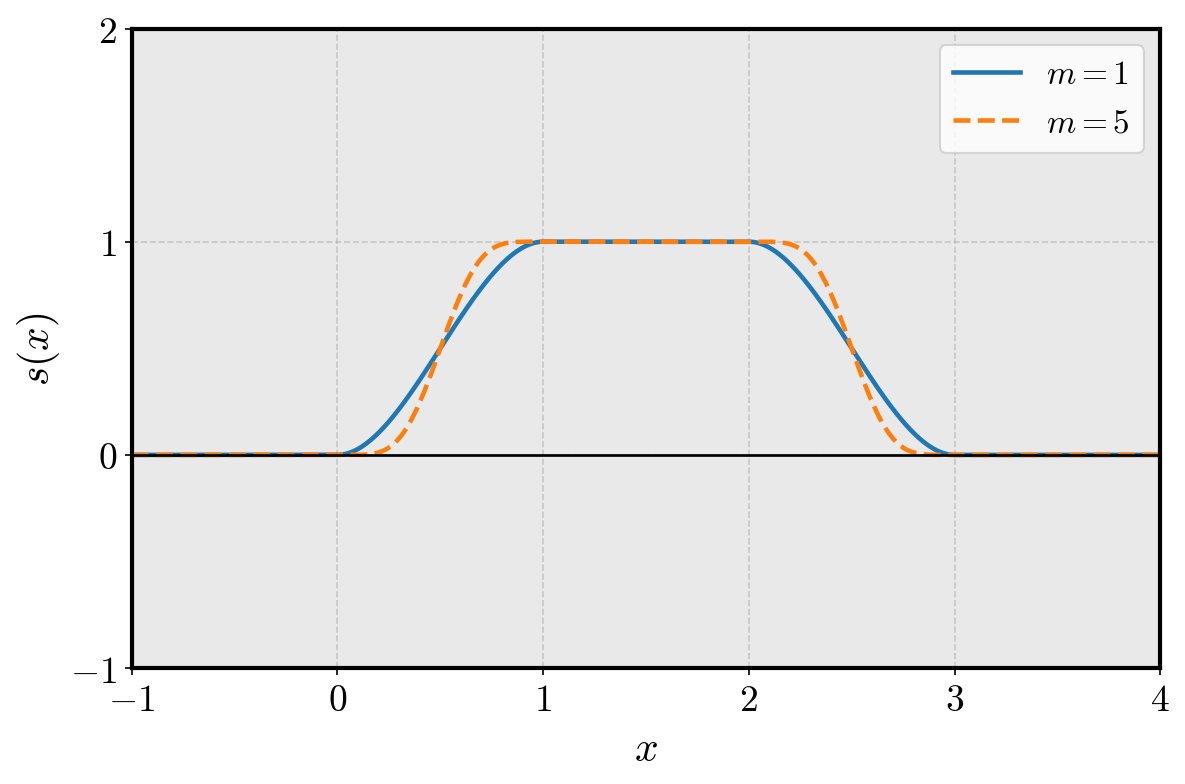}
    \caption{Plot of $s(x)$ for $m=1$ and $m=5$.}
    \label{ss}
\end{figure}

\medskip
\noindent
\textbf{(2).}
Define
\[
s_1(x):=\sum_{i=0}^{J}s(4Jx-1-4i),
\qquad
s_2(x):=s_1\!\left(x+\frac{1}{2J}\right).
\]

\medskip
\noindent
\textbf{(3).}
For any \(\vm=(m_1,\dots,m_d)\in\{1,2\}^d\), define
\[
s_{\vm}(\vx):=\prod_{j=1}^d s_{m_j}(x_j),
\qquad
\vx=(x_1,\dots,x_d)\in\mathbb{R}^d.
\]
\end{definition}

\begin{lemma}\label{lem:sm-support}
Let \(\{s_{\vm}\}_{\vm\in\{1,2\}^d}\) be given by Definition~\ref{def:sm}. Then \(\{s_{\vm}\}_{\vm\in\{1,2\}^d}\) forms a partition of unity on \([0,1]^d\), namely
\[
\sum_{\vm\in\{1,2\}^d}s_{\vm}(\vx)=1,
\qquad \vx\in[0,1]^d.
\]
Moreover, for each \(\vm\in\{1,2\}^d\),
\[
\operatorname{supp}(s_{\vm})\cap[0,1]^d\subset \Omega_{\vm},
\qquad
\Omega_{\vm}:=\prod_{j=1}^d \Omega_{m_j},
\]
where
\[
\Omega_1:=\bigcup_{i=0}^{J-1}\left[\frac{i}{J}+\frac{1}{4J},\,\frac{i+1}{J}\right],
\qquad
\Omega_2:=\left(\bigcup_{i=0}^{J}\left[\frac{i}{J}-\frac{1}{4J},\,\frac{i}{J}+\frac{1}{2J}\right]\right)\cap[0,1].
\]
\end{lemma}

\begin{proof}
We first show that
\[
s_1(x)+s_2(x)=1,\qquad x\in[0,1].
\]
Indeed, by definition,
\[
s_1(x)=\sum_{i=0}^{J}s(4Jx-1-4i),
\qquad
s_2(x)=\sum_{i=0}^{J}s(4Jx+1-4i).
\]
Fix \(x\in[0,1]\), and set
\[
y:=4Jx-1.
\]
Then \(y\in[-1,4J-1]\). Since \(\operatorname{supp}(s)\subset[0,3]\), only those terms for which
\[
y-4i\in[0,3]
\qquad\text{or}\qquad
y+2-4i\in[0,3]
\]
can contribute. One checks directly that for each \(x\in[0,1]\), exactly two such terms appear, and they are complementary. More precisely, there exists \(i\in\{0,\dots,J\}\) such that either
\[
y-4i\in[0,1]
\quad\text{and}\quad
y+2-4i\in[2,3],
\]
or
\[
y-4i\in[2,3]
\quad\text{and}\quad
y+2-4i\in[0,1],
\]
or one of them lies in \([1,2]\), where \(s\equiv 1\). Using the definition
\[
s(t)=S_m(t)\ \text{on }[0,1],\qquad
s(t)=1\ \text{on }[1,2],\qquad
s(t)=S_m(3-t)\ \text{on }[2,3],
\]
we see that the two nonzero contributions always add up to \(1\). Hence
\[
s_1(x)+s_2(x)=1,\qquad x\in[0,1].
\]

Now let \(\vx=(x_1,\dots,x_d)\in[0,1]^d\). Then
\[
\sum_{\vm\in\{1,2\}^d}s_{\vm}(\vx)
=
\sum_{\vm\in\{1,2\}^d}\prod_{j=1}^d s_{m_j}(x_j)
=
\prod_{j=1}^d \bigl(s_1(x_j)+s_2(x_j)\bigr)
=
1.
\]
Therefore, \(\{s_{\vm}\}_{\vm\in\{1,2\}^d}\) is a partition of unity on \([0,1]^d\).

Next we verify the support inclusion. Since \(\operatorname{supp}(s)\subset[0,3]\), we have
\[
s(4Jx-1-4i)\neq 0
\quad\Longrightarrow\quad
0\le 4Jx-1-4i\le 3,
\]
which is equivalent to
\[
\frac{i}{J}+\frac{1}{4J}\le x\le \frac{i+1}{J}.
\]
Hence
\[
\operatorname{supp}(s_1)\cap[0,1]
\subset
\bigcup_{i=0}^{J-1}\left[\frac{i}{J}+\frac{1}{4J},\,\frac{i+1}{J}\right]
=
\Omega_1.
\]
Similarly, since \(s_2(x)=s_1(x+\tfrac{1}{2J})\), we have
\[
s_2(x)\neq 0
\quad\Longrightarrow\quad
x+\frac{1}{2J}\in \Omega_1,
\]
which yields
\[
\operatorname{supp}(s_2)\cap[0,1]
\subset
\left(\bigcup_{i=0}^{J}\left[\frac{i}{J}-\frac{1}{4J},\,\frac{i}{J}+\frac{1}{2J}\right]\right)\cap[0,1]
=
\Omega_2.
\]

Finally, for \(\vm=(m_1,\dots,m_d)\in\{1,2\}^d\),
\[
\operatorname{supp}(s_{\vm})\cap[0,1]^d
\subset
\prod_{j=1}^d\bigl(\operatorname{supp}(s_{m_j})\cap[0,1]\bigr)
\subset
\prod_{j=1}^d \Omega_{m_j}
=
\Omega_{\vm}.
\]
This completes the proof.
\end{proof}
Here we would like to remark that the constant \(J\) will be chosen later, depending on the approximation error. The total number of partition-of-unity elements is \(2^d\), which is independent of \(J\). For the convenience of the reader, Fig.~\ref{fig:omega-sij} illustrates the domains \(\Omega_{(i,j)}\) and the corresponding partition-of-unity functions \(s_{(i,j)}\) in the case \(J=3\) and \(d=2\).

\begin{figure}[ht]
    \centering
    \begin{subfigure}[t]{0.41\textwidth}
        \centering
        \includegraphics[width=\textwidth]{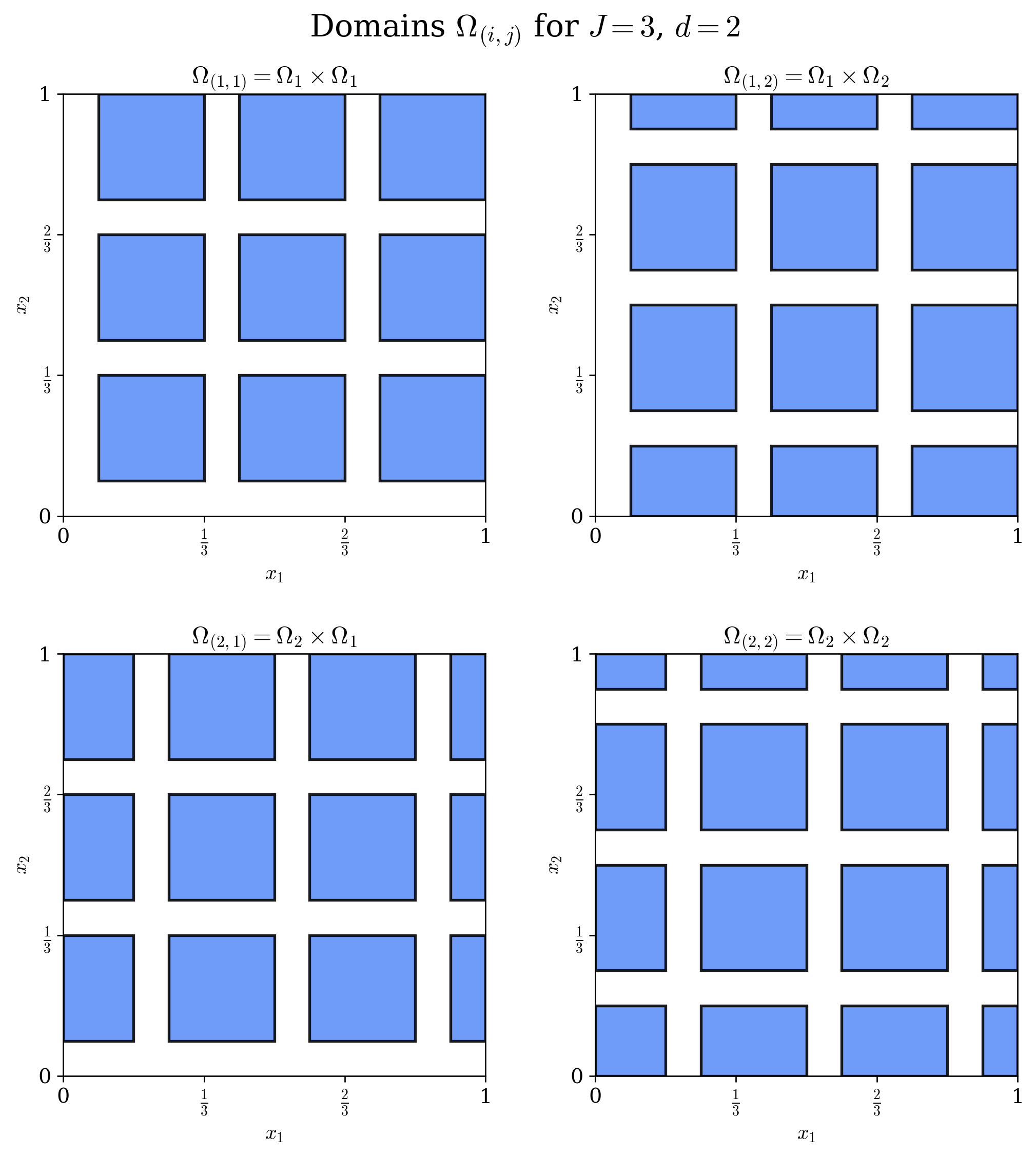}
        \caption{The domains \(\Omega_{(1,1)}\), \(\Omega_{(1,2)}\), \(\Omega_{(2,1)}\), \(\Omega_{(2,2)}\) for \(J=3\) and \(d=2\).}
        \label{fig:omega-sets}
    \end{subfigure}
    \hfill
    \begin{subfigure}[t]{0.48\textwidth}
        \centering
        \includegraphics[width=\textwidth]{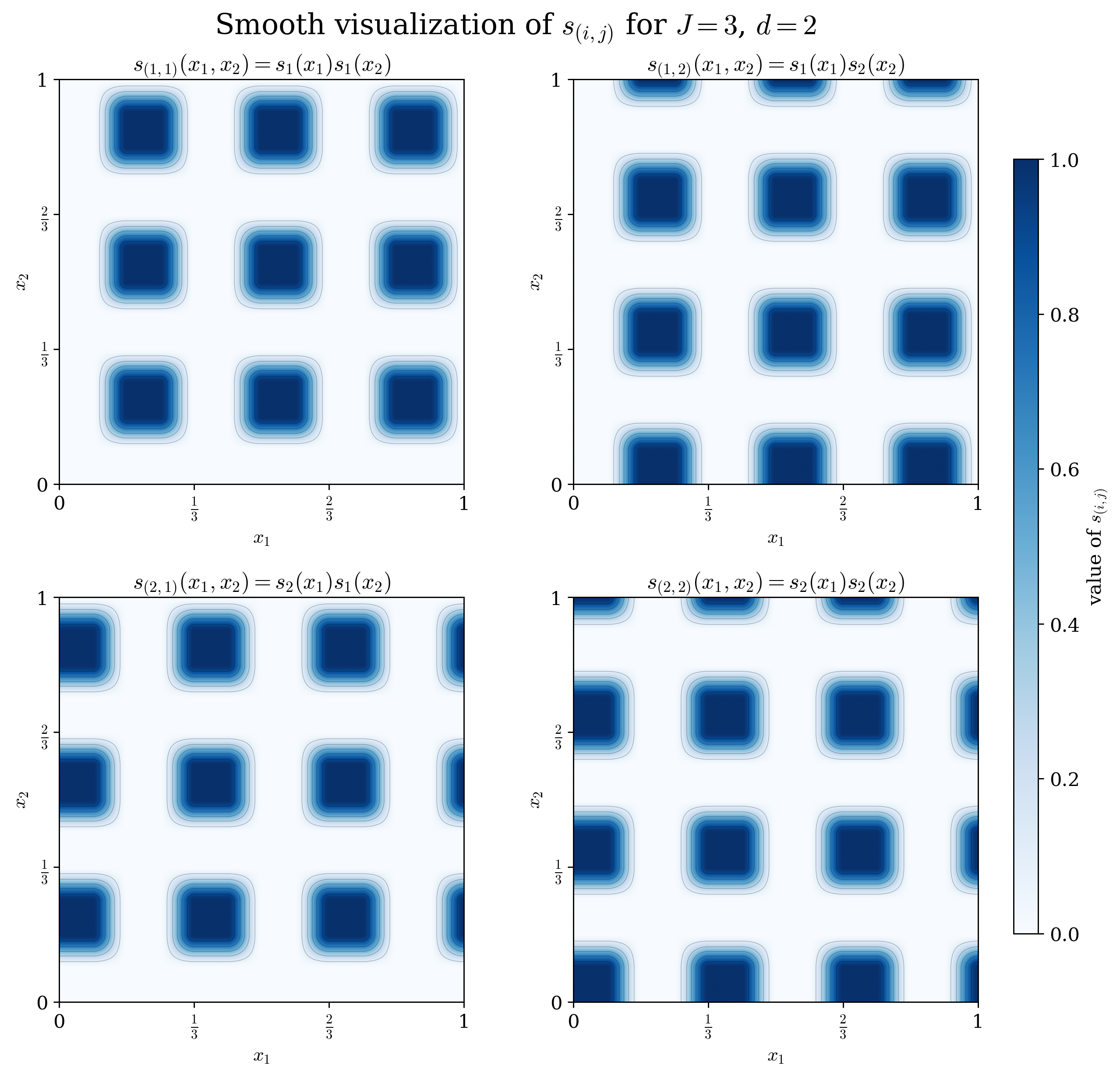}
        \caption{The corresponding partition of unity functions \(s_{(1,1)},s_{(1,2)},s_{(2,1)},s_{(2,2)}\).}
        \label{fig:sij-sets}
    \end{subfigure}
    \caption{Illustration of the domains \(\Omega_{\vm}\) and the associated partition of unity functions \(s_{\vm}\) in the case \(J=3\) and \(d=2\).}
    \label{fig:omega-sij}
\end{figure}

\begin{proposition}\label{app pou}
    For any $\vm\in\{1,2\}^d$, $N,L>0$ with $\log_2N\le L$, set $J=\lfloor N^{1/d}\rfloor^2\lfloor L^{2/d}\rfloor$, then for any $\varepsilon>0$, there is a neural network with depth $3L+2m+d-1$ and width $d(2N+120m+200)$ whose activation function meets Conditions \ref{condition1} and \ref{condition2} such that \begin{equation}
    \|\phi-s_{\vm}\|_{W^{m,\infty}(\Omega)}\le \varepsilon.\notag
\end{equation}
\end{proposition}

\begin{proof}
   Denote the activation function as $\sigma$. Without loss of generalization, we consider $\vm=\vm_*:=(1,1,\ldots,1)$. We begin by approximating \(s_1\) on the interval \([0,1]\).  Inspired by \cite{yang2023nearly}, one can write
\[
  s_1\left(x+\frac{1}{8J}\right)\;=\;\bigl(s(4J\,x - 1/2)\bigr)\circ \psi(x),
  \qquad x\in[0,1],
\]
where \(\psi\colon[0,1]\to\mathbb{R}\) is the \(2/J\)–periodic “hat” function of height \(1/J\), restricted to \([0,1]\):
\[
  \psi(x) =
  \begin{cases}
    x - \frac{2k}{J}, 
      & x\in\bigl[\tfrac{2k}{J},\,\tfrac{2k+1}{J}\bigr],\\[6pt]
    \frac{2k+2}{J} - x,
      & x\in\bigl[\tfrac{2k+1}{J},\,\tfrac{2k+2}{J}\bigr],
  \end{cases}
  \quad
  k=0,1,\dots,\Bigl\lfloor\frac{J}{2}\Bigr\rfloor.
\]
Each sub-interval \(\bigl[\tfrac{2k}{J},\tfrac{2k+2}{J}\bigr]\subset[0,1]\) carries one “hat” of height \(1/J\), rising linearly from \(0\) to \(1/J\) then symmetrically falling back to \(0\).  Composing with the affine‐scaled step \(s(4J\,x+1)\) yields the partition element \(s_1\) on \([0,1]\). For $s$, we know that it is a $C^m$-function on $[-1,4]$ with $2m+1$ degree piece-wise polynomial with $6$ interpolation points $\{-1,0,1,2,3,4\}$, based on \cite{de1978practical}, which can be represented as following:\begin{equation}
s(x)=\sum_{i=0}^3\sum_{j=m+1}^{2m+1}c_{ij}\text{ReLU}^j(x-i)+p_{2m+1}(x),\notag
\end{equation}where $p(x)$ is a $2m+1$ degree polynomial. Then based on Lemma \ref{p of ReLU}, we know for any $\varepsilon_1>0$, there is a neural network $\phi_{1,1}$ with depth $2m+1$ and width smaller than $4\cdot 20\cdot (m+2)$ whose activation function meets Conditions \ref{condition1} and \ref{condition2} such that\begin{equation}
\left\|\sum_{i=0}^3\sum_{j=m+1}^{2m+1}c_{ij}\text{ReLU}^j(x-i)-\phi_{1,1}\right\|_{W^{m,\infty}([-1,4])}\le \varepsilon_1/2.\notag
\end{equation} Based on Lemma \ref{poly}, we know for any $\varepsilon_1>0$, there is a $\sigma$ neural network $\phi_{1,2}$ with depth $2m\ge (m+1)\log_{m+1}(2m+1)$ and width smaller than $20(2m+2)$ such that\begin{equation}
\|p_{2m+1}(x)-\phi_{1,2}\|_{W^{m,\infty}([-1,4])}\le \varepsilon_1/2.\notag
\end{equation} 
Then define $\phi_1(x)=\phi_{1,1}(4J\,x - 1/2)\bigr)+\phi_{1,2}(4J\,x - 1/2)\bigr)$ with depth $2m+1$ and width $a_m$ such that \begin{equation}
    \|s(4J\,x - 1/2)\bigr)-\phi_1\|_{W^{m,\infty}([0,1/J])}\le \|s(4J\,x - 1/2)\bigr)-\phi_1\|_{W^{m,\infty}([-1/8J,1/8J+1/J])} \le \varepsilon_1,\label{small epsilon1}
\end{equation}where \[a_m=20(2m+2)+80(m+2)=120m+200.\] To construct $\psi$, proceed as follows: first build a shallow ReLU network that realizes on the interval $\bigl[0,\frac{1}{\lfloor N^{1/d}\rfloor\lfloor L^{2/d}\rfloor}\bigr]$ a piecewise-linear “hat” function consisting of $\lfloor N^{1/d}\rfloor/2$ identical hats; this requires a single hidden layer of width $3\lfloor N^{1/d}\rfloor/2\le 2N$.  Next, fold that template $\lceil \log_{2}(\lfloor N^{1/d}\rfloor\lfloor L^{2/d}\rfloor)\rceil\le\frac{2}{d}\log_{2}L+\frac{1}{d}\log_{2}N\le 3L$ times, so that $\psi$ is obtained as the composition of at most $3L$ such hat functions.  In summary, $\psi$ can be implemented by a ReLU network of width $2N$ and depth $3L$, with $J=\lfloor N^{1/d}\rfloor^{2}\lfloor L^{2/d}\rfloor$ and any $L\ge\log_{2}N$. Now we consider the approximation, we divide the domain into two part (Figure~\ref{BB}):\begin{align}
    B_1:=\bigcup_{k=0,1,\dots,\Bigl\lfloor\frac{J}{2}\Bigr\rfloor}\left(\bigl[\tfrac{2k}{J},\,\tfrac{32k+1}{16J}\bigr]\cup \bigl[\tfrac{32k+15}{16J},\,\tfrac{32k+17}{16J}\bigr]\cup\bigl[\tfrac{32k+31}{16J},\,\tfrac{2k+2}{J}\bigr]\right)\bigcap [0,1],~B_2:=[0,1]\backslash B_1.\notag
\end{align}\begin{figure}
    \centering
    \includegraphics[width=0.77\linewidth]{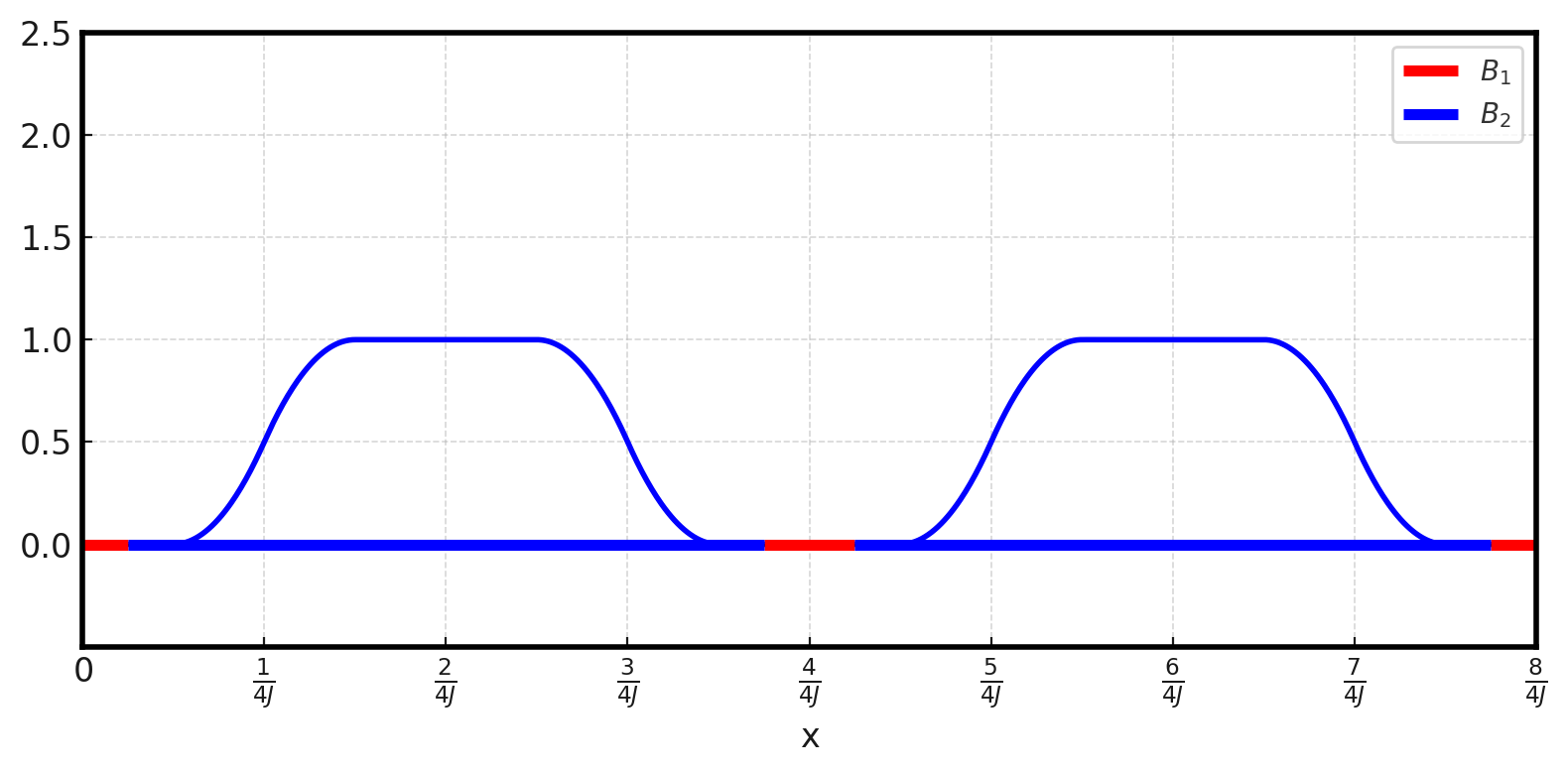}
    \caption{$B_1$ and $B_2$.}
    \label{BB}
\end{figure}It is easy to check that all the non-smooth point of $\psi$ is in $B_1$, therefore, for any $\varepsilon_2>0$, we can find a $\sigma$ neural network $\phi_2$ to make that \begin{equation}
    \|\psi-\phi_2\|_{W^{m,\infty}(B_2)}\le \varepsilon_2, \|\psi-\phi_2\|_{C(\Omega)}\le \varepsilon_2, \notag
\end{equation}based on Proposition \ref{non-di}. We choose \(\varepsilon_2<\frac{1}{32J}\), so that the values of \(\phi_2\) on \(B_1\) remain close to those of \(\psi\) on \(B_1\), within an error of order \(\frac{1}{32J}\). More precisely,
\[
\phi_2(B_1)\subset\left[-\frac{1}{32J},\frac{3}{32J}\right]\cup\left[\frac{29}{32J},\frac{33}{32J}\right].
\]
On this set, the function \(s(4Jx-\tfrac12)\) vanishes. Therefore, by \eqref{small epsilon1}, we obtain
\[
\|\phi_1\|_{W^{m,\infty}(\phi_2(B_1))}\le \varepsilon_1.
\]
Therefore, let us consider the error \begin{align}
    &\left\|\;\bigl(s(4J\,x - 1/2)\bigr)\circ \psi(x)-\phi_1\circ\phi_2\right\|_{W^{m,\infty}([0,1])}\notag\\\le& \left\|\;\bigl(s(4J\,x - 1/2)\bigr)\circ \psi(x)-\phi_1\circ\phi_2\right\|_{W^{m,\infty}(B_1)}+\left\|\;\bigl(s(4J\,x - 1/2)\bigr)\circ \psi(x)-\phi_1\circ\phi_2\right\|_{W^{m,\infty}(B_2)}.\notag
\end{align} For the first part, we have that \begin{align}
    \left\|\;\bigl(s(4J\,x - 1/2)\bigr)\circ \psi(x)-\phi_1\circ\phi_2\right\|_{W^{m,\infty}(B_1)}=\left\|\phi_1\circ\phi_2\right\|_{W^{m,\infty}(B_1)}\le C_*\varepsilon_1\|\phi_2\|_{W^{m,\infty}(B_1)}.\notag
\end{align}For the second part, we have that \begin{align}
    &\left\|\;\bigl(s(4J\,x - 1/2)\bigr)\circ \psi-\phi_1\circ\phi_2\right\|_{W^{m,\infty}(B_2)}\notag\\\le &\left\|\;\bigl(s(4J\,x - 1/2)\bigr)\circ \psi-\;\bigl(s(4J\,x - 1/2)\bigr)\circ\phi_2\right\|_{W^{m,\infty}(B_2)}+\left\|\;\bigl(s(4J\,x - 1/2)\bigr)\circ\phi_2-\phi_1\circ\phi_2\right\|_{W^{m,\infty}(B_2)}.\notag
\end{align}For the second term, it can bounded by \begin{align}
   \left\|\;\bigl(s(4J\,x - 1/2)\bigr)\circ \phi_2-\phi_1\circ\phi_2\right\|_{W^{m,\infty}(B_2)}\le &C\left\|\;\bigl(s(4J\,x - 1/2)\bigr)-\phi_1\right\|_{W^{m,\infty}([0,1/J])} \|\phi_2\|_{W^{m,\infty}(B_2)}\notag\\\le&C_*\varepsilon_1\|\phi_2\|_{W^{m,\infty}(B_2)}.\notag
\end{align} For the first term, we first choose $\varepsilon_2$ smaller than 1 such that\begin{align}
    &\left\|\;\bigl(s(4J\,x - 1/2)\bigr)\circ \psi-\;\bigl(s(4J\,x - 1/2)\bigr)\circ\phi_2\right\|_{W^{m,\infty}(B_2)}\notag\\\le&C(\|s(4J\,x - 1/2)\|_{W^{m+1,\infty}([0,1/J])},\|\psi\|_{W^{m,\infty}(B_2)},(1+\|\psi\|_{W^{m,\infty}(B_2)}))\varepsilon_2. \notag
\end{align}Therefore, for any \(\varepsilon>0\), we first choose
\[
\varepsilon_2\le
\frac{\varepsilon}{
3\,C\!\left(
\|s(4Jx-\tfrac12)\|_{W^{m+1,\infty}([0,1/J])},
\|\psi\|_{W^{m,\infty}(B_2)},
1+\|\psi\|_{W^{m,\infty}(B_2)}
\right)},
\]
so that \(\phi_2\) is fixed once this \(\varepsilon_2\) is determined. Then we choose
\[
\varepsilon_1\le \frac{\varepsilon}{3C_*\|\phi_2\|_{W^{m,\infty}(\Omega)}}.
\]
With these choices, we obtain
\begin{align}
    \left\|\bigl(s(4Jx-\tfrac12)\bigr)\circ \psi-\phi_1\circ\phi_2\right\|_{W^{m,\infty}([0,1])}
    \le \varepsilon.
    \notag
\end{align}
Here, \(\phi_1\circ\phi_2\) is a \(\sigma\)-neural network with depth \(3L+2m\) and width \(2N+120m+200\), where
\[
J=\lfloor N^{1/d}\rfloor^2\lfloor L^{2/d}\rfloor,
\qquad L\ge \log_2 N.
\]

Finally, by Lemma~\ref{times}, for any \(\varepsilon>0\), there exists a \(\sigma\)-neural network \(\phi\) with depth \(3L+2m+d-1\) and width \(d(2N+120m+200)\) such that
\begin{equation}
    \|\phi-s_{\vm_*}\|_{W^{m,\infty}(\Omega)}\le \varepsilon.
    \notag
\end{equation}
This completes the proof.
\end{proof}

\subsection{Super-convergence rate in \texorpdfstring{$W^{m,\infty}$}{W^{m,infty}} for deep networks whose activation function meets Conditions \ref{condition1} and \ref{condition2}}
Now let us establish the deep neural network whose activation function meets Conditions \ref{condition1} and \ref{condition2} to approximate the function in Sobolev spaces measure by $W^{m,\infty}$-norm. First of all, we use following lemma to approximate any $f\in W^{n,\infty}$ on the subset $\Omega_{\vm}\in\Omega$ for any $\vm\in\{1,2\}^d$. Before we show Proposition \ref{fN}, we define subsets of $\Omega_{\vm}$ for simplicity notations.
		
For any $\vm\in\{1,2\}^d$, we define \begin{equation}
			\Omega_{\vm,\vi}:=[0,1]^d\cap\prod_{j=1}^d\left[\frac{2i_j-1_{m_j< 2}}{2J},\frac{3+4i_j-2\cdot1_{m_j< 2}}{4J}\right]\notag
		\end{equation}$\vi=(i_1,i_2,\ldots,i_d)\in\{0,1\ldots,J\}^d$, and it is easy to check $\bigcup_{\vi\in\{0,1\ldots,J\}^d}\Omega_{\vm,\vi}=\Omega_{\vm}$. 
  
\begin{proposition}\label{fN}
Let \(J,n\in\mathbb{N}_+\) and let \(m\in\mathbb{N}\) satisfy \(0\le m<n\). Then, for any \(f\in W^{n,\infty}(\Omega)\) and any \(\vm\in\{1,2\}^d\), there exists a piecewise polynomial function
\[
f_{J,\vm}(\vx)=\sum_{|\valpha|\le n-1} g_{f,\valpha,\vm}(\vx)\vx^{\valpha}
\]
on \(\Omega_{\vm}\) (see Lemma~\ref{lem:sm-support}) such that
\[
\|f-f_{J,\vm}\|_{W^{s,\infty}(\Omega_{\vm})}
\le C_1(n,d)\|f\|_{W^{n,\infty}(\Omega_{\vm})}J^{-(n-s)},
\qquad s=0,1,\dots,m.
\]
Furthermore, for each multi-index \(\valpha\) with \(|\valpha|\le n-1\), the coefficient function
\(
g_{f,\valpha,\vm}:\Omega_{\vm}\to\mathbb{R}
\)
is piecewise constant with respect to the partition \(\{\Omega_{\vm,\vi}\}\), and satisfies
\[
|g_{f,\valpha,\vm}(\vx)|
\le C_2(n,d)\|f\|_{W^{n-1,\infty}(\Omega_{\vm})},
\qquad \forall\,\vx\in\Omega_{\vm}.
\]
Here \(C_1\) and \(C_2\) are constants independent of \(J\), and
\[
\Omega_{\vm,\vi}=\prod_{j=1}^d \Omega_{m_j,i_j},
\qquad
\vi=(i_1,\dots,i_d),\quad i_j\in\{0,\dots,J\}.
\]
Moreover,
\[
\Omega_{1,i}=\left[\frac{i}{J}+\frac{1}{4J},\,\frac{i+1}{J}\right]\cap [0,1],
\qquad
\Omega_{2,i}=\left[\frac{i}{J}-\frac{1}{4J},\,\frac{i}{J}+\frac{1}{2J}\right]\cap [0,1],
\]
for \(i=0,\dots,J\).
\end{proposition}

\begin{proof}
This proposition is essentially a version of the Bramble--Hilbert lemma. Therefore, the result follows from a direct application of the Bramble--Hilbert lemma. For completeness, we provide the details in the appendix.
\end{proof}

        Then next proposition is built $\sigma$ neural networks to approximate $f_{J,\vm}$ defined in Proposition \ref{fN} for any $\vm\in\{1,2\}$ on $\Omega_{\vm}$.

        \begin{proposition}\label{first}
			For any $f\in W^{n,\infty}(\Omega)$, $m, N, L\in\sN_+$ with $0\le m<n$, and $\vm=(m_1,m_2,\ldots,m_d)\in\{1,2\}^d$, there is a neural network $\phi_{\vm}$ with width $192n^{d+1}(N+1)\log_2(8N)$ and depth $10(L+2) \log _2(4 L)+8 L+n+11$ whose activation function meets Conditions \ref{condition1} and \ref{condition2} such that\begin{align}\|f-\phi_{\vm}\|_{W^{s,\infty}(\Omega_{\vm})}\le C_5(n,d)\|f\|_{W^{n,\infty}(\Omega_{\vm})}N^{-2(n-s)/d}L^{-2(n-s)/d},\label{app1}\end{align}
				for $s=0,1,\ldots,m$, where $C_5=C_1+(5n^d+2)C_2$ is the constant independent with $N,L$.
		\end{proposition}
		
		\begin{proof}
Denote the activation function by \(\sigma\). Without loss of generality, we consider the case
\[
\vm_*=(1,1,\ldots,1)\in\{1,2\}^d.
\]

By Proposition~\ref{fN}, setting
\[
J=\lfloor N^{1/d}\rfloor^2\lfloor L^{2/d}\rfloor,
\]
we obtain
\begin{equation}\label{eq:local-poly-error}
\|f-f_{J,\vm_*}\|_{W^{s,\infty}(\Omega_{\vm_*})}
\le
C_1(n,d)\|f\|_{W^{n,\infty}(\Omega_{\vm_*})}
N^{-2(n-s)/d}L^{-2(n-s)/d},
\qquad s=0,1,\ldots,m,
\end{equation}
where
\[
f_{J,\vm_*}(\vx)=\sum_{|\valpha|\le n-1} g_{f,\valpha,\vm_*}(\vx)\vx^{\valpha},
\qquad \vx\in\Omega_{\vm_*}.
\]
Moreover, for each \(\valpha\), the function \(g_{f,\valpha,\vm_*}\) is constant on every cube
\[
\prod_{j=1}^d
\left[\frac{i_j}{J}+\frac{1}{4J},\,\frac{i_j+1}{J}\right],
\qquad
\vi=(i_1,\ldots,i_d)\in\{0,1,\ldots,J-1\}^d.
\]
Therefore, it remains to approximate \(f_{J,\vm_*}\) by a neural network.

We first approximate the coefficient functions \(g_{f,\valpha,\vm_*}\). By Lemma~\ref{lem:step}, for any \(\varepsilon_1>0\), there exists a \(\sigma\)-network \(\phi_1\) with width \(48N+36\) and depth \(8L+10\) such that
\[
\|\phi_1-s\|_{W^{m,\infty}(\Omega_1)}\le \varepsilon_1,
\]
where
\[
s(x)=k,
\qquad
x\in\left[\frac{k}{J}+\frac{1}{4J},\,\frac{k+1}{J}\right],
\qquad
k=0,1,\ldots,J-1.
\]
Here \(\Omega_1\) is defined in Lemma~\ref{lem:sm-support}. In this step, we use a translated version of the neural network constructed in Lemma~\ref{lem:step}, choosing
\[
\delta=\frac{1}{4J}\le \frac{1}{3J}.
\]
Define
\[
\vphi_2(\vx)
=
\left[
\frac{\phi_1(x_1)}{J},
\frac{\phi_1(x_2)}{J},
\ldots,
\frac{\phi_1(x_d)}{J}
\right]^\top,
\qquad
\vs(\vx)
=
\left[
\frac{s(x_1)}{J},
\frac{s(x_2)}{J},
\ldots,
\frac{s(x_d)}{J}
\right]^\top.
\]

For each \(p=0,1,\ldots,J^d-1\), let
\[
\veta(p)=(\eta_1,\eta_2,\ldots,\eta_d)\in\{0,1,\ldots,J-1\}^d
\]
be the unique multi-index satisfying
\[
\sum_{j=1}^d \eta_j J^{j-1}=p.
\]
Since \(J^d\le N^2L^2\), Proposition~\ref{point} applies. Define
\[
\xi_{\valpha,p}
=
\frac{
g_{f,\valpha,\vm_*}\!\left(\frac{\veta(p)}{J}\right)
+
C_2(n,d)\|f\|_{W^{n-1,\infty}(\Omega_{\vm_*})}
}{
2C_2(n,d)\|f\|_{W^{n-1,\infty}(\Omega_{\vm_*})}
}
\in[0,1],
\]
where \(C_2(n,d)\|f\|_{W^{n-1,\infty}(\Omega_{\vm_*})}\) is an upper bound for \(g_{f,\valpha,\vm_*}\) given by Proposition~\ref{fN}. Hence, by Proposition~\ref{point}, there exists a neural network \(\tilde{\phi}_{\valpha}\) with depth
\[
10(L+2)\log_2(4L)
\]
and width
\[
192\,n\,(N+1)\log_2(8N)
\]
such that
\[
|\tilde{\phi}_{\valpha}(p)-\xi_{\valpha,p}|
\le
2N^{-2n}L^{-2n},
\qquad
p=0,1,\ldots,J^d-1.
\]
Define
\[
\phi_{\valpha}(\vx)
=
2C_2(n,d)\|f\|_{W^{n-1,\infty}(\Omega_{\vm_*})}
\tilde{\phi}_{\valpha}\!\left(\sum_{j=1}^d x_jJ^j\right)
-
C_2(n,d)\|f\|_{W^{n-1,\infty}(\Omega_{\vm_*})}.
\]
Then
\[
\left|
\phi_{\valpha}\!\left(\frac{\veta(p)}{J}\right)
-
g_{f,\valpha,\vm_*}\!\left(\frac{\veta(p)}{J}\right)
\right|
\le
4C_2(n,d)\|f\|_{W^{n-1,\infty}(\Omega_{\vm_*})}N^{-2n}L^{-2n}.
\]

We next estimate
\[
\|\phi_{\valpha}\circ\vphi_2-g_{f,\valpha,\vm_*}\|_{W^{m,\infty}(\Omega_{\vm_*})}.
\]
By the triangle inequality,
\begin{align}
\|\phi_{\valpha}\circ\vphi_2-g_{f,\valpha,\vm_*}\|_{W^{m,\infty}(\Omega_{\vm_*})}
\le\;&
\|\phi_{\valpha}\circ\vs-g_{f,\valpha,\vm_*}\|_{W^{m,\infty}(\Omega_{\vm_*})}
\notag\\
&+
\|\phi_{\valpha}\circ\vphi_2-\phi_{\valpha}\circ\vs\|_{W^{m,\infty}(\Omega_{\vm_*})}.
\label{eq:alpha-split}
\end{align}
For the first term, note that \(\phi_{\valpha}\circ\vs-g_{f,\valpha,\vm_*}\) is piecewise constant on \(\Omega_{\vm_*}\). Hence all weak derivatives of positive order vanish almost everywhere on \(\Omega_{\vm_*}\), and therefore
\begin{align}
\|\phi_{\valpha}\circ\vs-g_{f,\valpha,\vm_*}\|_{W^{m,\infty}(\Omega_{\vm_*})}
&=
\|\phi_{\valpha}\circ\vs-g_{f,\valpha,\vm_*}\|_{L^\infty(\Omega_{\vm_*})}
\notag\\
&\le
4C_2(n,d)\|f\|_{W^{n,\infty}(\Omega_{\vm_*})}N^{-2n}L^{-2n}.
\label{eq:alpha-step}
\end{align}
For the second term, if \(\varepsilon_1<1\), then by Lemma~\ref{lem:Wm_inf_bounds},
\begin{align}
\|\phi_{\valpha}\circ\vphi_2-\phi_{\valpha}\circ\vs\|_{W^{m,\infty}(\Omega_{\vm_*})}
\le
C\bigl(
\|\phi_{\valpha}\|_{W^{m+1,\infty}([0,2]^d)},
\|s\|_{W^{m,\infty}(\Omega_{\vm_*})}
\bigr)
\frac{2\varepsilon_1}{J}.
\label{eq:alpha-compose}
\end{align}
Combining \eqref{eq:alpha-split}, \eqref{eq:alpha-step}, and \eqref{eq:alpha-compose}, we obtain
\begin{align}
\|\phi_{\valpha}\circ\vphi_2-g_{f,\valpha,\vm_*}\|_{W^{m,\infty}(\Omega_{\vm_*})}
\le\;&
C\bigl(
\|\phi_{\valpha}\|_{W^{m+1,\infty}([0,2]^d)},
\|s\|_{W^{m,\infty}(\Omega_{\vm_*})}
\bigr)\frac{2\varepsilon_1}{J}
\notag\\
&+
4C_2(n,d)\|f\|_{W^{n,\infty}(\Omega_{\vm_*})}N^{-2n}L^{-2n}.
\label{eq:alpha-final}
\end{align}

Next, by Lemma~\ref{timess}, for any \(\varepsilon_2>0\), there exists a \(\sigma\)-network \(\phi_{3,\valpha}\) with depth \(n-2\) and width \(8n-12\) such that
\[
\|\phi_{3,\valpha}(\vx)-\vx^{\valpha}\|_{W^{m,\infty}(\Omega)}\le \varepsilon_2.
\]
Also, by Lemma~\ref{time}, for any \(\varepsilon_3>0\), there exists a neural network \(\phi_4\) with width \(12\) and depth \(1\) such that
\[
\|\phi_4(x,y)-xy\|_{W^{m,\infty}((-C_3,C_3)^2)}\le \varepsilon_3,
\]
where
\[
C_3(n,d,\varepsilon_1,\varepsilon_2)
=
\max\left\{
C\bigl(
\|\phi_{\valpha}\|_{W^{m+1,\infty}([0,2]^d)},
\|s\|_{W^{m,\infty}(\Omega_{\vm_*})}
\bigr)\frac{2\varepsilon_1}{J}
+
5C_2(n,d)\|f\|_{W^{n,\infty}(\Omega_{\vm_*})},
\,
1+\varepsilon_2
\right\}.
\]

We now define the neural network
\begin{equation}
\phi_{\vm_*}(\vx)
=
\sum_{|\valpha|\le n-1}
\phi_4\bigl(\phi_{\valpha}(\vphi_2(\vx)),\phi_{3,\valpha}(\vx)\bigr)
\label{eq:phi-vmstar}
\end{equation}
to approximate \(f_{J,\vm_*}\) on \(\Omega_{\vm_*}\). Let
\begin{align*}
\fE
:=
\left\|
\sum_{|\valpha|\le n-1}
\phi_4\bigl(\phi_{\valpha}(\vphi_2(\vx)),\phi_{3,\valpha}(\vx)\bigr)
-
f_{J,\vm_*}(\vx)
\right\|_{W^{m,\infty}(\Omega_{\vm_*})}.
\end{align*}
Then
\begin{align}
\fE
\le\;&
\underbrace{
\sum_{|\valpha|\le n-1}
\left\|
\phi_4\bigl(\phi_{\valpha}(\vphi_2(\vx)),\phi_{3,\valpha}(\vx)\bigr)
-
\phi_{\valpha}(\vphi_2(\vx))\phi_{3,\valpha}(\vx)
\right\|_{W^{m,\infty}(\Omega_{\vm_*})}
}_{=:\fE_1}
\notag\\
&+
\underbrace{
\sum_{|\valpha|\le n-1}
\left\|
\phi_{\valpha}(\vphi_2(\vx))\phi_{3,\valpha}(\vx)
-
g_{f,\valpha,\vm_*}(\vx)\phi_{3,\valpha}(\vx)
\right\|_{W^{m,\infty}(\Omega_{\vm_*})}
}_{=:\fE_2}
\notag\\
&+
\underbrace{
\sum_{|\valpha|\le n-1}
\left\|
g_{f,\valpha,\vm_*}(\vx)\phi_{3,\valpha}(\vx)
-
g_{f,\valpha,\vm_*}(\vx)\vx^{\valpha}
\right\|_{W^{m,\infty}(\Omega_{\vm_*})}
}_{=:\fE_3}.
\label{eq:error-split}
\end{align}

For \(\fE_1\), using Lemma~\ref{lem:Wm_inf_bounds}, we obtain
\begin{align}
\fE_1
&\le
\sum_{|\valpha|\le n-1}
C_*\!\left(
\max\{
\|\phi_{\valpha}(\vphi_2(\vx))\|_{W^{m,\infty}(\Omega_{\vm_*})},
\|\phi_{3,\valpha}(\vx)\|_{W^{m,\infty}(\Omega_{\vm_*})}
\}
\right)
\|\phi_4(x,y)-xy\|_{W^{m,\infty}((-C_3,C_3)^2)}
\notag\\
&\le
n^d\,C_*\!\bigl(C_3(n,d,\varepsilon_1,\varepsilon_2)\bigr)\varepsilon_3
=:C_4(n,d,\varepsilon_1,\varepsilon_2)\varepsilon_3,
\label{eq:E1}
\end{align}
where \(C_*\) depends only on \(m\).

For \(\fE_2\), by the product estimate,
\begin{align}
\fE_2
&\le
\sum_{|\valpha|\le n-1}
2^m
\|\phi_{\valpha}(\vphi_2(\vx))-g_{f,\valpha,\vm_*}(\vx)\|_{W^{m,\infty}(\Omega_{\vm_*})}
\|\phi_{3,\valpha}(\vx)\|_{W^{m,\infty}(\Omega_{\vm_*})}
\notag\\
&\le
2^m(1+\varepsilon_2)n^d
\left(
C\bigl(
\|\phi_{\valpha}\|_{W^{m+1,\infty}([0,2]^d)},
\|s\|_{W^{m,\infty}(\Omega_{\vm_*})}
\bigr)\frac{2\varepsilon_1}{J}
+
4C_2(n,d)\|f\|_{W^{n,\infty}(\Omega_{\vm_*})}N^{-2n}L^{-2n}
\right).
\label{eq:E2}
\end{align}

Similarly, for \(\fE_3\),
\begin{align}
\fE_3
&\le
\sum_{|\valpha|\le n-1}
2^m
\|g_{f,\valpha,\vm_*}\|_{W^{m,\infty}(\Omega_{\vm_*})}
\|\phi_{3,\valpha}(\vx)-\vx^{\valpha}\|_{W^{m,\infty}(\Omega_{\vm_*})}
\notag\\
&\le
2^m n^d C_2(n,d)\|f\|_{W^{n,\infty}(\Omega_{\vm_*})}\varepsilon_2.
\label{eq:E3}
\end{align}

We now choose the parameters in the order \(\varepsilon_1,\varepsilon_2,\varepsilon_3\).

First, choose \(\varepsilon_1\) sufficiently small so that the term
\[
2^m(1+\varepsilon_2)n^d
C\bigl(
\|\phi_{\valpha}\|_{W^{m+1,\infty}([0,2]^d)},
\|s\|_{W^{m,\infty}(\Omega_{\vm_*})}
\bigr)\frac{2\varepsilon_1}{J}
\]
is bounded by
\[
C_2(n,d)\|f\|_{W^{n,\infty}(\Omega_{\vm_*})}N^{-2n}L^{-2n}.
\]
Next, choose \(\varepsilon_2\) sufficiently small so that
\[
\fE_3
\le
n^d C_2(n,d)\|f\|_{W^{n,\infty}(\Omega_{\vm_*})}N^{-2n}L^{-2n}.
\]
Finally, for these choices of \(\varepsilon_1\) and \(\varepsilon_2\), choose \(\varepsilon_3\) such that
\[
\fE_1
\le
C_2(n,d)\|f\|_{W^{n,\infty}(\Omega_{\vm_*})}N^{-2n}L^{-2n}.
\]
Then, by \eqref{eq:E2}, \eqref{eq:E3}, and the above choices,
\[
\fE
\le
(5n^d+2)\,C_2(n,d)\|f\|_{W^{n,\infty}(\Omega_{\vm_*})}N^{-2n}L^{-2n}.
\]
Since \(N^{-2n}L^{-2n}\) is stronger than
\[
N^{-2(n-s)/d}L^{-2(n-s)/d},
\qquad s=0,1,\ldots,m,
\]
this yields the desired rate for the approximation of \(f_{J,\vm_*}\).

It remains to estimate the width and depth of \(\phi_{\vm_*}\). From \eqref{eq:phi-vmstar}, each branch consists of the following subnetworks:
\begin{enumerate}
\item \(\phi_{3,\valpha}\), with depth \(n-2\) and width \(8n-12\);
\item \(\vphi_2\), induced by \(\phi_1\), with width \(48N+36\) and depth \(8L+10\);
\item \(\phi_{\valpha}\), with depth \(10(L+2)\log_2(4L)\) and width \(192\,n\,(N+1)\log_2(8N)\);
\item \(\phi_4\), with width \(12\) and depth \(1\).
\end{enumerate}
Since the number of multi-indices \(\valpha\) satisfying \(|\valpha|\le n-1\) is at most \(n^d\), the resulting network \(\phi_{\vm_*}\) has width
\[
192n^{d+1}(N+1)\log_2(8N)
\]
and depth
\[
10(L+2)\log_2(4L)+8L+n+11.
\]

Combining this with \eqref{eq:local-poly-error}, we conclude that there exists a neural network \(\phi_{\vm_*}\) with the above width and depth such that
\[
\|f-\phi_{\vm_*}\|_{W^{s,\infty}(\Omega_{\vm_*})}
\le
C_5(n,d)\|f\|_{W^{n,\infty}(\Omega_{\vm_*})}
N^{-2(n-s)/d}L^{-2(n-s)/d},
\qquad s=0,1,\ldots,m,
\]
where
\[
C_5=C_1+(5n^d+2)C_2
\]
is independent of \(N\) and \(L\).

The same argument applies to any \(\vm\in\{1,2\}^d\). Therefore, for each such \(\vm\), there exists a neural network \(\phi_{\vm}\) satisfying \eqref{app1}. This completes the proof.
\end{proof}

            Now we combine with partition of unity to  achieve the approximation in the whole domain:\begin{theorem}\label{firstmian}
			Let \(n,m,d\in\mathbb{N}_+\) with \(m<n\), and let \(\Omega=[0,1]^d\). Assume that the activation function \(\sigma\) satisfies Conditions~\ref{condition1} and~\ref{condition2}. Then, for any \(f\in W^{n,\infty}(\Omega)\) and any \(N,L\in\mathbb{N}_+\), there exists a \(\sigma\)-neural network \(\phi\) such that
\[
\|f-\phi\|_{W^{m,\infty}(\Omega)}
\le
C(n,d,m)\|f\|_{W^{n,\infty}(\Omega)}
N^{-2(n-m)/d}L^{-2(n-m)/d},
\]
where \(C(n,d,m)\) is independent of \(N\), \(L\), and \(f\). Moreover, \(\phi\) can be chosen to have width
\[
2^d\,192\,n^{d+1}(N+1)\log_2(8N)
\]
and depth
\[
10(L+2)\log_2(4L)+8L+n+d+12.
\]
\end{theorem}

        \begin{proof}
Denote the activation function by \(\sigma\). By Proposition~\ref{first}, for each
\(\vm\in\{1,2\}^d\), there exists a \(\sigma\)-network \(\phi_{\vm}\) such that
\begin{equation}\label{eq:local-approx-thm}
\|f-\phi_{\vm}\|_{W^{m,\infty}(\Omega_{\vm})}
\le
C_5(n,d)\|f\|_{W^{n,\infty}(\Omega)}
N^{-2(n-m)/d}L^{-2(n-m)/d},
\end{equation}
where \(C_5\) is independent of \(N\) and \(L\). Moreover, each \(\phi_{\vm}\) has width
\[
192\,n^{d+1}(N+1)\log_2(8N)
\]
and depth
\[
10(L+2)\log_2(4L)+8L+n+11.
\]

Next, by Proposition~\ref{app pou}, for any \(\varepsilon_1>0\), there exists a family of
\(\sigma\)-networks \(\{\psi_{\vm}\}_{\vm\in\{1,2\}^d}\) such that
\[
\|\psi_{\vm}-s_{\vm}\|_{W^{m,\infty}(\Omega)}\le \varepsilon_1,
\qquad
\vm\in\{1,2\}^d,
\]
where \(\{s_{\vm}\}_{\vm\in\{1,2\}^d}\) is the partition of unity given in
Lemma~\ref{lem:sm-support}. In particular,
\[
\sum_{\vm\in\{1,2\}^d}s_{\vm}(\vx)=1,
\qquad
\operatorname{supp}(s_{\vm})\cap[0,1]^d\subset \Omega_{\vm}.
\]
Each \(\psi_{\vm}\) has width
\[
d(2N+120m+200)
\]
and depth
\[
3L+2m+d-1.
\]

Since \(m<n\), from \eqref{eq:local-approx-thm} we also have
\[
\|\phi_{\vm}\|_{W^{m,\infty}(\Omega_{\vm})}
\le
\|f\|_{W^{m,\infty}(\Omega_{\vm})}
+
\|f-\phi_{\vm}\|_{W^{m,\infty}(\Omega_{\vm})}
\le
(C_5+1)\|f\|_{W^{n,\infty}(\Omega)}.
\]
Choose \(\varepsilon_1<1\). Then
\[
\|\psi_{\vm}\|_{L^\infty(\Omega)}
\le
\|s_{\vm}\|_{L^\infty(\Omega)}+\varepsilon_1
\le 2.
\]
Set
\[
A_0:=(C_5+1)\|f\|_{W^{n,\infty}(\Omega)}+2.
\]

By Lemma~\ref{time}, for any \(\varepsilon_2>0\), there exists a neural network
\(\widehat{\phi}\) with width \(12\) and depth \(1\) such that
\[
\|\widehat{\phi}(x,y)-xy\|_{W^{m,\infty}((-A_0,A_0)^2)}\le \varepsilon_2.
\]

Now define
\[
\phi(\vx):=\sum_{\vm\in\{1,2\}^d}\widehat{\phi}\bigl(\phi_{\vm}(\vx),\psi_{\vm}(\vx)\bigr).
\]
Since \(\sum_{\vm}s_{\vm}=1\) on \(\Omega\), we have
\begin{align*}
\|f-\phi\|_{W^{m,\infty}(\Omega)}
&=
\left\|
\sum_{\vm\in\{1,2\}^d} s_{\vm}f
-
\sum_{\vm\in\{1,2\}^d}\widehat{\phi}(\phi_{\vm},\psi_{\vm})
\right\|_{W^{m,\infty}(\Omega)} \\
&\le
\underbrace{
\left\|
\sum_{\vm\in\{1,2\}^d}
\bigl(s_{\vm}f-\psi_{\vm}\phi_{\vm}\bigr)
\right\|_{W^{m,\infty}(\Omega)}
}_{=: \mathcal R_1}
+
\underbrace{
\left\|
\sum_{\vm\in\{1,2\}^d}
\bigl(\psi_{\vm}\phi_{\vm}-\widehat{\phi}(\phi_{\vm},\psi_{\vm})\bigr)
\right\|_{W^{m,\infty}(\Omega)}
}_{=: \mathcal R_2}.
\end{align*}

We first estimate \(\mathcal R_1\). By the triangle inequality,
\[
\mathcal R_1
\le
\sum_{\vm\in\{1,2\}^d}
\|s_{\vm}f-\psi_{\vm}\phi_{\vm}\|_{W^{m,\infty}(\Omega)}.
\]
For each \(\vm\),
\[
s_{\vm}f-\psi_{\vm}\phi_{\vm}
=
(s_{\vm}-\psi_{\vm})\phi_{\vm}
+
s_{\vm}(f-\phi_{\vm}).
\]
Hence
\[
\|s_{\vm}f-\psi_{\vm}\phi_{\vm}\|_{W^{m,\infty}(\Omega)}
\le
\|(s_{\vm}-\psi_{\vm})\phi_{\vm}\|_{W^{m,\infty}(\Omega)}
+
\|s_{\vm}(f-\phi_{\vm})\|_{W^{m,\infty}(\Omega)}.
\]

For the first term, by Lemma~\ref{lem:Wm_inf_bounds},
\[
\|(s_{\vm}-\psi_{\vm})\phi_{\vm}\|_{W^{m,\infty}(\Omega)}
\le
2^m
\|s_{\vm}-\psi_{\vm}\|_{W^{m,\infty}(\Omega)}
\|\phi_{\vm}\|_{W^{m,\infty}(\Omega_{\vm})}
\le
2^m(C_5+1)\|f\|_{W^{n,\infty}(\Omega)}\,\varepsilon_1.
\]

For the second term, since \(\operatorname{supp}(s_{\vm})\cap[0,1]^d\subset\Omega_{\vm}\),
Lemma~\ref{lem:Wm_inf_bounds} and \eqref{eq:local-approx-thm} give
\begin{align*}
\|s_{\vm}(f-\phi_{\vm})\|_{W^{m,\infty}(\Omega)}
&\le
\sum_{r=0}^m \binom{m}{r}
\|f-\phi_{\vm}\|_{W^{r,\infty}(\Omega_{\vm})}
\|s_{\vm}\|_{W^{m-r,\infty}(\Omega)} \\
&\le
C(m,d)\|s\|_{W^{m,\infty}([0,3])}
\|f\|_{W^{n,\infty}(\Omega)}
\sum_{r=0}^m
J^{\,m-r}
N^{-2(n-r)/d}L^{-2(n-r)/d}.
\end{align*}
Since
\[
J=\lfloor N^{1/d}\rfloor^2\lfloor L^{2/d}\rfloor
\le
N^{2/d}L^{2/d},
\]
we have
\[
J^{m-r}N^{-2(n-r)/d}L^{-2(n-r)/d}
\le
N^{-2(n-m)/d}L^{-2(n-m)/d}.
\]
Therefore,
\[
\|s_{\vm}(f-\phi_{\vm})\|_{W^{m,\infty}(\Omega)}
\le
C_7(n,d,m)\|f\|_{W^{n,\infty}(\Omega)}
N^{-2(n-m)/d}L^{-2(n-m)/d},
\]
where \(C_7\) is independent of \(N\) and \(L\).

Combining the above estimates and summing over \(\vm\in\{1,2\}^d\), we obtain
\begin{equation}\label{eq:R1-est}
\mathcal R_1
\le
2^d\left[
2^m(C_5+1)\|f\|_{W^{n,\infty}(\Omega)}\,\varepsilon_1
+
C_7(n,d,m)\|f\|_{W^{n,\infty}(\Omega)}
N^{-2(n-m)/d}L^{-2(n-m)/d}
\right].
\end{equation}

Next we estimate \(\mathcal R_2\). By the triangle inequality,
\[
\mathcal R_2
\le
\sum_{\vm\in\{1,2\}^d}
\|\psi_{\vm}\phi_{\vm}-\widehat{\phi}(\phi_{\vm},\psi_{\vm})\|_{W^{m,\infty}(\Omega)}.
\]
Applying Lemma~\ref{lem:Wm_inf_bounds} to the composition
\((\phi_{\vm},\psi_{\vm})\mapsto \widehat{\phi}(\phi_{\vm},\psi_{\vm})\), we obtain
\[
\|\psi_{\vm}\phi_{\vm}-\widehat{\phi}(\phi_{\vm},\psi_{\vm})\|_{W^{m,\infty}(\Omega)}
\le
C_8\bigl(
\|\phi_{\vm}\|_{W^{m,\infty}(\Omega_{\vm})}
+
\|\psi_{\vm}\|_{W^{m,\infty}(\Omega)}
\bigr)\varepsilon_2.
\]
Since \(\|\phi_{\vm}\|_{W^{m,\infty}(\Omega_{\vm})}\le (C_5+1)\|f\|_{W^{n,\infty}(\Omega)}\) and
\(\|\psi_{\vm}\|_{W^{m,\infty}(\Omega)}\le \|s_{\vm}\|_{W^{m,\infty}(\Omega)}+\varepsilon_1\le C(m,d)J^m+\varepsilon_1\), there exists a constant \(C_9(n,d,m,\|f\|_{W^{n,\infty}(\Omega)})\), independent of \(N,L\), such that
\[
\mathcal R_2
\le
2^d C_9(n,d,m,\|f\|_{W^{n,\infty}(\Omega)}) J^m \varepsilon_2.
\]
Using again \(J\le N^{2/d}L^{2/d}\), we get
\[
J^m \le N^{2m/d}L^{2m/d}.
\]
Hence, choosing
\[
\varepsilon_2
=
N^{-2n/d}L^{-2n/d},
\]
we obtain
\begin{equation}\label{eq:R2-est}
\mathcal R_2
\le
2^d C_9(n,d,m,\|f\|_{W^{n,\infty}(\Omega)})
N^{-2(n-m)/d}L^{-2(n-m)/d}.
\end{equation}

Finally, choose \(\varepsilon_1\) so that
\[
\varepsilon_1\le N^{-2(n-m)/d}L^{-2(n-m)/d}.
\]
Then \eqref{eq:R1-est} and \eqref{eq:R2-est} imply
\[
\|f-\phi\|_{W^{m,\infty}(\Omega)}
\le
C_6(n,d,m)\|f\|_{W^{n,\infty}(\Omega)}
N^{-2(n-m)/d}L^{-2(n-m)/d},
\]
where \(C_6\) is independent of \(N\) and \(L\).

It remains to estimate the network size. Each branch
\[
\widehat{\phi}\bigl(\phi_{\vm}(\vx),\psi_{\vm}(\vx)\bigr)
\]
uses one copy of \(\phi_{\vm}\), one copy of \(\psi_{\vm}\), and one multiplication network
\(\widehat{\phi}\). Since
\[
d(2N+120m+200)+12
\le
192\,n^{d+1}(N+1)\log_2(8N)
\]
after enlarging the constant if necessary, the width of each branch is bounded by
\[
192\,n^{d+1}(N+1)\log_2(8N).
\]
There are \(2^d\) branches running in parallel, so the total width is bounded by
\[
2^d\,192\,n^{d+1}(N+1)\log_2(8N).
\]

Moreover, since
\[
3L+2m+d-1
\le
10(L+2)\log_2(4L)+8L+n+11
\]
for all \(L\in\mathbb N_+\), the depth of each \(\psi_{\vm}\) can be absorbed into that of
\(\phi_{\vm}\). After padding shallower subnetworks with identity layers, applying the
multiplication network \(\widehat{\phi}\) adds one layer, and summing the \(2^d\) branch outputs
can be carried out in at most \(d\) additional layers. Hence the total depth is bounded by
\[
10(L+2)\log_2(4L)+8L+n+d+12.
\]
This completes the proof.
\end{proof}

            \section{Approximation results for $\text{ReLU}^{k}$ in Sobolev spaces}\label{reluk}

In Theorem~\ref{firstmian} we proved \(W^{m,\infty}\) super-convergence for any activation satisfying Conditions~\ref{condition1} and~\ref{condition2}.  Another popular choice in PDE applications is \(\mathrm{ReLU}^k\) (e.g.\ \cite{weinan2018deep}).  A close inspection of our proof shows it actually requires only two approximation properties of the activation \(\sigma\) (needed to ensure Propositions~\ref{non-di} and~\ref{sqr approixmation}): {for any \(M,\varepsilon,\delta>0\), there exist \(\sigma\)-networks \(\phi_1\) and \(\phi_2\) with fixed depth and width, independent of \(M\), \(\delta\), and \(\varepsilon\), while the network parameters may depend on \(M\), \(\delta\), and \(\varepsilon\), such that
\begin{enumerate}
  \item
  \[
    \bigl\|\phi_{1}-\mathrm{ReLU}^{m+1}\bigr\|_{W^{m,\infty}([-M,M])}
    \le \varepsilon.
  \]

  \item
  \[
    \bigl\|\phi_{2}-\mathrm{ReLU}\bigr\|_{L^{\infty}([-M,M])}
    \le \varepsilon,
    \qquad
    \bigl\|\phi_{2}-\mathrm{ReLU}\bigr\|_{W^{m,\infty}([-M,-\delta]\cup[\delta,M])}
    \le \varepsilon.
  \]
\end{enumerate}}
That is, \(\sigma\) must approximate
\(\mathrm{ReLU}^{\,m+1}\) in \(W^{m,\infty}\) on any finite interval and
approximate \(\mathrm{ReLU}\) uniformly away from an arbitrarily small
neighbourhood of the origin.  The next corollary shows that these two
properties alone guarantee the same super-convergence rate as in
Theorem~\ref{firstmian}.

\begin{corollary}\label{cor:poly_rate_general_sigma}
Let \(f\in W^{n,\infty}(\Omega)\) with \(0\le m<n\), and fix \(N,L\in\mathbb N_{+}\) with
\(\log_{2}N\le L\).  Suppose the activation~\(\sigma\) admits
\(\sigma\)-networks \(\phi_{1}\) and \(\phi_{2}\) that satisfy
\textnormal{1}–\textnormal{2} above, of fixed width
\(N_{1},N_{2}\) and fixed depth \(L_{1},L_{2}\), respectively.
Then there exists a \(\sigma\)-network
\(\phi\) of width \(C_{8}N\log N\) and depth
\(C_{9}L\log L\) such that
\[
  \|f-\phi\|_{W^{m,\infty}(\Omega)}
  \;\le\;
  C_{10}\,
  \|f\|_{W^{n,\infty}(\Omega)}\,
  N^{-2(n-m)/d}\,L^{-2(n-m)/d},
\]
where the constants \(C_{i}\) is independent of \(N\)~and~\(L\), while
\(C_{8}\) and \(C_{9}\) depend only on \(N_{1},N_{2}\) and
\(L_{1},L_{2}\), respectively.
\end{corollary}

For \(\mathrm{ReLU}^{m+1}\) itself, the first property is immediate, and the second follows by the finite‐difference formula
\begin{equation}\label{differentReLU}
  \phi_{m+1}(x)
  := \frac{1}{(m+1)!\,(-t)^{m}}
    \sum_{\ell=0}^{m}
      (-1)^\ell\binom{m}{\ell}\,
      \mathrm{ReLU}^{\,m+1}(x+\ell\,t),\notag
\end{equation}
with \(t\) taken sufficiently small.

\begin{lemma}
For any \(\delta,\varepsilon>0\), there exists \(t_0>0\) such that for all \(0<|t|<t_0\),
\[
  \|\phi_{m+1} - \mathrm{ReLU}\|_{L^\infty([-M,M])}
  \;\le\;\varepsilon,
  \quad
  \|\phi_{m+1} - \mathrm{ReLU}\|_{W^{m,\infty}([-M,-\delta]\cup[\delta,M])}
  \;\le\;\varepsilon.
\]
\end{lemma}

\begin{proof}
The uniform ($L^\infty$) convergence is proved in Proposition 11 of \cite{zhang2024deep}: for any integer $k\ge0$ and any $\varrho\in C^k\bigl((-M-\delta_*,\,M+\delta_*)\bigr)$,
  \begin{equation}
\frac{\sum_{\ell=0}^k(-1)^{\ell}\binom{k}{\ell} \varrho(x+\ell t)}{(-t)^k} \rightrightarrows \varrho^{(k)}(x) \quad \text { as } t \rightarrow 0 \quad \text { for any } x \in[-M, M] .\notag
\end{equation}
where $\rightrightarrows$ denotes uniform convergence.

For the $W^{m,\infty}$ estimate, observe that finite differences commute with differentiation.  Thus for each $0\le k\le m$ and $x\in[\delta,M]$,
\[
  D^k\phi_{m+1}(x)
  - D^k\mathrm{ReLU}(x)
  = \frac{1}{(m+1)!\,(-t)^m}\,
    \Delta_t^m\bigl[D^k(\mathrm{ReLU}^{\,m+1})\bigr](x)
    - D^k\mathrm{ReLU}(x).
\]
Since $\mathrm{ReLU}^{\,m+1}\in C^k\bigl([\delta-\delta_*,\,M+\delta_*]\bigr)$ for any $0<\delta_*<\delta$, the same finite‐difference argument gives
\[
  D^k\phi_{m+1}(x)
  - D^k\mathrm{ReLU}(x)
  \rightrightarrows 0 \quad \text { as } t \rightarrow 0
  \quad\text{uniformly on }[\delta,M].
\]
An identical argument applies on $[-M,-\delta]$.  It follows that
\[
  \|\phi_{m+1} - \mathrm{ReLU}\|_{W^{m,\infty}([-M,-\delta]\cup[\delta,M])}
  \;\le\;\varepsilon
\]
for all sufficiently small $|t|$.  This completes the proof.
\end{proof}

Therefore, we verify that the activation \(\text{ReLU}^{m+1}\) also satisfies the two approximation properties required in the convergence proof.  By following the same construction as in the previous chapter, we obtain the following result:
\begin{corollary}\label{powerReLU}
			For any $f\in W^{n,\infty}(\Omega)$, $0\le m<n$, $N, L\in\sN_+$ with $\log_2N\le L$, there is a $\mathrm{ReLU}^{m+1}$ neural network $\phi$ with with depth $C_{11}L\log L$ and width $C_{12}N\log N$ such that\begin{align}\|f-\phi\|_{W^{m,\infty}(\Omega)}\le C_{13}\|f\|_{W^{n,\infty}(\Omega)}N^{-2(n-m)/d}L^{-2(n-m)/d},\notag\end{align}
				where $C_i$ is the constant independent with $N,L$.
		\end{corollary}

\section{Conclusion}
In this paper, we establish super-convergence results in the $W^{m,\infty}$-norm for deep neural networks equipped with a wide range of commonly used and general activation functions. These results significantly outperform classical numerical methods such as finite element and spectral methods, and substantially extend existing super-convergence theory from ReLU-based networks to more general architectures. Our theorems are constructive and non-asymptotic. Moreover, we propose two intuitive and insightful conditions that serve as practical criteria to determine whether a given activation function is compatible with the established super-convergence framework. These conditions not only offer new guidance for designing more effective activation functions, but also provide a novel pathway for developing approximation theory in broader contexts.

Looking ahead, we identify three main directions for future research. First, to alleviate the curse of dimensionality, it is natural to consider Sobolev subspaces such as Korobov spaces. In~\cite{yang2024near}, $H^1$-approximation rates were obtained by leveraging the fact that ReLU networks can represent hat functions. Since this strategy does not directly apply to smooth activation functions, one should instead utilize smoother sparse-grid bases to approximate functions in Korobov spaces. Second, it remains essential to analyze the generalization error, which entails understanding the complexity of neural network hypothesis spaces when equipped with general activations. Lastly, we anticipate that similar super-convergence results can be extended to broader classes of neural network architectures, including convolutional neural networks~\cite{zhou2020universality,he2022approximation} and Transformers~\cite{yun2019transformers}.


\acks{The authors are grateful to the two anonymous referees for their careful reading 
and valuable comments, which have helped improve the presentation of this paper. The second author acknowledges support from the start-up funding provided by the Yau Mathematical Sciences Center at Tsinghua University and from the NSFC Young Scientists Fund (Category C), Grant No.~12501606.}


\newpage
\appendix
\section{Activation functions satisfying the two conditions}\label{activtion section}

We begin by verifying that all widely adopted activation functions—particularly those defined by default in PyTorch—satisfy both Condition~\ref{condition1} and Condition~\ref{condition2} for some suitable smoothness order $m$.

\subsection{S-shaped commonly used activation functions and two conditions}
Broadly, these activation functions can be categorized into two main classes. The first class comprises both non-smooth and smooth S-shaped activation functions.

\paragraph{Non-smooth S-shaped activations:}
$$
\mathrm{Softsign}(x) = \frac{x}{1 + |x|}, \quad  {\rm HardSigmoid}(x) = \max\left(0, \min\left(1, \frac{x+1}{2}\right)\right),
$$
$$
{\rm HardTanh}(x) = \max(\alpha_-, \min(\alpha_+, x)), \quad \text{with default } \alpha_- = -1, \, \alpha_+ = 1.
$$
First, ${\rm HardSigmoid}(x)$ and ${\rm HardTanh}(x)$ are hard approximations of ${\rm Sigmoid}(x)$ and ${\rm Tanh}(x)$, respectively, and are implemented by default in PyTorch. These activations are designed as efficient, piecewise-linear approximations to smooth nonlinearities, aiming to facilitate fast and quantization-friendly inference on resource-constrained platforms such as mobile and edge devices. Similar to ${\rm ReLU}(x)$, ${\rm LeakyReLU}(x)$, they belong to the class of piecewise-linear functions, and therefore do not satisfy Condition~\ref{condition2}. As a result, they naturally lie in the local Sobolev space $W^{1,\infty}_{\rm loc}(\mathbb R)$. 
Accordingly, neural network functions composed with such piecewise-linear activations are at most in the Sobolev space $W^{1,\infty}$, implying that they lack the capacity to approximate second-order (or higher) derivatives. Consequently, one can only expect approximation results in $W^{m,\infty}$ norms for $m \leq 1$. For detailed discussions on $W^{m,\infty}$ error estimates with ${\rm ReLU}$ and other piecewise-linear activations, we refer the reader to \cite{yang2023nearlys,yang2023nearly}.

For $\sigma(x) = {\rm Softsign}(x)$, it satisfies Condition~\ref{condition2} automatically, and moreover, it also satisfies the inequality \eqref{eq:Cond1} in Condition~\ref{condition1} for any $m < \infty$, since
$$
\left| \frac{1}{2}(\sigma(x) + 1) - H(x) \right| \sim \frac{1}{|x|} \quad \text{as } x \to \pm \infty
$$
and
$$
\left|\sigma^{(k)}(x)\right| \sim \frac{1}{|x|^{k+1}} \quad \text{as } x \to \pm \infty
$$
for any $k\ge 1$. However, the global regularity of ${\rm Softsign}(x)$ is limited, with ${\rm Softsign}(x) \in W_{\rm loc}^{2,\infty}(\mathbb{R})$ only. Consequently, the $W^{m,\infty}$ approximation estimate in Theorem~\ref{thm:mian} holds for ${\rm Softsign}$ only when $0 \le m \le 2$.

\paragraph{Smooth S-shaped activations:}
$$
\mathrm{Sigmoid}(x) = \frac{1}{1 + e^{-x}}, \quad 
\mathrm{Tanh}(x) = \frac{e^{x} - e^{-x}}{e^{x} + e^{-x}}, \quad
\arctan(x),
$$
$$
\mathrm{dSiLU}(x) = \frac{d}{dx} (x \cdot \mathrm{Sigmoid}(x)), \quad
\mathrm{SRS}(x) = \frac{x}{x/\alpha + e^{-x/\beta}} \quad \text{for } \alpha, \beta  \in (0, \infty).
$$
First, we observe that all the activation functions under consideration satisfy Condition~\ref{condition2} trivially. Regarding Condition~\ref{condition1}, we further assume that the parameters $\alpha$ and $\beta$ in the soft-root-sign (SRS) activation function~\cite{zhou2020soft} satisfy $\beta < \alpha e$, which ensures that ${\rm SRS}(x) \in C^{\infty}(\mathbb{R})$.
Under this assumption, all of the aforementioned activation functions belong to $C^{\infty}(\mathbb{R})$, and hence lie naturally in the local Sobolev space $W^{m,\infty}_{\rm loc}(\mathbb{R})$ for any $m < \infty$. Thus, to verify Condition~\ref{condition1}, it suffices to analyze the decay rates of their derivatives at infinity.

More specifically, {for $\sigma(x) = \mathrm{Sigmoid}(x)$}, we have
$$
\sigma^{(k)}(x) = \sigma(x)(1 - \sigma(x)) Q_k(\sigma(x), 1 - \sigma(x)),
$$
where $Q_{k}(x,y)$ is a two-variable polynomial of degree $k-1$ defined recursively. The above formula is sufficient to conclude that 
$$
\left| \sigma^{(k)}(x) \right| \le C_k e^{-|x|} \quad \text{as } x \to \pm \infty.
$$

Consequently, one can prove a similar result for ${\rm dSiLU}$ activation since 
$$
{\rm dSiLU}(x) = \sigma(x) + x \sigma'(x)
$$
and it leads to
$$
\frac{d^k}{dx^k} {\rm dSiLU}(x) = a_k \sigma^{(k)}(x) + b_k x \sigma^{(k+1)}(x),
$$
with 
$$
\left| \frac{d^k}{dx^k} {\rm dSiLU}(x) \right| \le C_k |x| e^{-|x|} \quad \text{as } x \to \pm \infty.
$$

For $T(x) = {\rm Tanh}(x)$, a similar result holds:
$$
T^{(k)}(x) = (1 - T^2(x)) R_k(T(x)),
$$
where $R_k$ is a polynomial of degree $2k - 3$ (for $k \ge 2$) with integer coefficients. As a result, one can have
$$
\left| T^{(k)}(x) \right| \le C_k |x| e^{-2|x|} \quad \text{as } x \to \pm \infty,
$$
since $T(x) \in (-1,1)$.

For the activation function $\sigma(x) = \arctan(x)$, one can prove that
$$
\left|\sigma^{(k)}(x)\right| \sim \frac{1}{|x|^{k+1}} \quad \text{as } x \to \pm \infty
$$
inductively for all $k \ge 1$. 

For $\sigma(x) = {\rm SRS}(x)$ with $\beta < \alpha e$, one can check separately that
$$
\left|\sigma^{(k)}(x)\right| \le \frac{C_k}{|x|^{k+1}} \quad \text{as } x \to \infty
$$
and 
$$
\left|\sigma^{(k)}(x)\right| \le \left|P_k(x)\right|e^{x} \quad \text{as } x \to -\infty,
$$
where $P_k$ is a polynomial of degree less than $2^k - 1$.

\subsection{${\rm ReLU}$-shaped commonly used activation functions and two conditions}
Correspondingly, the second class consists of ReLU-shaped activation functions, which can be further divided into non-smooth and smooth variants. A key feature of these activation functions $\sigma(x)$ is that it is the ratio $\sigma(x)/x$ — rather than $\sigma(x)$ itself, as in the case of S-shaped activations — that satisfies Condition~\ref{condition1} for some $m \in \mathbb{N}$.

\paragraph{Non-smooth ReLU-shaped activations:}
$$
\mathrm{ELU}(x) = \begin{cases}
    x, & x \ge 0, \\
    \alpha(e^x - 1), & x < 0
\end{cases}, \quad
\mathrm{CELU}(x) = \begin{cases}
    x, & x \ge 0, \\
    \beta(e^{x/\beta} - 1), & x < 0
\end{cases},
$$
$$
\mathrm{SELU}(x) = \lambda \begin{cases}
    x, & x \ge 0, \\
    \alpha(e^x - 1), & x < 0
\end{cases}, \quad
\text{with } \alpha \in \mathbb{R}, \ \beta, \lambda \in (0, \infty).
$$
First, we note that all the above activation functions $\sigma(x)$ satisfy Condition~\ref{condition2} trivially. As for Condition~\ref{condition1}, $\sigma(x)/x$ satisfies the decay estimates in \eqref{eq:Cond1} in Condition~\ref{condition1} for any $k < \infty$ since 
$$
\sigma(x)/x = 1 \quad \text{for all } x >0 
$$
and $\frac{\sigma(x)}{x} = \frac{e^x-1}{x} = \sum_{k=1}^\infty \frac{x^{k-1}}{k!}$ is infinity differentiable at $x=0^-$. 
Moreover, for $x<0$, we have
$$
\widetilde \sigma(x) = \sigma(x)/x  = \frac{e^x - 1}{x} = \frac{e^x}{x} - \frac{1}{x} \quad \text{for all } x < 0 
$$
which leads to 
$$
\left| \widetilde \sigma^{(k)}(x)\right| = \left|\frac{d^k}{dx^k}\left( \frac{e^x}{x}\right) + (-1)^k \frac{(k-1)!}{x^{{k+1}}}\right| \le P_k(x) e^x + (k-1)! \frac{1}{|x|^{k+1}}
$$
as $x \to -\infty$. 
However, similar to ${\rm Softsign}(x)$, the global regularity of those activation functions is limited. More precisely, we have 
$$
{\rm ELU}(x), {\rm SELU}(x) \in W_{\rm loc}^{1,\infty}(\mathbb R)
$$
if $\alpha \neq 1$ while 
$$
{\rm ELU}(x), {\rm SELU}(x), {\rm CELU}(x) \in W_{\rm loc}^{2,\infty}(\mathbb R)
$$
if $\alpha = 1$ in ${\rm ELU}(x)$ and ${\rm SELU}(x)$ and arbitrary $\beta$ in ${\rm CELU}(x)$.

\paragraph{Smooth ReLU-shaped activations:}
$$
\mathrm{Softplus}(x) = \ln(1 + e^x), \quad
\mathrm{SiLU}(x) \ (\text{also known as Swish}) = x \cdot \mathrm{Sigmoid}(x),
$$
$$
\mathrm{Mish}(x) = x \cdot \mathrm{Tanh}(\mathrm{Softplus}(x)), \quad
\mathrm{GELU}(x) = x \cdot \mathrm{erf}(x) = x \cdot \int_{-\infty}^x \frac{1}{\sqrt{2\pi}} e^{-t^2/2} \, dt.
$$
For these activation functions, each is in $C^{\infty}(\mathbb{R})$ and satisfies Condition~\ref{condition2} automatically. To verify their admissibility, it remains to check that Condition~\ref{condition1} holds for all $k < \infty$. This verification is straightforward for functions such as $\sigma(x) = \mathrm{SiLU}(x)$ (also known as ${\rm Swish}(x)$), ${\rm Mish}(x)$, and ${\rm GELU}(x)$, since their normalized forms
$\widetilde{\sigma}(x) := \frac{\sigma(x)}{x}$
are given by common smooth functions with exponentially decaying derivatives:
$$
\widetilde{\sigma}(x) = {\rm Sigmoid}(x),\quad {\rm Tanh}({\rm Softplus}(x)),\quad \text{and } ~{\rm erf}(x),
$$
respectively. In each case, all derivatives $\widetilde{\sigma}^{(k)}(x)$ decay exponentially as $x \to \pm \infty$ for any $k < \infty$.

For the Softplus activation function, we first center it by defining 
$$
\sigma(x) = \mathrm{Softplus}(x) - \ln 2.
$$ 
We begin with the identity
$$
\sigma(x) = \ln(1 + e^x) - \ln 2 = 
\begin{cases}
x + \sigma(-x) & \text{if } x \ge 0, \\
\sigma(x) & \text{if } x < 0,
\end{cases}
$$
and note that its derivative is
$$
\sigma'(x) = \frac{e^x}{1 + e^x} = \mathrm{Sigmoid}(x) =: S(x).
$$
Now, let us define the normalized function $\widetilde{\sigma}(x) := \frac{\sigma(x)}{x}$ and verify that it satisfies Condition~\ref{condition1} for any $k < \infty$. 
We begin by checking its smoothness at $x = 0$. For any $x$ sufficiently close to $0$, we observe that
$$
\widetilde{\sigma}(x) = \frac{\mathrm{Softplus}(x) - \mathrm{Softplus}(0)}{x - 0} = \sum_{k=1}^\infty \frac{h^{(k)}(0)}{k!} x^{k-1},
$$
where $h(x) := \mathrm{Softplus}(x) = \ln(1 + e^x)$. Since $h(x)$ is smooth around $x = 0$, the series expansion guarantees the smoothness of $\widetilde{\sigma}(x)$ in a neighborhood of $x = 0$.

For $|x| > 0$, its derivative is
$$
\widetilde{\sigma}'(x) = \frac{xS(x) - \sigma(x)}{x^2} = 
\begin{cases}
(S(x) - 1)x^{-1} + \sigma(-x) x^{-2} & \text{if } x > 0, \\
S(x)x^{-1} + \sigma(x) x^{-2} & \text{if } x < 0.
\end{cases}
$$
Using similar arguments for the derivatives of ${\rm Sigmoid}(x)$ along with Leibniz’s rule, we obtain:
$$
\begin{cases}
\left| \frac{d^k}{dx^k} \left((S(x) - 1)x^{-1}\right) \right| \le e^{-x} \cdot |P_k(x^{-1})| & \text{as } x \to \infty, \\
\left| \frac{d^k}{dx^k} \left(S(x)x^{-1}\right) \right| \le e^{x} \cdot |P_k(x^{-1})| & \text{as } x \to -\infty,
\end{cases}
$$
for some polynomial $P_k$ of degree at most $k+1$.
Similarly, for the logarithmic terms, an inductive argument yields:
$$
\begin{cases}
\left| \frac{d^k}{dx^k} \left(\sigma(-x)x^{-2} \right) \right| \le e^{-x} \cdot |R_k(x^{-1})| + \frac{|\sigma(-x)|}{|x|^{k+2}} & \text{as } x \to \infty, \\
\left| \frac{d^k}{dx^k} \left(\sigma(x)x^{-2} \right) \right| \le e^{x} \cdot |R_k(x^{-1})| + \frac{|\sigma(x)|}{|x|^{k+2}} & \text{as } x \to -\infty,
\end{cases}
$$
for some polynomial $R_k$ of degree at most $k+1$. Here, the last term arises from the $k$-th order derivative of $x^{-2}$ in the application of Leibniz’s rule to 
$\frac{d^k}{dx^k} \left(\sigma(-x)\, x^{-2} \right)$ and $\frac{d^k}{dx^k} \left(\sigma(x)\, x^{-2} \right)$.

All the activation functions discussed above exhibit sufficient smoothness and asymptotic decay to satisfy the required conditions, thereby enabling theoretical guarantees for their approximation properties under the proposed framework. 

\section{Proofs in Section \ref{preliminaries}}\label{app lem}
\begin{proof}[Proof of Lemma~\ref{time}]
By Lemma~\ref{sqr11}, for any \(\varepsilon>0\), there exists a depth-\(1\), width-\(4\) neural network \(\phi_1\), where \(4\) is independent of \(\varepsilon\), such that
\[
\|x^2-\phi_1(x)\|_{C^m([-2M,2M])}\le \frac{\varepsilon}{3}.
\]
Define
\[
\phi(x,y):=\frac{\phi_1(x+y)-\phi_1(x)-\phi_1(y)}{2}.
\]
Since
\[
xy=\frac{(x+y)^2-x^2-y^2}{2},
\]
we have
\[
xy-\phi(x,y)
=
\frac{\bigl((x+y)^2-\phi_1(x+y)\bigr)-\bigl(x^2-\phi_1(x)\bigr)-\bigl(y^2-\phi_1(y)\bigr)}{2}.
\]
Therefore, by the triangle inequality and the fact that \(x+y\in[-2M,2M]\) for \((x,y)\in[-M,M]^2\), it follows that
\[
\|xy-\phi(x,y)\|_{C^m([-M,M]^2)}
\le \varepsilon.
\]
This completes the proof.
\end{proof}

\begin{proof}[Proof of Lemma~\ref{Id}]
By Lemma~\ref{sqr11}, for any \(\varepsilon>0\), there exists a depth-\(1\), width-\(4\) \(\sigma\)-network \(\phi_1\) such that
\[
\|x^2-\phi_1(x)\|_{C^m([-M,M+1])}\le 2\varepsilon.
\]
Define
\[
\phi(x):=\frac{\phi_1(x+1)-\phi_1(x)-1}{2}.
\]
Since \(\bigl((x+1)^2-x^2-1\bigr)/2=x\), a direct application of the triangle inequality yields
\[
\|x-\phi\|_{C^m([-M,M])}\le \varepsilon.
\]
\end{proof}

\begin{proof}[Proof of Lemma~\ref{xm+1}]
The construction iteratively combines the networks from
Lemmas~\ref{time} and~\ref{Id}; see Fig.~\ref{fig:xp}. We now verify that
its approximation error can be made arbitrarily small.

\begin{figure}[ht]
    \centering
    \includegraphics[width=0.7\linewidth]{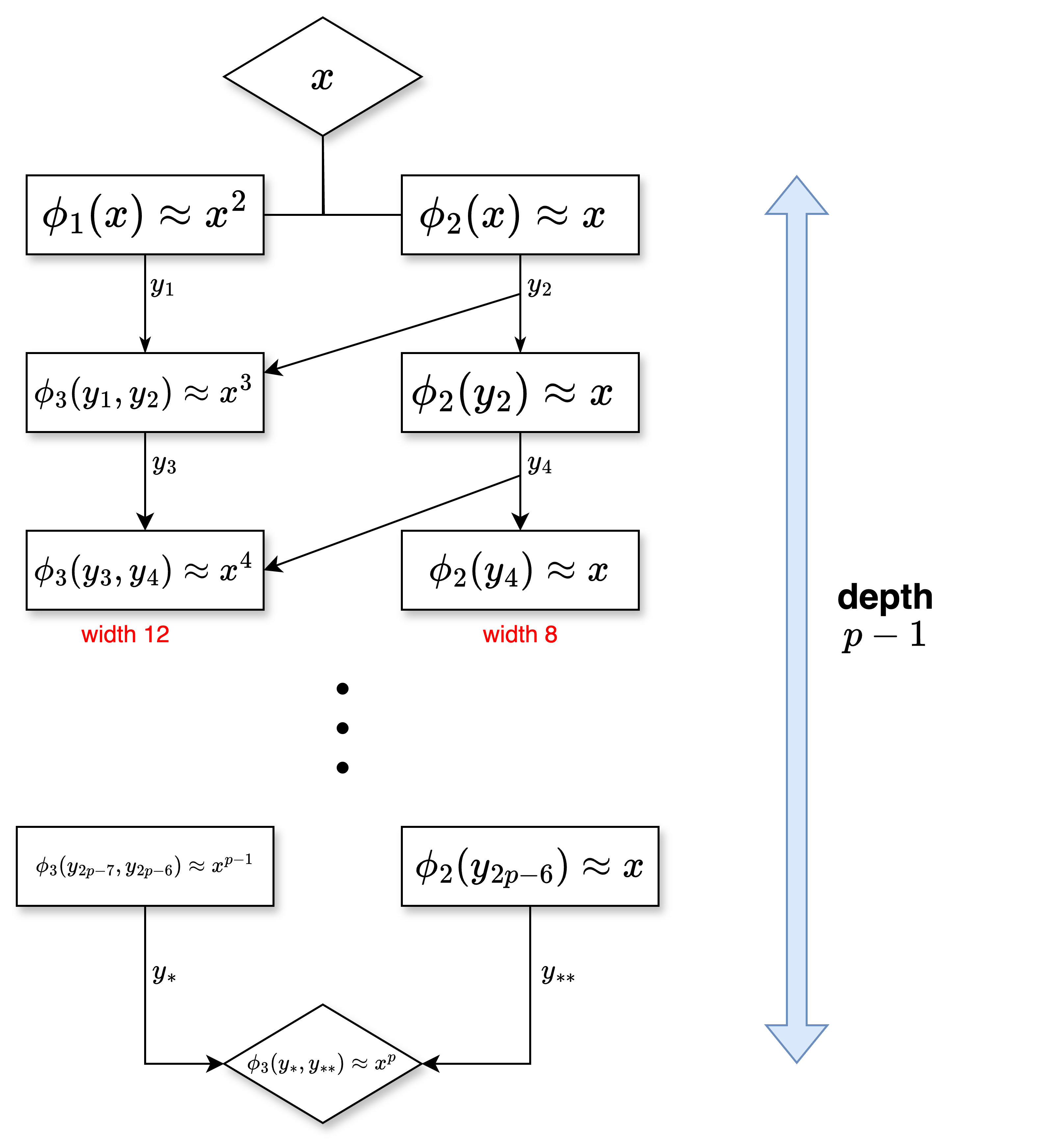}
    \caption{Network architecture used to approximate \(x^p\).}
    \label{fig:xp}
\end{figure}

Fix \(m\) and \(M\). By Lemmas~\ref{sqr11}, \ref{time}, and~\ref{Id}, for any
\(0<\varepsilon_0<1\), there exist networks \(\phi_1,\phi_2,\phi_3\) such that
\[
\|\phi_1-x^2\|_{W^{m,\infty}([-M,M])}\le \varepsilon_0,\qquad
\|\phi_2-x\|_{W^{m,\infty}([-M,M])}\le \varepsilon_0,
\]
and
\[
\|\phi_3-xy\|_{W^{m,\infty}([-M,M]^2)}\le \varepsilon_0.
\]

We first define
\[
y_1:=\phi_1(x),\qquad y_2:=\phi_2(x).
\]
Then
\[
\|y_1-x^2\|_{W^{m,\infty}([-M,M])}\le \varepsilon_0,
\qquad
\|y_2-x\|_{W^{m,\infty}([-M,M])}\le \varepsilon_0.
\]

For the second layer, define
\[
y_3:=\phi_3(y_1,y_2),\qquad y_4:=\phi_2(y_2).
\]
By Lemma~\ref{lem:Wm_inf_bounds}, there exists a constant \(C_1>0\),
depending only on \(\phi_1,\phi_2,\phi_3\), such that
\[
\|y_4-x\|_{W^{m,\infty}([-M,M])}
=
\|\phi_2(y_2)-x\|_{W^{m,\infty}([-M,M])}
\le C_1\varepsilon_0,
\]
and
\[
\|y_3-x^3\|_{W^{m,\infty}([-M,M])}
=
\|\phi_3(y_1,y_2)-x^3\|_{W^{m,\infty}([-M,M])}
\le C_1\varepsilon_0.
\]

For the third layer, define
\[
y_5:=\phi_3(y_3,y_4),\qquad y_6:=\phi_2(y_4).
\]
Applying Lemma~\ref{lem:Wm_inf_bounds} again, there exists a constant
\(C_2>0\), depending only on \(C_1\), \(\phi_2\), \(\phi_3\), and the
\(W^{m,\infty}\)-bounds of \(y_3\) and \(y_4\), such that
\[
\|y_6-x\|_{W^{m,\infty}([-M,M])}\le C_2\varepsilon_0,
\qquad
\|y_5-x^4\|_{W^{m,\infty}([-M,M])}\le C_2\varepsilon_0.
\]

Repeating the same argument, at the \(k\)-th layer we obtain two outputs
\(y_{2k-1}\) and \(y_{2k}\) satisfying
\[
\|y_{2k-1}-x^{k+1}\|_{W^{m,\infty}([-M,M])}\le C_k\varepsilon_0,
\qquad
\|y_{2k}-x\|_{W^{m,\infty}([-M,M])}\le C_k\varepsilon_0,
\]
where \(C_k>0\) depends only on the previous constants and the
\(W^{m,\infty}\)-bounds of the intermediate outputs.

In particular, after \(p-2\) layers, we obtain
\[
y_*:=y_{2p-7}\approx x^{p-1},
\qquad
y_{**}:=y_{2p-6}\approx x,
\]
and the final output
\[
\phi:=\phi_3(y_*,y_{**})
\]
satisfies
\[
\|\phi-x^p\|_{W^{m,\infty}([-M,M])}\le C_{p-2}\varepsilon_0.
\]
Finally, choosing
\[
\varepsilon_0=\frac{\varepsilon}{C_{p-2}}
\]
yields
\[
\|\phi-x^p\|_{W^{m,\infty}([-M,M])}\le \varepsilon.
\]
This completes the proof.
\end{proof}

\begin{proof}[Proof of Lemma~\ref{times}]
The construction iteratively combines the networks from
Lemmas~\ref{time} and~\ref{Id}; see Fig.~\ref{fig:xprod}. More precisely, the
left branch recursively applies the multiplication network \(\phi_3\) to build
successive approximations of \(x_1x_2\), \(x_1x_2x_3\), \(\ldots\),
\(x_1x_2\cdots x_{d-1}\), while the right branch recursively applies the
identity network \(\phi_2\) to propagate an approximation of \(x_d\) through
the layers. In the final layer, these two outputs are combined by \(\phi_3\)
to approximate \(x_1x_2\cdots x_d\).

The verification that the approximation error can be made arbitrarily small is
the same as in the proof of Lemma~\ref{xm+1}, by iteratively applying
Lemma~\ref{lem:Wm_inf_bounds}. We therefore omit the details here.
\end{proof}
\section{Proof of Proposition \ref{fN}}
Before we show Proposition \ref{fN}, we define subsets of $\Omega_{\vm}$ for simplicity notations.
		
		For any $\vm\in\{1,2\}^d$, we define \begin{equation}
			\Omega_{\vm,\vi}:=[0,1]^d\cap\prod_{j=1}^d\left[\frac{2i_j-1_{m_j\le 2}}{2J},\frac{3+4i_j-2\cdot1_{m_j\le 2}}{4J}\right]\notag
		\end{equation}$\vi=(i_1,i_2,\ldots,i_d)\in\{0,1\ldots,J\}^d$, and it is easy to check $\bigcup_{\vi\in\{0,1\ldots,J\}^d}\Omega_{\vm,\vi}=\Omega_{\vm}$. 
		
	Before the proof, we must introduce the partition of unity, and average Taylor polynomials.
		
		\begin{definition}[The partition of unity]\label{unity} Let $d,J\in\sN_+$, then
			\[\Psi=\left\{h_{\vi}:{\vi}\in\{0,1,\ldots,J\}^d\right\}\] with $h_{\vi}:\sR^d\to\sR$ for all ${\vi}\in\{0,1,\ldots,J\}^d$ is called the partition of unity $[0,1]^d$ if it satisfies
			
			(i): $0\le h_{\vi}(\vx)\le 1$ for every $h_{\vi}\in\Psi$.
			
			(ii): $\sum_{h_{\vi}\in\Psi}h_{\vi}=1$ for every $x\in [0,1]^d$.
		\end{definition}	
		
		\begin{definition}\label{average}
			Let $n\ge 1$ and $f\in W^{n,\infty}(\Omega)$, $\vx_0\in\Omega$ and $r>0$ such that for the ball $B(\vx_0):=B(\vx_0)_{r,|\cdot|}$ which is a compact subset of $ \Omega$. The corresponding Taylor
			polynomial of order $n$ of $f$ averaged over $B$ is defined for \begin{equation}
				Q^nf(x):=\int_{B}	T^n_{\vy}f(\vx)b_r(\vy)\,\D \vy\notag
			\end{equation} where \begin{align}
				T^n_{\vy}{f}(\vx)&:=\sum_{|\valpha|\le n-1}\frac{1}{\valpha!}D^{\valpha}{f}(\vy)(\vx-\vy)^{\valpha},\notag\\b_r(\vx)&:=\begin{cases}
					\frac{1}{c_r}e^{-\left(1-\left(| \vx-\vx_{0}| / r\right)^{2}\right)^{-1}},~&|\vx-\vx_0|<r, \\
					0, ~&|\vx-\vx_0|\le r ,
				\end{cases}\quad c_r=\int_{\sR^d}e^{-\left(1-\left(| \vx-\vx_{0}| / r\right)^{2}\right)^{-1}}\,\D x.\notag
			\end{align}
		\end{definition}
		
		\begin{lemma}\label{average coe}
			Let $n\ge 1$ and $f\in W^{n,\infty}(\Omega)$, $\vx_0\in\Omega$ and $r>0$ such that for the ball $B(\vx_0):=B_{r,|\cdot|}(\vx_0)$ which is a compact subset of $ \Omega$. The corresponding Taylor
			polynomial of order $n$ of $f$ averaged over $B$ can be read as \[Q^nf(\vx)=\sum_{|\valpha|\le n-1}c_{f,\valpha}\vx^{\valpha}.\]
			Furthermore, \begin{align}
				\left|c_{f,\valpha}\right|\le C_2(n,d)\|{f}\|_{W^{n-1,\infty}(B)}.\notag
			\end{align} where $C_2(n,d)=\sum_{|\valpha+\vbeta|\le n-1}\frac{1}{\valpha!\vbeta!}$.
		\end{lemma}
		
		\begin{proof}
			Based on \cite[Lemma B.9.]{guhring2020error}, $Q^nf(x)$ can be read as\begin{equation}
				Q^nf(\vx)=\sum_{|\valpha|\le n-1}c_{f,\valpha}\vx^{\valpha}\notag
			\end{equation} where \begin{equation}
				c_{f,\valpha}=\sum_{|\valpha+\vbeta|\le n-1}\frac{1}{(\vbeta+\valpha)!}a_{\vbeta+\valpha}\int_{B}	D^{\valpha+\vbeta}{f}(\vx)\vy^{\vbeta}b_r(\vy)\,\D \vy\notag
			\end{equation} for $a_{\vbeta+\valpha}\le \frac{(\valpha+\vbeta)!}{\valpha!\vbeta!}$.
			Note that \begin{align}
				\left|\int_{B}	D^{\valpha+\vbeta}{f}(\vx)\vy^{\vbeta}b_r(\vy)\,\D \vy\right|\le \|{f}\|_{W^{n-1,\infty}(B)}\|b_r(x)\|_{L^1(B)}=\|{f}\|_{W^{n-1,\infty}(B)}.\notag
			\end{align} Then  \begin{align}
				\left|c_{f,\valpha}\right|\le C_2(n,d)\|{f}\|_{W^{n-1,\infty}(B_{\vm,N})}.\notag
			\end{align} where $C_2(n,d)=\sum_{|\valpha+\vbeta|\le n-1}\frac{1}{\valpha!\vbeta!}$.
		\end{proof}

		The proof of Proposition \ref{fN} is based on average Taylor polynomials and the Bramble--Hilbert Lemma \cite[Lemma 4.3.8]{brenner2008mathematical}. 
		
		\begin{definition}
			Let $\Omega,~B\in\sR^d$. Then $\Omega$ is called stared-shaped with respect to $B$ if \[\overline{\text{conv}}\left(\{\vx\}\cup B\subset \Omega\right),~\text{for all }\vx\in\Omega.\]
		\end{definition}
		
		\begin{definition}
			Let $\Omega\in\sR^d$ be bounded, and define \[\fR:=\left\{r>0: \begin{array}{l}
				\text { there exists } \vx_0 \in \Omega \text { such that } \Omega \text { is } \\
				\text { star-shaped with respect to } B_{r,|\cdot|}\left(\vx_0\right)
			\end{array}\right\} .\]Then we define\[r_{\max }^{\star}:=\sup \fR \quad \text { and call } \quad \gamma:=\frac{\operatorname{diam}(\Omega)}{r_{\max }^{\star}}\]the chunkiness parameter of $\Omega$ if $\fR\not=\emptyset$.
		\end{definition}

        We will make use of the Sobolev semi-norm, defined as follows:\begin{definition}
Let $n\in\mathbb{N}_+$ and $1\le p\le\infty$. For a scalar function $f:\Omega\to\mathbb{R}$, its Sobolev semi-norm is
\[
|f|_{W^{n,p}(\Omega)}
:=
\begin{cases}
\displaystyle
\Bigl(\sum_{|\alpha|=n}\|D^\alpha f\|_{L^p(\Omega)}^p\Bigr)^{\!1/p}, 
& 1 \le p < \infty, \\[1em]
\displaystyle
\max_{|\alpha|=n}\|D^\alpha f\|_{L^\infty(\Omega)},
& p = \infty.
\end{cases}
\]
Moreover, if $\vf=(f_1,\dots,f_d)\in W^{n,\infty}(\Omega;\mathbb{R}^d)$ is vector-valued, we define
\[
|\vf|_{W^{n,\infty}(\Omega;\mathbb{R}^d)}
:=\max_{1\le i\le d}\bigl|f_i\bigr|_{W^{n,\infty}(\Omega)}.
\]
\end{definition}
		
		\begin{lemma}[{Bramble--Hilbert Lemma \cite[Lemma 4.3.8]{brenner2008mathematical}}]\label{BH}
			Let $\Omega\in\sR^d$ be open and bounded, $\vx_0\in\Omega$ and $r>0$ such that $\Omega$ is the stared-shaped with respect to $B:=B_{r,|\cdot|}\left(\vx_0\right)$, and $r\ge \frac{1}{2}r_{\max }^{\star}$. Moreover, let $n\in\sN_+$, $1\le p\le \infty$ and denote by $\gamma$ by the chunkiness parameter of $\Omega$. Then there is a constant $C(n,d,\gamma)>0$ such that for all $f\in W^{n,p}(\Omega)$\[\left|f-Q^n f\right|_{W^{k, p}(\Omega)} \le C(n,d,\gamma) h^{n-k}|f|_{W^{n, p}(\Omega)} \quad \text { for } k=0,1, \ldots, n\]where $Q^n f$ denotes the Taylor polynomial of order $n$ of $f$ averaged over $B$ and $h=\operatorname{diam}(\Omega)$.
		\end{lemma}
		\begin{proof}[Proof of Proposition~\ref{fN}]
Without loss of generality, we only consider the case
\[
\vm=(1,1,\ldots,1)=:\vm_*.
\]

Let
\[
E:W^{n,\infty}(\Omega)\to W^{n,\infty}(\mathbb{R}^d)
\]
be an extension operator (see \cite{stein1970singular}), and set
\[
\tilde f:=Ef.
\]
Denote by \(C_E\) the operator norm of \(E\).

For each \(\vi=(i_1,\ldots,i_d)\in\{0,1,\ldots,J-1\}^d\), define
\[
Q_{\vi}:=\Omega_{\vm_*,\vi}=\prod_{j=1}^d\left[\frac{i_j}{J}+\frac{1}{4J},\,\frac{i_j+1}{J}\right]
\]
and
\[
B_{\vi,J}:=B_{\frac{1}{4J},|\cdot|}\!\left(\frac{8\vi+5}{8J}\right).
\]
Let \(p_{f,\vi}\) be the averaged Taylor polynomial from Definition~\ref{average}, namely
\[
p_{f,\vi}(\vx)
:=
\int_{B_{\vi,J}} T^n_{\vy}\tilde f(\vx)\,b_{\frac{1}{4J}}(\vy)\,\mathrm{d}\vy.
\]
By Lemma~\ref{average coe}, \(p_{f,\vi}\) can be written as
\[
p_{f,\vi}(\vx)=\sum_{|\valpha|\le n-1} c_{f,\vi,\valpha}\vx^{\valpha},
\]
where
\[
|c_{f,\vi,\valpha}|
\le
C_2(n,d)\|f\|_{W^{n-1,\infty}(\Omega)}.
\]

The reason for defining the averaged Taylor polynomial on \(B_{\vi,J}\) is that
\(B_{\vi,J}\subset Q_{\vi}\), so that the Bramble--Hilbert lemma can be applied on \(Q_{\vi}\).
Indeed, the cube \(Q_{\vi}\) has side length \(3/(4J)\), hence
\[
r_{\max}^\star(Q_{\vi})=\frac{3}{8J},
\qquad
\gamma(Q_{\vi})
=
\frac{\operatorname{diam}(Q_{\vi})}{r_{\max}^\star(Q_{\vi})}
=
2\sqrt d,
\]
and
\[
\frac{1}{4J}\le r_{\max}^\star(Q_{\vi}).
\]
Therefore, by Lemma~\ref{BH},
\[
|\tilde f-p_{f,\vi}|_{W^{s,\infty}(Q_{\vi})}
\le
C_{BH}(n,d)\|\tilde f\|_{W^{n,\infty}(\mathbb{R}^d)}J^{-(n-s)},
\qquad s=0,1,\ldots,m.
\]
Using \(\|\tilde f\|_{W^{n,\infty}(\mathbb{R}^d)}\le C_E\|f\|_{W^{n,\infty}(\Omega)}\), we obtain
\[
\|\tilde f-p_{f,\vi}\|_{W^{s,\infty}(Q_{\vi})}
\le
C_1(n,d,s)\|f\|_{W^{n,\infty}(\Omega)}J^{-(n-s)},
\qquad s=0,1,\ldots,m,
\]
for some constant \(C_1(n,d,s)\) independent of \(J\).

Next, by the construction in \cite{lee2003smooth}, there exists a family
\[
\{h_{\vi}:\vi\in\{0,1,\ldots,J-1\}^d\}\subset C^\infty(\Omega)
\]
such that
\[
h_{\vi}(\vx)=1,\qquad \vx\in Q_{\vi},
\]
and
\[
\operatorname{supp}(h_{\vi})\cap\Omega_{\vm_*}=Q_{\vi}.
\]
Since \(h_{\vi}=1\) on \(Q_{\vi}\), we have
\[
\|h_{\vi}(\tilde f-p_{f,\vi})\|_{W^{s,\infty}(Q_{\vi})}
=
\|\tilde f-p_{f,\vi}\|_{W^{s,\infty}(Q_{\vi})}
\le
C_1(n,d,s)\|f\|_{W^{n,\infty}(\Omega)}J^{-(n-s)}.
\]

Now define
\[
f_{J,\vm_*}(\vx)
:=
\sum_{\vi\in\{0,1,\ldots,J-1\}^d} h_{\vi}(\vx)p_{f,\vi}(\vx).
\]
Since
\[
\Omega_{\vm_*}=\bigcup_{\vi\in\{0,1,\ldots,J-1\}^d}Q_{\vi}
\]
and the supports satisfy
\[
\operatorname{supp}(h_{\vi})\cap\Omega_{\vm_*}=Q_{\vi},
\]
at each point \(\vx\in\Omega_{\vm_*}\) only one term in the above sum is nonzero. Hence
\begin{align*}
\|f-f_{J,\vm_*}\|_{W^{s,\infty}(\Omega_{\vm_*})}
&=
\left\|
\sum_{\vi\in\{0,1,\ldots,J-1\}^d}
h_{\vi}(\tilde f-p_{f,\vi})
\right\|_{W^{s,\infty}(\Omega_{\vm_*})} \\
&\le
\max_{\vi\in\{0,1,\ldots,J-1\}^d}
\|h_{\vi}(\tilde f-p_{f,\vi})\|_{W^{s,\infty}(Q_{\vi})} \\
&\le
C_1(n,d,s)\|f\|_{W^{n,\infty}(\Omega)}J^{-(n-s)}.
\end{align*}

Finally,
\begin{align*}
f_{J,\vm_*}(\vx)
&=
\sum_{\vi\in\{0,1,\ldots,J-1\}^d}
h_{\vi}(\vx)p_{f,\vi}(\vx) \\
&=
\sum_{\vi\in\{0,1,\ldots,J-1\}^d}
\sum_{|\valpha|\le n-1}
h_{\vi}(\vx)c_{f,\vi,\valpha}\vx^{\valpha} \\
&=
\sum_{|\valpha|\le n-1}
\left(
\sum_{\vi\in\{0,1,\ldots,J-1\}^d}
h_{\vi}(\vx)c_{f,\vi,\valpha}
\right)\vx^{\valpha} \\
&=:
\sum_{|\valpha|\le n-1}
g_{f,\valpha,\vm_*}(\vx)\vx^{\valpha}.
\end{align*}
Moreover,
\[
|g_{f,\valpha,\vm_*}(\vx)|
\le
C_2(n,d)\|f\|_{W^{n-1,\infty}(\Omega)},
\qquad
\vx\in\Omega_{\vm_*}.
\]
Since \(\operatorname{supp}(h_{\vi})\cap\Omega_{\vm_*}=Q_{\vi}\) and \(h_{\vi}=1\) on \(Q_{\vi}\), we have
\[
g_{f,\valpha,\vm_*}(\vx)=c_{f,\vi,\valpha},
\qquad
\vx\in Q_{\vi}.
\]
Therefore, \(g_{f,\valpha,\vm_*}\) is a step function on \(\Omega_{\vm_*}\), constant on each cube \(Q_{\vi}\). This completes the proof.
\end{proof}

\vskip 0.2in

\bibliography{main}

\end{document}